%% file: main.tex
\documentclass{article}

\usepackage{macros}
\usepackage{macros_project}

\title{Optimal Prediction-Augmented Algorithms for Testing Independence of Distributions}

\author{Maryam Aliakbarpour\footnote{Department of Computer Science and Ken Kennedy Institute} \\
  Rice University \\
  \texttt{maryama@rice.edu} \\
  \and
  Alireza Azizi \\
  Rice University \\
  \texttt{alireza.azizi@rice.edu} \\
  \and
  Ria Stevens \\
  Rice University \\
  \texttt{ria.stevens@rice.edu}
}

\date{\today}

\newif\ifcolt
\coltfalse

\begin{document}

\maketitle

\begin{abstract}
    \input{00-abstract}
\end{abstract}

\clearpage

\section{Introduction}

\input{01-intro}
\subsection{Problem Statement}

\input{01.1-problemStatement}
\subsection{Our Main Result}
\input{01.2-contributions}

\subsection{Technical Overview}
\subsubsection{Background: Flattening} \label{sec:background:flattening}

\input{02.1-flattening}
\subsubsection{Overview of our Results}
\input{02.2-overview}

\subsection{Related Works}

\input{01.3-relatedWorks}

\input{03-upperbounds}

\input{appendix/app09-HighDimensionalTester}

\section{Lower Bounds for Bivariate Augmented Independence Testing}
\label{sec:2D-LB}
\input{appendix/app10-2D-LB}

\section{Lower Bounds for Multivariate Augmented Independence Testing}
\label{sec: d_Dim_LB}
\input{appendix/app11-highD-LB}

\section{Acknowledgements}
\input{99-acknowledgements}


\bibliographystyle{alpha}
\bibliography{references, new_references}

\clearpage

\appendix
\input{Appendix}
\end{document}

%% file: 00-abstract.tex
Independence testing is a fundamental problem in statistical inference: given samples from a joint distribution $p$ over multiple random variables, the goal is to determine whether $p$ is a product distribution or is $\tol$-far from all product distributions in total variation distance. In the non-parametric finite-sample regime, this task is notoriously expensive, as the minimax sample complexity scales polynomially with the support size. In this work, we move beyond these worst-case limitations by leveraging the framework of \textit{augmented distribution testing}. We design independence testers that incorporate auxiliary, but potentially untrustworthy, predictive information. Our framework ensures that the tester remains robust, maintaining worst-case validity regardless of the prediction's quality, while significantly improving sample efficiency when the prediction is accurate. 
Our main contributions include: (i) a bivariate independence tester for discrete distributions that adaptively reduces sample complexity based on the prediction error; (ii) a generalization to the high-dimensional multivariate setting for testing the independence of $d$ random variables; and (iii) matching minimax lower bounds demonstrating that our testers achieve optimal sample complexity. 

%% file: 01-intro.tex
Independence testing is a central problem in statistical inference, with applications spanning from medical domains~\citep{Cornfield51,Sasieni97, HouYFZHLLW22} to machine learning problems such as causal discovery and feature selection~\citep{BachJ02, SzekelyRB07, SongSGBB07, GrettonFTSSS07, ZhangPJS11, BrownPZL12, CaiLZ22}. 
Formally, given samples from a joint distribution over multiple random variables, the goal is to determine whether the variables are statistically independent or exhibit some form of dependence.

While classical approaches like the $\chi^2$ test of \cite{pearson1900x} are widely used, they typically lack guarantees for the non-parametric finite-sample regime (such as small type I and type II error). Instead, their validity stems from either restrictive parametric assumptions or asymptotic properties that hold only in the limit as the number of samples tends to infinity. In the absence of such assumptions, independence testing is notoriously expensive: the minimax sample complexity scales polynomially with the support size~\citep{BatuFFKRW01, DiakonikolasK16}. This high lower bound presents a formidable challenge for sample-efficient inference.

To move beyond these worst-case limitations, a emerging line of work suggests that the analyst need not operate in a vacuum. In many modern data science contexts, one often possesses auxiliary information: public datasets gathered from disparate institutions, or a distribution originating from historical data, a generative model, or domain-specific heuristics. While such information may be abundant, it often lacks formal quality guarantees. A natural question is to determine if one can leverage such \textit{untrustworthy} predictive information in a rigorous manner—designing tests that remain mathematically sound (and robust) when the prediction is poor, yet achieve significantly improved sample efficiency when the prediction is accurate.

Recently, the framework of \textit{augmented algorithms} has been proposed to bridge this gap, incorporating auxiliary information into classical statistical problems~\citep{canonne2014testing, CanonneRS:2015, PRS2018, eden2021learning, AIRS24} and beyond~\citep{algorithms-with-predictions}. This paradigm has shown great promise in distribution testing tasks such as uniformity testing, identity testing (goodness of fit), and closeness testing (equivalence testing)~\citep{AIRS24, ABCR25}. In the augmented testing framework, the tester receives both samples from $p$, the true underlying data distribution, and a predicted distribution $\hat{p}$. The tester is required to be \textit{robust}, maintaining worst-case validity regardless of $\pred$, while simultaneously \textit{exploiting} accurate predictions to surpass classical minimax lower bounds.

In this work, we bring the framework of augmented distribution testing to the problem of independence testing over a discrete support. Our main contributions are as follows:
${i)}$ We develop an independence tester for a bivariate distribution with finite sample complexity and worst-case error guarantees. This tester leverages a predicted distribution and adaptively reduces the sample complexity depending on the quality of the prediction without hindering the accuracy of the test.
${ii)}$ We generalize this result to the high-dimensional setting, testing the independence of multiple random variables.
${iii)}$ We provide matching lower bounds which indicate that our testers achieve the optimal sample complexity.


%% file: 01.1-problemStatement.tex

We adopt the framework of property testing of distributions~\citep{GoldreichBook17}. In this setting, given sample access to a distribution $p$, a tester, with high probability, must output \accept\textsc{accept} if $p$ belongs to a property $\mathcal{P}$ (viewed as a set of distributions) and \textsc{reject} if $p$ is $\tol$-far in total variation distance from every distribution in $\mathcal{P}$. Here, we consider $\mathcal{P}$ to be the property of \textit{independence}---specifically, the set of all product distributions over a desired factorization of the domain.

In the augmented setting~\citep{AIRS24}, in addition to sample access to the distribution $p$, we are provided with a predicted distribution $\hat{p}$ and a suggested error bound $\predError$ representing the purported accuracy of $\hat{p}$. We assume full knowledge of $\hat{p}$ (i.e., its probability mass function is known). In this framework, an augmented independence tester has a third option beyond outputting \textsc{accept} or \textsc{reject}: it may choose not to answer the test and instead output \textsc{inaccurate information} if the actual quality of the prediction is worse than the suggested error $\predError$. However, the tester is required to provide an answer whenever the prediction is sufficiently accurate (i.e., when the suggested error $\predError$ is correct). Crucially, the tester cannot output an incorrect answer; if the tester does decide to answer the test, its output must be correct with high probability, regardless of the quality of the prediction. Formally, we define the desired augmented tester as follows:

\begin{definition}[Augmented Independence Tester]
    \label{def:augIndTester}
    Let $\predError \in [0, 1], \tol \in (0, 1), \probFailure \in (0, 1)$. For all $i \in [d]$, let $n_i \ge 2$ be an integer denoting the domain size. Let $\dist$ and $\pred$ be distributions over $\prod_{i=1}^d [n_i]$, and let $\dist_1, \ldots, \dist_d$ denote the marginals of $\dist$. An algorithm $\AA$ which is given sample access to $\dist$ and explicit access to $\pred$, and outputs \textsc{accept}, \textsc{reject}, or \textsc{inaccurate information}, is an $(\predError, \tol, \probFailure)$-augmented independence tester if for all $\dist$, it satisfies:
    \begin{itemize}
        \item If $\dist = \dist_1 \times \dist_2 \times \cdots \times \dist_d$, $\AA$ outputs \textsc{reject} with probability at most $\probFailure/2$.
        \item If $p$ is $\tol$-far from all product distributions, $\AA$ outputs \textsc{accept} with probability at most $\probFailure/2$.
        \item If $\TV{p}{\pred} \leq \predError$, $\AA$ outputs \textsc{inaccurate information} with probability at most $\probFailure/2$.
    \end{itemize}
\end{definition}

\begin{remark}[Finding the prediction error.]
\label{remark:choosingPredError}
Note that in many practical settings, the prediction error $\predError$ may not be known \textit{a priori}. However, the robustness of our augmented tester is naturally suited for identifying a desired error level adaptively. In particular, by employing the \textsc{search} meta-algorithm of \cite{AIRS24}, we can search for the true error $\predError^\star$ by starting with the smallest $\predError$ and repeatedly invoking the augmented tester with increasing values of $\predError$ until the algorithm provides a definitive answer (i.e., it no longer outputs \textsc{inaccurate information}). The expected sample complexity of this procedure is $\tilde{O}(f(\predError^\star))$, where $f(\predError^\star)$ denotes the sample complexity of the tester when provided with the true error $\predError^\star$ of the prediction. For the remainder of this paper, we assume a fixed $\predError$ is provided and derive the sample complexity based on this error bound. 
\end{remark}

%% file: 01.2-contributions.tex

    
In this work, we provide matching upper and lower bounds for augmented independence testing of both two-dimensional and $\numDims$-dimensional distributions. Our contributions are summarized in the following theorem.
\begin{theorem}[Informal version of Theorems~\ref{thm:2D-UB},~\ref{thrm: HighDim},~\ref{thm:2D-LB}~and~\ref{thm:LB-d-dim-testing}]
Let $\numDims \geq 2$. Let $\distHighDim$ be an unknown distribution over $n_1 \times \ldots \times n_\numDims$
, and let $\predHighDim$ be a known prediction for $\distHighDim$ over the same domain. Let $\dimTotal \coloneqq \prod_{i=1}^\numDims n_i$ denote the total domain size. For every $\tol \in (0, 1)$ and $\predError \in [0, 1]$, the sample complexity of $(\predError, \tol, \probFail = 0.1)$-augmented independence testing of $\distHighDim$ given prediction $\predHighDim$ is
\begin{equation}
    \Theta\left( \max_{j \in [\numDims]} \left( \frac{\sqrt{\dimTotal}}{\tol^2}, \frac{n_j^{1/3} \dimTotal^{1/3} \predError^{1/3}}{\tol^{4/3}} \right) \right).
\end{equation}
    
\end{theorem}

The two-dimensional augmented independence tester and sample complexity upper bound are discussed in Section~\ref{sec:2D}, while the corresponding lower bound is established in Section~\ref{sec:2D-LB}. The $d$-dimensional tester and upper bound are presented in Section~\ref{sec: d_Dim_UB}, and the matching lower bound is proved in Section~\ref{sec: d_Dim_LB}. 
\begin{remark}[Success Amplification]
\label{remark:successAmplification}
    In this work, we consider a constant failure probability ($\probFail = 0.1$). Using standard amplification techniques, we can choose $\probFail$ to be arbitrarily small, at the cost of a multiplicative $\OO\left( \log (1/\probFail) \right)$ blow-up in sample complexity.
\end{remark}

%% file: 02.1-flattening.tex


In constructing our upper bounds, we use a technique from the literature of standard distribution testing known as \textit{flattening} \citep{DiakonikolasK16}.
This technique transforms one distribution testing problem into an equivalent problem over a larger domain by redistributing the mass of high-probability elements of the original distribution. By significantly decreasing the $\ell_2$-norm of this distribution, flattening enables the use of sample-efficient testing algorithms.

\paragraph{Standard Flattening.} Flattening divides the probability mass of each domain element of a distribution into element-wise equal buckets, defining a new distribution with a larger domain. Formally, given a distribution $p$ over $[\dimPOne]$ and a set of bucket sizes $b_i$ for all $i \in [\dimPOne]$, flattening maps each element $i$ to a tuple $(i, \ell)$, where $\ell$ is sampled uniformly from $[b_i]$.
This operation defines a new distribution $p^{(F)}$ such that:
\begin{equation}
    p^{(F)}(i, \ell) = \frac{p(i)}{b_i}
\end{equation}
\noindent Flattening satisfies two key properties, proved by \cite{DiakonikolasK16}, that enable its use in distribution testing:
\begin{enumerate}
    \item If two distributions $p$ and $q$ are flattened using the same bucket sizes, their $\ell_1$-distance is preserved. As a result, instead of providing $p$ and $q$ to a tester, we may provide $p^{(F)}$ and $q^{(F)}$ and expect the same result.
    \item One can utilize a preliminary set of samples from $p$ to identify which elements possess large probability mass. By setting the bucket sizes $b_i$ proportional to these empirical frequencies, one can effectively ``break down" the high-probability elements across multiple domain points.  This ensures that, in expectation, the resulting flattened distribution $p^{(F)}$ has a sufficiently small $\ell_2$-norm to permit sample-efficient testing.
\end{enumerate}

\paragraph{Augmented Flattening.}
\cite{AIRS24} extend the flattening techniques of \cite{DiakonikolasK16} to the augmented setting, allowing the prediction, $\hat{p}$, to guide the number of buckets created for each element. 
In addition to using knowledge gathered from sampling from $p$, this approach flattens elements that have a large probability mass under the prediction. This allows for a greater reduction of the $\ell_2$-norm when the prediction is accurate.

The augmented flattening technique of \cite{AIRS24} begins by drawing a sample set $\SS$ of size $\Poi(s)$ from $p$. Then, for all $i \in [\dimPOne]$, they let $N_i$ be the number of occurrences of $i \in \SS$ and choose the number of buckets for all $i$ as:
\begin{equation}
    b_i = \floor{n\cdot \pred(i)} + N_i + 1.
    \label{eqn:background-augFlatBuckets}
\end{equation}
The first term of this expression accounts for the influence of the prediction, assigning more buckets to elements with higher predicted probability mass. The second term accounts for information from the sample set, assigning more buckets to elements more frequently seen in $\SS$.
\cite{AIRS24} prove that, when $\pred$ is $\predError$-close to $p$, the distribution $p^{(F)}$ defined by this augmented flattening has a small $\ell_2^2$-norm:
\begin{equation}
        \E{\norm{p^{(F)}}^2_2} \leq \frac{2\predError}{s} + \frac{4}{n}.
    \end{equation}
See Fact~\ref{fact:augFlatCloseness} for further details. 


\paragraph{Closeness Testing via Flattening.} 
The problem of testing the closeness of two distributions over $[n]$ involves distinguishing whether $p = q$ or $\|p - q\|_1 \ge \tol$ with probability at least $0.9$. The minimax sample complexity for this task is $\Theta(n^{2/3}/\tol^{4/3} + \sqrt{n}/\tol^2)$ \citep{DiakonikolasK16}. A key advancement in this area is the development of closeness testers whose sample complexity scales with the $\ell_2$-norms of the underlying distributions rather than the raw support size \citep{ChanDVV14}. Specifically, if the distributions have $\ell_2$-norms at most $b$, they can be tested using $O(n\,b/\tol^2)$ many samples, which in certain regimes can be much more efficient than the worst-case bounds suggest. This relationship between $\ell_2$-norm and testing complexity is the core motivation for the flattening technique: by reducing the $\ell_2$-norm of the distributions, one can leverage these refined testers to achieve optimal sample complexity. Indeed, this approach was central to establishing the optimal bounds for standard independence testing in \cite{DiakonikolasK16}.

\paragraph{Independence Testing via Flattening.}
Independence testing can be viewed as the task of distinguishing whether a joint distribution $p$ is equal to the product of its marginals, $p_1 \times p_2$, or $3\tol$-far from it~\citep{BatuFFKRW01}. From this perspective, one can use samples from $p$ to simulate samples from both $p$ and $p_1 \times p_2$, subsequently employing a closeness tester to determine independence.

While a direct flattening of the joint domain $[n] \times [m]$ where $n\ge m$ yields a suboptimal sample complexity, \cite{DiakonikolasK16} achieved the optimal rate of $$O\left(\frac{n^{2/3}m^{1/3}}{\tol^{4/3}} + \frac{\sqrt{n m}}{\tol^2}\right)\,.$$ Roughly speaking, their approach relies on constructing a \textit{product of flattenings}: by flattening each dimension independently based on its respective marginal, they ensured that an independent distribution $p$ is transformed into a joint flattened distribution with a low $\ell_2$-norm. This makes the distribution amenable to the refined closeness tester of \cite{ChanDVV14}, which is highly efficient for distributions with small $\ell_2$-norms.

%% file: 02.2-overview.tex

\paragraph{Upper Bound for Two-Dimensional Augmented Independence Testing.}
Our approach combines augmented flattening with the product-flattening technique by using the prediction $\hat{p}$ to achieve a more efficient reduction of the marginal $\ell_2$-norms. The correctness of this approach hinges on a critical validation step: if $\hat{p}$ is accurate, the flattened marginals are guaranteed to have low $\ell_2$-norms. Conversely, a violation of this bound serves as a certificate of the prediction's inaccuracy. Furthermore, if $p$ is independent, the flattened joint distribution must also exhibit a low $\ell_2$-norm. Our algorithm explicitly estimates this joint norm; if it exceeds the expected threshold, it serves as a certificate of non-independence, allowing us to \textsc{reject} immediately. This ensures that we either detect dependence early through norm inflation or proceed with a ``well-behaved'' distribution that can be tested efficiently for closeness to its marginals.

\paragraph{Lower Bound for Two-Dimensional Augmented Independence Testing.}

Our lower bound for two-dimensional augmented independence testing follows in two separate cases defined by the relationships between $\dimPOne, \dimPTwo, \predError$ and $\tol$. The first case establishes a bound of $\Omega(\sqrt{\dimPOne\dimPTwo} / \tol^2)$ and follows from a reduction from the hard instance of standard independence testing \cite{ADK15, DiakonikolasK16}.
The second case establishes a bound of $\Omega\left( \dimPOne^{2/3} \dimPTwo^{1/3} \predError^{1/3} \tol^{-4/3} \right)$. To prove this bound, we construct two families of distributions which are information-theoretically indistinguishable but which an augmented tester must be able to distinguish. This construction builds upon techniques established in prior works \citep{Valiant11, BatuFRSW13, DiakonikolasK16} 
wherein a distribution's distance to a property is hidden among ``light'' elements with low probability mass, while the majority of the probability mass is concentrated in ``heavy'' rows that are identical across the two families. Here, by a row, we mean the probability distribution over the second coordinate obtained by fixing the first coordinate.

In the standard setting, the heavy rows follow a uniform distribution, while the light rows are uniform in one family and $\tol$-far from uniform in the other.
This construction forces a tester to draw enough samples to effectively uncover the heavy rows. However, in our augmented setting, the prediction may reveal the locations of the heavy rows to the tester. We extend this technique to the augmented setting by designing our families of distributions such that only a $\OO(\predError)$-factor of the rows are heavy. By choosing the prediction to be the uniform distribution, we trap the tester: this prediction provides no information about which rows are heavy or light, yet is $\predError$-close to the true distribution. This forces the tester to output \textsc{accept} or \textsc{reject}, rather than \textsc{inaccurate information}---essentially requiring it to differentiate between the two families.

We then provide an information-theoretic argument to demonstrate that the two families are indistinguishable without sufficient samples.
Following the framework of \cite{DiakonikolasK16}, we 
bound the mutual information between an indicator of a distribution's family and a sample set drawn from that distribution.
To do so, we show that the mutual information is only non-negligible when the tester sees a collision in (i.e. draws multiple samples from) a light row. Heavy rows contribute no mutual information because they're identical across the families, while without collisions in a light row, a tester gains no information about whether that row is uniform or $\tol$-far from uniform.
Under our hard instance, the probability of the tester observing enough collisions in the light rows is small unless the sample complexity is large.

\paragraph{Upper Bound for $\numDims$-Dimensional Augmented Independence Testing.}
A na\"{\i}ve approach to extending augmented independence testing to a $d$-dimensional distribution is to apply augmented flattening independently to all marginals, and then perform closeness testing between the flattened joint distribution and the product of the flattened marginals. However, augmented flattening increases the domain size of each coordinate by at least a constant factor (greater than 2). This incurs an additional $2^{\OO(d)}$ factor in the sample complexity for general $d$, which becomes prohibitive for large $d$. In contrast, this overhead remains harmless in low dimensions. Indeed, we leverage this approach to obtain a three-dimensional tester, which serves as a building block for our higher-dimensional construction. To achieve the general tester, we partition the $d$ coordinates into at most three groups, each having total domain size at most $\sqrt{\dimTotal}$, where $\dimTotal \coloneqq \prod_{i=1}^d n_i$. If any coordinate has a domain size greater than $\sqrt{\dimTotal}$, we isolate it as a singleton set and aggregate the remaining coordinates into a second group. After this partitioning, we first invoke a two- or three-dimensional augmented tester to check independence across the resulting groups. If this test outputs \textsc{accept}, it remains to verify independence within each group. By construction, each group has a total domain size of at most $\sqrt{\dimTotal}$, which allows us to learn the empirical distribution over each group up to $\ell_1$ error of $\Theta(\tol)$ using only $\OO(\sqrt{\dimTotal}/\tol^2)$ samples.

By explicitly bounding the error introduced in learning the marginals, we show that for a product distribution the learned empirical distribution is close to the product of its marginal as well. Therefore, independence testing can be carried out on the empirical distribution. 
This final step is entirely computational and requires no additional samples.

\paragraph{Lower Bound for $\numDims$-Dimensional Augmented Independence Testing.}
Our lower bound for $\numDims$-dimensional augmented independence testing follows from a reduction from the two-dimensional augmented independence testing problem. Although this reduction does not apply to all distributions, we show that a $\numDims$-dimensional tester can nevertheless distinguish the specific hard instances used to establish the lower bound in the two-dimensional case. In particular, we show that a two-dimensional distribution can be reshaped into a $\numDims$-dimensional distribution while preserving both its distance from a product distribution and to the prediction. This argument relies on a structural property of the hard instances constructed in the two-dimensional setting, namely that the corresponding product distributions are uniform. As a consequence, we can solve two-dimensional augmented independence testing by reshaping two-dimensional samples into $\numDims$-dimensional samples, applying a $\numDims$-dimensional tester, and returning its output.

%% file: 01.3-relatedWorks.tex
\paragraph{Distribution Testing.}
Since its introduction by \cite{RubinfeldS96, GGR98} and \cite{Batu01}, a long line of work has studied algorithms for testing different properties of distributions. Central problems under this umbrella include identity testing \citep{ValiantV14, DiakonikolasKN14}, uniformity testing \citep{Paninski08} and closeness testing \cite{BatuFRSW00, AcharyaDJOPS12, BatuFRSW13, ChanDVV14}. These properties and others have also been studied in, e.g., 
\citep{BatuKR04, Valiant11, GR11, ILR12, LRR13, DaskalakisDSV13, ADK15, DiakonikolasK16, AliakbarpourBR16, CanonneDGR18, aliakbarpourGPRY18, AliakbarpourKR19, AS20}. For a survey of the field, we refer the reader to \cite{Rub12, GoldreichBook17, CanonneSurvey20, Canonne22book}.

\paragraph{Independence Testing.}
Testing the independence of two distributions has been studied since the foundational work of \cite{pearson1900x}, which proposed the famed $\chi^2$-test. Modern research has focused on developing sample-efficient algorithms. In particular, \cite{Batu01} and \cite{LRR13} provided the first algorithms with sublinear sample complexity. \cite{ADK15} and \cite{DiakonikolasK16} refined their results, with the latter providing tight sample complexity bounds for for two-dimensional and $d$-dimensional independence testing in the standard setting.
Related problems include $k$-wise independent testing \citep{alon2007testing, rubinfeld2010testing} and testing conditional independence \citep{canonneDKS18}.

\paragraph{Augmented Algorithms for Distribution Testing.}
To bypass the standard lower bounds of distribution testing, recent works have assumed stronger access models beyond random sampling. 
These include exact or approximate oracles of the probability mass function \citep{canonne2014testing, PRS2018, eden2021learning}, as well as conditional oracles that sample from a subset of the domain \citep{CanonneRS:2015}. 
Most relevant to our work, \cite{AIRS24} initiated the study of augmented distribution testing, where the tester is given a prediction of the true distribution and a suggestion of that prediction's accuracy. Recently, \cite{ABCR25} extended this framework to the private setting. In this work, we build upon the model of \cite{AIRS24}, studying augmented algorithms for independence testing.

%% file: 03-upperbounds.tex
\section{Upper Bounds for Bivariate Augmented Independence Testing}
\label{sec:indepTestUB}

\subsection{Augmented Flattening for Independence Testing}
\label{sec:augFlat2D}
The augmented flattening technique of \cite{AIRS24} described in Section~\ref{sec:background:flattening} applies to distributions over a one-dimensional domain $[\dimPOne]$. In this section, we combine this approach with the product-of-flattenings technique of \cite{DiakonikolasK16} to develop an augmented flattening technique for distributions over a two dimensional domain $[\dimPOne] \times [\dimPTwo]$. 

Assume that we are given a distribution $p$ over $[\dimPOne] \times [\dimPTwo]$ with marginals $p_1$ and $p_2$. To flatten $p$, we draw $\Poi(s_1)$ samples from $p_1$ to form a sample set $\SS_1$ and independently draw $\Poi(s_2)$ samples from $p_2$ to form a sample set $\SS_2$.
We construct two separate coordinate-wise augmented flattenings, $F_1$ and $F_2$, where the number of buckets $b_i^{(1)}$ and $b_j^{(2)}$ are chosen according to \eqref{eqn:background-augFlatBuckets}, with $F_1$ using prediction $\pred_1$ and sample set $\SS_1$, and $F_2$  using prediction $\pred_2$ and sample set $\SS_2$.
We then define a global flattening $F$ over $[\dimPOne] \times [\dimPTwo]$ by assigning $b^{(1)}_i b^{(2)}_j$ buckets to each pair $(i, j)$.
This is equivalent to the flattening created by mapping each $(i, j)$ to a tuple $(i, j, k, \ell)$ where $k$ and $\ell$ are chosen uniformly from $[b_i^{(1)}]$ and $[b_j^{(2)}]$, respectively. Consequently, when $p$ is a product distribution, $p^{(F)} = p_1^{(F_1)} \times p_2^{(F_2)}$.




\subsection{Bivariate Augmented Independence Tester} \label{sec:2D}

In this section, we provide an augmented independence tester for bivariate distributions that achieves optimal sample complexity. Given sample access to a distribution $\dist$, full knowledge of a prediction $\pred$ and an estimate of the prediction's accuracy $\predError$, the algorithm constructs an augmented flattening of the domain $[\dimPOne] \times [\dimPTwo]$ according to the procedure described in Section~\ref{sec:augFlat2D}.

It then tests if the estimate $\predError$ was reliable by comparing the norms of the flattened marginals with their expected value when $\TV{\dist}{\pred} \leq \predError$. 
If it finds the estimate to be unreliable, it outputs \textsc{inaccurate information}. Otherwise, it performs a second validation step to ensure that the norms of the flattened joint and product distributions are sufficiently close. If they are not, then the joint distribution must be far from a product distribution and the tester outputs \textsc{reject}. If they are, then the algorithm compares the two flattened distributions using the closeness tester of \cite{ChanDVV14} and outputs the result of this test.

The pseudocode of the tester is given in Algorithm~\ref{alg:augIndTest2D}. In Theorem~\ref{thm:2D-UB}, we prove its correctness and its sample complexity.

\begin{algorithm}[!ht]
    \caption{Augmented Independence Tester for Two-Dimensional Distributions}
    \label{alg:augIndTest2D}
    
    \LinesNumbered
    \SetAlgoLined
    
    \SetKwProg{AugTest}{Augmented-Independence-Tester-2D}{:}{end}
    \SetKwFunction{EstNorm}{Estimate-\ensuremath{\ell_2^2}} 
    \SetKwFunction{CloseTest}{Closeness-Tester}
    
    \KwIn{sample access to a distribution $\dist$ over $[\dimPOne] \times [\dimPTwo]$ with marginals $p_1$ and $p_2$; prediction $\pred$; estimate of prediction error $\predError$; proximity parameter $\tol$; confidence $\probFailure$}
    
    \AugTest{$\left(\dimPOne, \dimPTwo, \dist, \pred, \predError, \tol, \probFailure = 0.1 \right)$}{
        $c, c' \leftarrow$ constants
    
        $s_1 \leftarrow \min \left( \dimPOne^{2/3} \dimPTwo^{1/3} \predError^{1/3} \tol^{-4/3}, \dimPOne \predError \right)$  , \quad
        $s_2 \leftarrow \dimPTwo \predError$ 
        
        $\hat{s}_1 \leftarrow \Poi(s_1)$ , \quad
        $\hat{s}_2 \leftarrow \Poi(s_2)$ 
        
        \If{ $\hat{s}_1 > c' s_1$ or $\hat{s}_2 > c' s_2$ }{ 
            \KwRet \textsc{reject} \label{line:reject-PoiSamples} 
        }
        
        $\SS_1 \leftarrow \hat{s}_1$ samples from $\dist_1$ , \quad 
        $\SS_2 \leftarrow \hat{s}_2$ samples from $\dist_2$ 
        
        $\tau_1 \leftarrow \frac{2 \predError}{s_1} + \frac{4}{\dimPOne}$ \label{line:setTau} , \quad
        $\tau_2 \leftarrow \frac{2 \predError}{s_2} + \frac{4}{\dimPTwo}$ 
        
        \For{$i \in [\dimPOne]$ and $j \in [\dimPTwo]$}{     
            construct $b_i^{(1)}, b_j^{(2)}$ as in \eqref{eqn:background-augFlatBuckets} with sample sets $\SS_1$ and $\SS_2$ 
        }
        $F_1, F_2, F \leftarrow$ flattenings of $[\dimPOne], [\dimPTwo], [\dimPOne] \times [\dimPTwo]$ given buckets $b_i^{(1)}$ for all $i$ and $b_j^{(2)}$ for all $j$ \; \tcp{See Section~\ref{sec:augFlat2D}} 

        $L_1 \leftarrow$ \EstNorm $\left(\distOneFlat,  \sum_{i=1}^\dimPOne b_i^{(1)}, 1/120 \right)$  \tcp{See Fact~\ref{fact:estimateNorm}}
        
        $L_2 \leftarrow$ \EstNorm $\left(\distTwoFlat, \sum_{j=1}^\dimPTwo b_j^{(2)}, 1/120 \right)$ 
        
        \If{ $L_1 > c \tau_1$ or $L_2 > c \tau_2$ \label{line:testNormsP1andP2} }{
            \KwRet \textsc{inaccurate information} \label{line:inaccurate} 
        }
        
        $L \leftarrow$ \EstNorm $\left(\distFlat, \left( \sum_{i=1}^\dimPOne b_i^{(1)}\right) \left( \sum_{j=1}^\dimPTwo b_j^{(2)} \right) ,  1/120 \right)$ 
        
        \If{$L > 10 c^2 \tau_1 \tau_2$ \label{line:testNormP}}{
            \KwRet \textsc{reject}
            \label{line:reject-normsFar}}
            
        \KwRet \CloseTest $\left( \dist^{(F)}, \distOneFlat \times \distTwoFlat, 4 c^{2} \tau_1 \tau_2, \tol, 1/80 \right)$ \label{line:return-closeness} 
        \tcp{See Fact~\ref{fact:closenessTest}} 
    }
\end{algorithm}

\begin{theorem}
   \label{conj:AugTest2D-UB}
    Let $\dimPOne, \dimPTwo \in \nats$ such that $\dimPOne \geq \dimPTwo \geq 2$. For all $\tol, \probFailure \in (0, 1), \predError \in [0, 1]$, unknown distributions $\distLB$ over $[\dimPOne] \times [\dimPTwo]$ and known distributions $\pred$ over $[\dimPOne] \times [\dimPTwo]$, Algorithm~\ref{alg:augIndTest2D} is an $(\predError, \tol, \probFailure=0.1)$-augmented independence tester with sample complexity
    $
        O\left(\max\left\{\frac{\sqrt{nm}}{\tol^2},\frac{n^{2/3}m^{1/3}\predError^{1/3}}{\tol^{4/3}}\right\}\right)
    $.
    \label{thm:2D-UB}
\end{theorem}

\begin{proof}
~\paragraph{Correctness}

We show that Algorithm~\ref{alg:augIndTest2D} satisfies the three requirements of an augmented independence tester, described in Definition~\ref{def:augIndTester}. Throughout this proof, we set the constants to be $c = 120$ and $c' = 160$. 

The correctness proof of Algorithm~\ref{alg:augIndTest2D} begins by showing that, with high probability, the estimates of the $\ell_2^2$ norms of each flattened distribution are accurate. Conditioned on this events, the algorithm satisfies the three requirements of an augmented independence tester: (i) if the prediction is $\predError$-close to $\dist$, the algorithm does not output \textsc{inaccurate information} w.h.p.; (ii) if $\dist$ is far from a product distribution, the algorithm does not output \textsc{accept}; (iii) if $\dist$ is a product distribution, the algorithm does not output \textsc{reject}. 


By Fact~\ref{fact:estimateNorm}, the estimates $L_1, L_2$ and $L$ of the norms of the flattened distributions $\distOneFlat, \distTwoFlat$ and $\distFlat$, respectively, are within a constant factor of the true norms with high probability. Specifically, by Fact~\ref{fact:estimateNorm} and a union bound, the following hold with probability at least $1-1/40$:
\begin{equation}
    \frac{\norm{\distOneFlat}_2^2}{2} \leq L_1 \leq \frac{3\norm{\distOneFlat}_2^2}{2} ~~~\text{,}~~~ \frac{\norm{\distTwoFlat}_2^2}{2} \leq L_{2} \leq \frac{3\norm{\distTwoFlat}_2^2}{2}
    ~~~\text{and}~~~ 
    \frac{\norm{\distFlat}_2^2}{2} \leq L \leq \frac{3\norm{\distFlat}_2^2}{2}.
    \label{eqn:estimate-norm-succeeds}
\end{equation}
We will condition on this event throughout the remainder of the proof.

We now prove that when the prediction is $\predError$-close to the distribution, the probability of outputting \textsc{inaccurate information} conditioned on \eqref{eqn:estimate-norm-succeeds} is at most $0.025$. Assume that $\TV{\dist}{\pred} \leq \predError$. Observe the algorithm may only output \textsc{inaccurate information} in Line~\ref{line:inaccurate}, when $L_1 > c \tau_1$ or $L_2 > c \tau_2$. 
Applying a union bound and the fact that, by conditioning on \eqref{eqn:estimate-norm-succeeds}, we can upper bound $L_1$ and $L_2$ by constant multiples of $\norm{\distOneFlat}_2^2$ and $\norm{\distTwoFlat}_2^2$, the probability of outputting \textsc{inaccurate information} is at most:
\begin{align}
    &\Pr{\textsc{inaccurate information} \mid \TV{\dist}{\pred} \leq \predError, \eqref{eqn:estimate-norm-succeeds}} \\
    &\quad = \Pr{L_1 > c \tau_1 ~\text{or}~ L_2 > c \tau_2 \mid \TV{\dist}{\pred} \leq \predError, \eqref{eqn:estimate-norm-succeeds}} \\
    &\quad \leq \Pr{\frac{3}{2}\norm{\distOneFlat}_2^2 > c \tau_1 \mid \TV{\dist}{\pred} \leq \predError} \nonumber \\
    & \quaaaaaad + \Pr{\frac{3}{2}\norm{\distTwoFlat}_2^2 > c\tau_2 \mid \TV{\dist}{\pred} \leq \predError} 
    \label{eqn:2D-UB-correctness-probIIUnionBound}
\end{align}
By Fact~\ref{fact:augFlatCloseness}, in the case that $\TV{\dist}{\pred} \leq \predError$,
the expected $\ell_2^2$-norms of $\distOneFlat$ and $\distTwoFlat$ are upper bounded by $\tau_1$ and $\tau_2$, respectively. Combined with a Markov Inequality, this fact implies a bound on \eqref{eqn:2D-UB-correctness-probIIUnionBound} of:
\begin{align}
    &\Pr{\textsc{inaccurate information} \mid \TV{\dist}{\pred} \leq \predError} \\
    &\quaad \leq \frac{\E{ \frac{3}{2} \norm{\distOneFlat}_2^2 \mid \TV{\dist}{\pred} \leq \predError}}{c \tau_1} + \frac{\E{ \frac{3}{2} \norm{\distTwoFlat}_2^2 \mid \TV{\dist}{\pred} \leq \predError}}{c \tau_2} \\
    &\quaad \leq \frac{3\tau_1}{2c\tau_1} + \frac{3\tau_1}{2 c \tau_2} \\
    &\quaad = 1/40,
    \label{eqn:2D-UB-correctness-probIIMarkov}
\end{align}
where the final equality holds by substituting $c = 120$.

Next, we show that conditioned on \eqref{eqn:estimate-norm-succeeds}, the failure probability of the closeness tester is at most $1/80$. By Fact~\ref{fact:closenessTest}, this holds provided the parameter $b = 4 c^2 \tau_1 \tau_2$ upper bounds minimum of the $\ell_2^2$-norms of the two input distributions. 
Our algorithm guarantees that this condition holds by validating the norms of the marginals and the distribution itself in Line~\ref{line:testNormsP1andP2} and Line~\ref{line:testNormP}. 
Specifically, the closeness tester is only invoked at Line~\ref{line:return-closeness} only if Line~\ref{line:testNormsP1andP2} and Line~\ref{line:testNormP} evaluate to false. This implies that when the closeness tester is invoked, we have the following bounds:
\begin{equation}
    L_1 \leq c\tau_1, ~~~ L_2 \leq c \tau_2, ~~~\text{and}~~~ L \leq 10c^2 \tau_1 \tau_2.
\end{equation}
Combined with \eqref{eqn:estimate-norm-succeeds}, these bounds imply bounds on the squared norms of the flattened distributions:
\begin{equation}
    \norm{\distOneFlat}_2^2 \leq 2c\tau_1, ~~~ \norm{\distTwoFlat}_2^2 \leq 2c\tau_2, ~~~\text{and}~~~ \norm{\distFlat}_2^2 \leq 20c^2 \tau_1 \tau_2.
\end{equation}
Further, because the $\ell_2^2$-norm of $\distOneFlat \times \distTwoFlat$ is the product of the squared norms of its marginals, it is upper bounded by $4 c^2 \tau_1 \tau_2$. 
Therefore, the minimum of the squared norms of the distributions given to the standard closeness tester is at most $4 c^2 \tau_1 \tau_2$, which matches our setting of the parameter $b$. We can therefore apply Fact~\ref{fact:closenessTest} and guarantee that the failure probability of the closeness tester is at most $1/80$.

We now prove that if $\dist$ is $\tol$-far from a product distribution, Algorithm~\ref{alg:augIndTest2D} outputs \textsc{accept} with probability at most $1/80$. Observe that the algorithm outputs \textsc{accept} only in Line~\ref{line:return-closeness}, if the standard closeness tester on $\distFlat$ and $\distOneFlat \times \distTwoFlat$ outputs \textsc{accept}. 
However, when $\dist$ is $\tol$-far from all product distributions, $\distFlat$ must be $\tol$-far from $\distOneFlat \times \distTwoFlat$, due to the following observations: (i) if $\dist$ is $\tol$-far from any product distribution, then it is necessarily $\tol$-far from the product of its marginals, $\dist_1 \times \dist_2$; (ii) by the construction of $F$, $\prodDistFlat = \distOneFlat \times \distTwoFlat$; (iii) flattening preserves total variation distance \citep{DiakonikolasK16}. Therefore, the inputs to the closeness tester are $\tol$-far and the probability that the closeness tester outputs \textsc{accept} is at most $1/80$.

Finally, we show that if $\dist$ is a product distribution, the algorithm outputs \textsc{reject} with probability at most $1/40$, conditioned on \eqref{eqn:estimate-norm-succeeds}. Observe that the algorithm will output \textsc{reject} if any of the following events occur: (i) $\hat{s}_1 > c' s_1$ or $\hat{s}_2 > c' s_2 $ (Line~\ref{line:reject-PoiSamples}); (ii) the closeness tester outputs \textsc{reject} (Line~\ref{line:return-closeness}); (iii) $L > 10c^2  \tau_1 \tau_2$ (Line~\ref{line:reject-normsFar});. We show that, conditioned on \eqref{eqn:estimate-norm-succeeds} and $\dist$ being a product distribution, (i) and (ii) occur with low probability, while (iii) does not occur.

First, recall that $\hat{s}_1$ and $\hat{s}_2$ are Poisson random variables with means $s_1$ and $s_2$, respectively. By Markov's Inequality and a union bound, we can therefore bound the probability that $\hat{s}_1 > c' s_1$ or $\hat{s}_2 > c' s_2$ by $1/80$. Further, if $\dist = \dist_1 \times \dist_2$, then $\distFlat = \prodDistFlat$, and the closeness tester will output \textsc{reject} with probability at most $1/80$. Finally, to prove that (iii) will not occur, observe that if $\dist$ is a product distribution, then the flattened distribution $\distFlat$ is equal to the product of its flattened marginals, $\distOneFlat \times \distTwoFlat$. In particular, the two distributions must have the same squared norm. However, as we argued when proving that the closeness tester succeeds, we can bound the squared norm of the product of the flattened marginals as 
\begin{equation}
    \norm{\distOneFlat \times \distTwoFlat}_2^2 \leq 4 c^2 \tau_1^2 \tau_2^2.
    \label{eqn:2D-UB-correctness-normOfProd}
\end{equation}
Conditioned on \eqref{eqn:estimate-norm-succeeds}, the estimate $L$ is upper-bounded by a constant multiple of the $\ell_2^2$-norm of $\distFlat$. Combining this fact with \eqref{eqn:2D-UB-correctness-normOfProd}, we have
\begin{equation}
    L \leq \frac{3}{2} \norm{\distFlat}_2^2  = \frac{3}{2} \norm{\distOneFlat \times \distTwoFlat}_2^2 \leq 6 c^2 \tau_1^2 \tau_2^2.
\end{equation}
Therefore, if $\dist$ is a product distribution, $L > 10  c^2 \tau_1^2 \tau_2^2$ will never evaluate to true and the algorithm will not output \textsc{reject} at Line~\ref{line:reject-normsFar}. Ultimately, by a union bound over events (i) and (ii), the probability that the algorithm outputs \textsc{reject} on a product distribution is at most $1/80$, conditioned on \eqref{eqn:estimate-norm-succeeds}.

We have thus far shown that the estimates of the norms satisfy \eqref{eqn:estimate-norm-succeeds} with probability at least $1-1/40$. Conditioned on this event, we have shown that: (i) the tester outputs inaccurate information when $\TV{\dist}{\pred} \leq \predError$ with probability at most $1/40$; (ii) the tester outputs \textsc{accept} when $\dist$ is at least $\tol$-far from product with probability at most $1/80$; and (iii) the tester outputs \textsc{reject} when $\dist$ is a product distribution with probability at most $1/40$. By a union bound over the failure of the estimates of the norms and the conditional failure of the remainder of the tester, the total failure probability in each case is at most $1/20$. Therefore, Algorithm~\ref{alg:augIndTest2D} is an $(\predError, \tol, 0.1)$-augmented independence tester. 




~ \paragraph{Sample complexity}

Algorithm~\ref{alg:augIndTest2D} must draw samples from $\dist$ in three instances: when flattening the distributions, when estimating the squared norms of the flattened distributions, and when running the standard closeness tester on the flattened distributions.

To construct the flattening, we draw at most $c's_1  + c's_2  = \OO(s_1 + s_2)$ samples. 
Given the construction of the flattening $F_1$ described in Fact~\ref{fact:augFlatCloseness}, the domain size of $\distOneFlat$ is:
\begin{equation}
    \sum_{i=1}^\dimPOne \left( \floor{\dimPOne \pred_1(i)} + N_i + 1 \right)
    \leq \sum_{i=1}^\dimPOne \left( \dimPOne \pred_1(i)+ 1 \right) + \sum_{i=1}^\dimPOne N_i + \sum_{i=1}^\dimPOne 1 = 3\dimPOne + s_1 = \OO(\dimPOne).
\end{equation}
Similarly, the domain size of $\distTwoFlat$ is at most $3\dimPTwo + s_2 = \OO(\dimPTwo)$. Consequently, the domain size of $\distFlat$ is $\OO(\dimPOne \dimPTwo)$.

By Fact~\ref{fact:estimateNorm}, estimating the $\ell_2$-norm of a distribution over a domain of size $n$ requires $\OO\left(\sqrt{n}\right)$ samples. Therefore, to estimate the norms of our flattened distributions $\distOneFlat, \distTwoFlat$ and $\distFlat$ we must draw  $\OO\left(\sqrt{\dimPOne}\right), \OO\left(\sqrt{\dimPTwo}\right)$ and $\OO\left(\sqrt{\dimPOne \dimPTwo}\right)$ samples, respectively.

Applying Fact~\ref{fact:closenessTest}, the standard closeness tester draws $\OO( \dimPOne \dimPTwo \sqrt{\tau_1 \tau_2} / \tol^2)$ samples from both $\distFlat$ and $\distOneFlat \times \distTwoFlat$. Note that we may draw one sample from $\distOneFlat \times \distTwoFlat$ by drawing two independent samples from $\distFlat$ and returning the first coordinate of one point and the second coordinate of the other. Therefore, the overall sample complexity of the closeness tester is $\OO({\dimPOne \dimPTwo}\sqrt{\tau_1 \tau_2} / \tol^2)$.

Ultimately, substituting $\tau_1$ and $\tau_2$ with their realizations (in Line~\ref{line:setTau}) and including each component, the sample complexity of Algorithm~\ref{alg:augIndTest2D} is:
\begin{align}
    \OO &\left( s_1 + s_2 + \sqrt{\dimPOne} + \sqrt{\dimPTwo} + \sqrt{\dimPOne \dimPTwo} + \frac{\dimPOne \dimPTwo}{\tol^2} \sqrt{\left(\frac{2\predError}{s_1}+\frac{4}{\dimPOne}\right)\left(\frac{2\predError}{s_2}+\frac{4}{\dimPTwo} \right)} \right) \\
    &= \OO \left( s_1 + \dimPTwo \predError +  \sqrt{\dimPOne \dimPTwo} + \frac{\dimPOne \dimPTwo}{\tol^2} \sqrt{\left(\frac{\predError}{s_1}+\frac{1}{\dimPOne}\right)\left(\frac{\predError}{\dimPTwo \predError}+\frac{1}{\dimPTwo} \right)} \right) \\
    &= \OO \left( s_1 + \dimPTwo \predError +  \sqrt{\dimPOne \dimPTwo} + \frac{\dimPOne \sqrt{\dimPTwo}}{\tol^2} \sqrt{\left(\frac{\predError}{s_1}+\frac{1}{\dimPOne}\right)} \right),
\end{align}
where the second line follows by substituting $s_2 = \dimPTwo \predError$. Recall that $s_1 = \min\left( \frac{\dimPOne^{2/3} \dimPTwo^{1/3} \predError^{1/3}}{\tol^{4/3}}, \dimPOne \predError \right)$. Substituting each of these possibilities, the overall sample complexity of Algorithm~\ref{alg:augIndTest2D} is:
\begin{equation}
    \label{eq:2Dsample_complexity}
    \OO \left( \max\left( \frac{\dimPOne^{2/3} \dimPTwo^{1/3} \predError^{1/3}}{\tol^{4/3}}, \frac{\sqrt{\dimPOne \dimPTwo}}{\tol^2} \right)\right).
\end{equation}


\end{proof}

%% file: appendix/app09-HighDimensionalTester.tex
\section{Upper Bounds for Multivariate Augmented Independence Testing}
\label{sec: d_Dim_UB}

We first note that the two-dimensional augmented independence tester can be extended to the three-dimensional setting as a direct generalization of the two-dimensional case; the description and sample complexity analysis of this 3D tester are deferred to Appendix~\ref{ApxSec: 3D_Tester}.
The key idea underlying our high-dimensional construction is to partition the coordinates into at most three groups in a nearly balanced manner, such that every group containing more than one coordinate has a total domain size of at most $\sqrt{\dimTotal}$, where $\dimTotal$ denotes the size of the full domain. Since the number of groups is at most three, we can use either a two- or three-dimensional augmented independence tester to verify independence among the groups. Moreover, the bound on each group’s domain size ensures that we can efficiently learn an empirical distribution that is $\tol$-close to the true distribution by using only $\OO(\sqrt{\dimTotal}/\tol^2)$ samples by Fact~\ref{fact: EmpiricalDist}. Independence within each group is then tested using Algorithm~\ref{alg:test-ind-by-learning}, whose correctness and sample complexity are analyzed in Section~\ref{subsec: TestIndByLearning}. 
We now describe the partitioning strategy that completes the construction of Algorithm~\ref{alg:aug-ind-d} as an augmented independence tester for $d$-dimensional distributions. Specifically, we isolate the coordinate with the largest domain size (assumed to be $n_1$) and partition the remaining coordinates into one or two additional groups, depending on whether \(n_1\) is larger or smaller than \(\sqrt{\dimTotal}\):

If \(n_1 \ge \sqrt{\dimTotal}\), the remaining coordinates are grouped into a single block, \(B = \{2,3,\dots,d\}\), whose induced domain size satisfies \(\dimTotal_B \le \sqrt{\dimTotal}\). In this case, the distribution \(\distHighDim\) is viewed as a two-dimensional distribution over \([n_1] \times [\dimTotal_B]\), and a two-dimensional augmented independence tester is applied to test independence between the first coordinate and the block \(B\). Conditioned on acceptance, the algorithm then invokes Algorithm~\ref{alg:test-ind-by-learning} to verify independence among the coordinates in \(B\).

Otherwise, if \(n_1 < \sqrt{\dimTotal}\), the remaining coordinates are partitioned into two disjoint blocks \(B\) and \(C\). Let $t$ be the smallest index such that $\prod_{i=1}^t n_i > \sqrt{\dimTotal}$ while $\prod_{i=1}^{t{-}1} n_i \le \sqrt{\dimTotal}$. We define $B=\{2,3,\ldots,t\}$ and $C=\{t{+}1,\ldots,d\}$. The first inequality ensures that $\dimTotal_C \le \sqrt{\dimTotal}$, and multiplying the second inequality by $n_t/n_1$ (given $n_t\le n_1$) ensures that $\dimTotal_B \le \sqrt{\dimTotal}$. The distribution is then viewed as a three-dimensional distribution over \([n_1] \times [\dimTotal_B] \times [\dimTotal_C]\), and a three-dimensional augmented independence tester is applied. If this test outputs \textsc{accept}, Algorithm~\ref{alg:test-ind-by-learning} is subsequently used to test independence within blocks \(B\) and \(C\) separately, and the algorithm outputs \textsc{accept} only if both tests succeed.

\begin{algorithm}
    \caption{Augmented Independence Tester for $d$-Dimensional Distributions}
    \label{alg:aug-ind-d}

    \LinesNumbered
    \SetAlgoLined

    \SetKwProg{AugTestD}{Augmented-Independence-Tester-$d$-Dimensional}{:}{end}
    \SetKwFunction{AugTestTwoD}{Augmented-Independence-Tester-2D}
    \SetKwFunction{AugTestThreeD}{Augmented-Independence-Tester-3D}
    \SetKwFunction{TestLearn}{TestIndependenceByLearning}
    \SetKwFunction{Proj}{\ensuremath{\pi}}

    \KwIn{sample access to $\distHighDim$ over $\prod_{i=1}^d [n_i]$ (assume $n_1 \ge n_2 \ge \cdots \ge n_d \ge 2$); prediction $\predHighDim$; estimate of prediction error $\predError \in [0,1]$; proximity parameter $\tol \in (0,1)$; failure probability $\probFail = 0.1$}

    \AugTestD{$(\distHighDim,\predHighDim,\predError,\tol,\probFail)$}{

        $\tolPrime \leftarrow \tol/12$  
        
        $\probFailPrime \leftarrow \probFail/5$  
        
        $\dimTotal \leftarrow \prod_{i=1}^d n_i$  

        \uIf{$n_1 \ge \sqrt{\dimTotal}$}{
            $B \leftarrow \{2,3,\dots,d\}$  
            $\dimTotal_B \leftarrow \prod_{i\in B} n_i$  
            
            view $\distHighDim$ as a two-dimensional distribution over $[n_1] \times \prod_{i\in B}[n_i]$  
            
            $out \leftarrow$ \AugTestTwoD$\big(n_1, \dimTotal_B, \distHighDim, \predHighDim, \predError, \tolPrime, \probFailPrime\big)$  
            \tcp{Test if $\distHighDim = \distHighDim_1 \times \distHighDim_B$}

            \If{$out \in \{\textsc{Inaccurate Information},\textsc{Reject}\}$}{
                \KwRet $out$  
            }

            $\distHighDim_B \leftarrow \Proj_B(\distHighDim)$  
            \tcp{Project samples of $\distHighDim$ onto coordinates in $B$}

            \KwRet \TestLearn$\big(B,\distHighDim_B,\tolPrime,\probFailPrime\big)$  
            \tcp{Test $\distHighDim_B = \prod_{i\in B}\distHighDim_i$}
        }
        \Else{
            choose the smallest $t$ such that $\prod_{i=1}^{t} n_i \ge \sqrt{\dimTotal}$  

            $B \leftarrow \{2,\dots,t\}$,\quad
            $C \leftarrow \{t{+}1,\dots,d\}$  

            $\dimTotal_B \leftarrow \prod_{i\in B} n_i$,\quad
            $\dimTotal_C \leftarrow \prod_{i\in C} n_i$  

            view $\distHighDim$ as a three-dimensional distribution over
            $[n_1]\times \prod_{i\in B}[n_i]\times \prod_{i\in C}[n_i]$  
            
            $out \leftarrow$ \AugTestThreeD$\big(n_1,\dimTotal_B,\dimTotal_C,\distHighDim,\predHighDim,\predError,\tolPrime,\probFailPrime\big)$  
            \tcp{Test if $\distHighDim = \distHighDim_1 \times \distHighDim_B \times \distHighDim_C$}

            \If{$out \in \{\textsc{Inaccurate Information},\textsc{Reject}\}$}{
                \KwRet $out$  
            }

            $\distHighDim_B \leftarrow \Proj_B(\distHighDim)$ 
            \tcp{Project samples of $\distHighDim$ onto coordinates in $B$}
            
            $\distHighDim_C \leftarrow \Proj_C(\distHighDim)$  
            \tcp{Project samples of $\distHighDim$ onto coordinates in $C$}

            $out_B \leftarrow$ \TestLearn$\big(B,\distHighDim_B,\tolPrime,\probFailPrime\big)$  
            \tcp{Test $\distHighDim_B = \prod_{i\in B}\distHighDim_i$}

            $out_C \leftarrow$ \TestLearn$\big(C,\distHighDim_C,\tolPrime,\probFailPrime\big)$  
            \tcp{Test $\distHighDim_C = \prod_{i\in C}\distHighDim_i$}

            \uIf{$out_B=\textsc{Reject}$ \textbf{or} $out_C=\textsc{Reject}$}{
                \KwRet \textsc{Reject}  
            }
            \Else{
                \KwRet \textsc{Accept}  
            }
        }
    }
\end{algorithm}

\subsection{Testing Independence By Learning}
\label{subsec: TestIndByLearning}
In this section, we describe a learning-based procedure for testing independence within a set of coordinates $S$ by constructing an empirical histogram of the induced distribution over $S$, computing its univariate marginals, and verifying whether the learned distribution factors as the product of its marginals. The pseudo-code of our algorithm is provided in Algorithm~\ref{alg:test-ind-by-learning}, and its correctness and sample complexity are analyzed in Theorem~\ref{thrm: LearningError}.

\begin{algorithm}
    \caption{\textsc{TestIndependenceByLearning}}
    \label{alg:test-ind-by-learning}

    \LinesNumbered
    \SetAlgoLined

    \SetKwProg{TestLearn}{TestIndependenceByLearning}{:}{end}
    \SetKwFunction{TVFunc}{\ensuremath{\|\cdot\|_{\mathsf{TV}}}}

    \KwIn{sample access to $\distHighDim_S$ over $\prod_{i\in S}[n_i]$; index set $S \subseteq [d]$ with $\dimTotal_S \coloneqq \prod_{i\in S} n_i$; proximity parameter $\tol \in (0,1)$; failure probability $\probFail$}

    \TestLearn{$(S,\distHighDim_S,\tol,\probFail)$}{

        $\probFailPrime \leftarrow \probFail/(|S|{+}1)$  
        
        $\eta \leftarrow \tol/7$  
        \tcp{Budget: learning accuracy $\eta$ and product-check slack}

        $t \leftarrow \Theta\!\left(\frac{\dimTotal_S+\log(1/\probFailPrime)}{\eta^2}\right)$  
        \tcp{Histogram learning on a domain of size $\dimTotal_S$}

        Draw $t$ i.i.d.\ samples $X_S^{(1)},\dots,X_S^{(t)} \sim \distHighDim_S$  

        Construct the empirical distribution $\distEmpHighDim_S$ over $\prod_{i\in S}[n_i]$ from $\{X_S^{(j)}\}_{j=1}^t$  

        Compute the univariate marginals $\{\distEmpHighDim_i\}_{i\in S}$ and the product distribution
        $\distEmpHighDim_{\Pi} \leftarrow \prod_{i\in S}\distEmpHighDim_i$  

        $\Delta \leftarrow \|\distEmpHighDim_S - \distEmpHighDim_{\Pi}\|_\mathsf{TV}$  
        \tcp{Exact computation since $\distEmpHighDim_S$ is explicit}

        \uIf{$\Delta \le 6\eta$}{
            \KwRet \textsc{Accept}  
        }
        \Else{
            \KwRet \textsc{Reject}  
        }
    }
\end{algorithm}


\begin{theorem}
    \label{thrm: LearningError}
    Let $S$ be a set of coordinates, and let $\distHighDim_S$ be a distribution over $\prod_{i\in S} [n_i]$ with total domain size $\dimTotal_S \coloneqq \prod_{i\in S} n_i$. For any proximity parameter $\tol \in(0,1)$ and constant failure probability $\probFail$, there exists an algorithm which, given $t=\Theta\left(\frac{\dimTotal_S}{\tol^2}\right)$ i.i.d. samples from $\distHighDim_S$, distinguishes with probability at least $1{-}\probFail$ between the case that $\distHighDim_S$ is a product distribution and the case that $\distHighDim_S$ is $\tol$-far (in total variation distance) from any product distribution.
\end{theorem}
\begin{proof}
    We prove the theorem by showing that Algorithm~\ref{alg:test-ind-by-learning} correctly distinguishes product distributions from distributions that are $\tol$-far from independence using the stated number of samples. To this end, we first show that the empirical distribution and its marginals are close to the true distribution and its corresponding marginals in Lemma~\ref{lem: EmpiricalDist_B}. Then this fact is used to prove the correctness of Algorithm~\ref{alg:test-ind-by-learning}.
\begin{restatable}{lem}{EmpDist}
\label{lem: EmpiricalDist_B}
Let $S\subset[d]$ be any set of coordinates, and let $\distHighDim_S$ denote the induced distribution over $\prod_{i\in S} [n_i]$, whose domain has size $\dimTotal_S \coloneqq \prod_{i\in S} n_i$. For any proximity parameter $\tolDPrime\in(0,1)$ and constant failure parameter $\probFail \ge 10^{-4}$, using $t=\frac{\dimTotal_S+\log((|S|{+}1)/\probFail)}{\tolDPrime^2}$ i.i.d. samples, we can construct the empirical distribution $\distEmpHighDim_S$ such that, with probability at least $1{-}\probFail$,  the following two statements hold \emph{simultaneously}:
\begin{enumerate}[label=(\roman*)]
    \item \(\|\distEmpHighDim_S-\distHighDim_S\|_\mathsf{TV} \leq \tolDPrime.\) \label{eq: Joint_TV_distance}
    \item \(\sum_{i\in S}\|\distEmpHighDim_i-\distHighDim_i\|_\mathsf{TV} \leq 5\tolDPrime,\) \label{eq: Sum_Marginal_TV_distances}
\end{enumerate}
Here $\distEmpHighDim_i$ and $\distHighDim_i$ denote the univariate marginals of $\distEmpHighDim_S$ and $\distHighDim_S$ on coordinate $i$, respectively.
\end{restatable}
\begin{proof}
Set $\probFailPrime \coloneqq \probFail/(|S|{+}1)$.
We will prove that the two inequalities \ref{eq: Joint_TV_distance}--\ref{eq: Sum_Marginal_TV_distances} hold \emph{simultaneously} with probability at least $1{-}\probFail$ by establishing $|S|{+}1$ high-probability events and applying a union bound.\\
Let $\mathcal{E}_0=\{\|\distEmpHighDim_S-\distHighDim_S\|_\mathsf{TV} \leq \tolDPrime\}$ denote the event in~\ref{eq: Joint_TV_distance}. Apply Fact~\ref{fact: EmpiricalDist} to the induced distribution $\distHighDim_S$ over a domain of size $\dimTotal_S$ with failure probability $\probFailPrime$. With the stated choice of $t$, this yields
\[
\Pr{\mathcal{E}_0}=\Pr{\|\distEmpHighDim_S-\distHighDim_S\|_{\mathsf{TV}} \le \tolDPrime} \ge 1-\probFailPrime.
\]
A trivial bound on the total variation distance across all marginals is $\sum_{i\in S}\|\distEmpHighDim_i-\distHighDim_i\|_\mathsf{TV} \leq |S|{\cdot}\tolDPrime$, which is generally too loose for our purposes. To establish the tighter upper bound stated in \ref{eq: Sum_Marginal_TV_distances}, we leverage the fact that for each $i\in S$, the univariate marginal $\distEmpHighDim_i$ of the empirical histogram $\distEmpHighDim_S$ coincides with the empirical distribution obtained by projecting the same $t$ samples onto coordinate $i$. Therefore, we may apply Corollary~\ref{cor:empirical_bound} to the univariate marginal $\distHighDim_i$ (whose domain size is $n_i$) with failure probability $\probFailPrime$, obtaining for every $i\in S$,
\[
\Pr{\|\distEmpHighDim_i-\distHighDim_i\|_{\mathsf{TV}} \le \sqrt{\frac{n_i+\log\!\left(\frac{|S|+1}{\probFail}\right)}{t}}} \ge 1-\probFailPrime.
\]
Let $\mathcal{E}_i$ denote the above event.
By a union bound over the $|S|{+}1$ events $\mathcal{E}_0$ and $\{\mathcal{E}_i\}_{i\in S}$,
\[
\Pr{\mathcal{E}_0 \cap \bigcap_{i\in S}\mathcal{E}_i} \ge 1{-}(|S|{+}1)\probFailPrime = 1{-}\probFail.
\]
Condition on this intersection event. Summing the marginal bounds over $i\in S$ gives
\[
\sum_{i\in S}\|\distEmpHighDim_i-\distHighDim_i\|_{\mathsf{TV}}
\le
\sum_{i\in S} \sqrt{\frac{n_i+\log\!\left(\frac{|S|+1}{\probFail}\right)}{t}}.
\]
Since $t\ge \dimTotal_S/\tolDPrime^2$, we have $1/\sqrt{t} \le \tolDPrime/\sqrt{\dimTotal_S}$. Thus,

\begin{align*}
    \sum_{i\in S} \sqrt{\frac{n_i+\log\!\left(\frac{|S|+1}{\probFail}\right)}{t}}
    &\le \tolDPrime \sum_{i\in S} \sqrt{\frac{n_i+\log\!\left(\frac{|S|+1}{\probFail}\right)}{\dimTotal_S}}\\
    &\le \tolDPrime \sum_{i\in S}\left( \sqrt{\frac{n_i}{\dimTotal_S}} + \sqrt{\frac{\log\!\left(\frac{|S|+1}{\probFail}\right)}{\dimTotal_S}} \right) \tag{($\sqrt{a+b} \le \sqrt{a}+\sqrt{b}$)}\\
    &= \tolDPrime \sum_{i\in S} \left( \sqrt{\frac{1}{\prod_{j\in S,\, j\neq i} n_j}} + \sqrt{\frac{\log\!\left(\frac{|S|+1}{\probFail}\right)}{\prod_{j\in S} n_j}} \right)
    && \tag{definition of $\dimTotal_S$}\\
    &\le \tolDPrime \left(\frac{|S|}{\sqrt{2}^{\,|S|-1}} + \frac{|S| \sqrt{\log\!\left(\frac{|S|{+}1}{\probFail}\right)}}{\sqrt{2}^{\,|S|}}\right)
    && \tag{$\forall j,\; n_j \ge 2$}\\
    &\le 5\tolDPrime.
\end{align*}
where the bracketed expression is increasing as $\probFail$ decreases, so it suffices to bound the expression at $\probFail=10^{-4}$. To control the dependence on $|S|$, consider numerical sequences $a_d \coloneqq \frac{d}{\sqrt{2}^{d-1}}$ and $b_d \coloneqq \frac{d \sqrt{\log((d+1)/\probFail)}}{\sqrt{2}^d}$. A direct ratio test shows that for all $d\ge3$ we have $a_{d+1}/a_d<1$ and $b_{d+1}/b_d<1$, and hence the bracketed expression is decreasing in $|S|$ for $|S|\ge3$. Therefore, its maximum over all admissible $|S|$ is attained at $|S|=3$. Substituting $|S|=3$ and $\probFail=10^{-4}$ yields the numerical bound 5.\\
To summarize, on the intersection event $\mathcal{E}_0 \cap \bigcap_{i\in S}\mathcal{E}_i$ (which holds with probability at least $1{-}\probFail$), both \ref{eq: Joint_TV_distance} and \ref{eq: Sum_Marginal_TV_distances} hold simultaneously, concluding the proof.
\end{proof}
\textbf{Correctness proof for Algorithm~\ref{alg:test-ind-by-learning}:}
First, suppose that $\distHighDim_S$ is a product distribution, i.e. $\distHighDim_S = \prod_{i\in S} \distHighDim_i$. We will show that its empirical histogram $\distEmpHighDim_S$ is $6\tolDPrime$-close to the product of its marginals, $\distEmpHighDim_{\Pi} \coloneqq \prod_{i\in S} \distEmpHighDim_i$, with probability at least $1{-}\probFail$ using Lemma~\ref{lem: EmpiricalDist_B}.
By the triangle inequality and a standard hybrid (telescoping) argument for product measures, we have:
    \begin{align}
        \left\|\distEmpHighDim_S-\distEmpHighDim_{\Pi} \right\|_\mathsf{TV} &\leq \left\|\distEmpHighDim_S-\distHighDim_S\right\|_\mathsf{TV} + \left\|\distHighDim_S-\distEmpHighDim_{\Pi} \right\|_\mathsf{TV} \notag\\
        &= \left\|\distEmpHighDim_S-\distHighDim_S\right\|_\mathsf{TV} + \left\|\prod_{i\in S}\distHighDim_i-\prod_{i\in S}\distEmpHighDim_i\right\|_\mathsf{TV} \notag\\
        &\leq \left\|\distEmpHighDim_S-\distHighDim_S\right\|_\mathsf{TV} + \sum_{i\in S} \left\|\distHighDim_i-\distEmpHighDim_i\right\|_\mathsf{TV} \notag\\
        &\leq \tolDPrime + 5\tolDPrime = 6\tolDPrime 
        \label{eq: 6Eta_closeness}
    \end{align}
Next, suppose that $\distHighDim_S$ is $\tol$-far from any product distribution, specifically $\distEmpHighDim_\Pi$. We can prove that $\distEmpHighDim_S$ is $(\tol{-}\tolDPrime)$-far from the product of its marginals with probability at least $1{-}\probFail$:
    \begin{align}
        \left\|\distEmpHighDim_S-\distEmpHighDim_{\Pi} \right\|_\mathsf{TV} &= \left\|\left(\distHighDim_S-\distEmpHighDim_{\Pi}\right)-\left(\distHighDim_S-\distEmpHighDim_S\right)\right\|_\mathsf{TV} \notag\\
        &\geq \left\|\distHighDim_S-\distEmpHighDim_\Pi\right\|_\mathsf{TV} - \left\|\distHighDim_S-\distEmpHighDim_S\right\|_\mathsf{TV} \notag \\
        &> \tol {-} \tolDPrime 
        \label{eq: Eps-Eta_farness}
    \end{align}
Combining \eqref{eq: 6Eta_closeness} and \eqref{eq: Eps-Eta_farness}, if we choose $\tolDPrime=\tol/7$, then the completeness region and the soundness region do not overlap. Hence, with probability at least $1{-}\probFail$, Algorithm~\ref{alg:test-ind-by-learning} correctly tests independence over the coordinates in $S$.

\textbf{Analysis of sample complexity:} We first note that once the empirical distribution $\distEmpHighDim_S$ is estimated, all subsequent steps of the algorithm are purely computational and require no additional samples from $\distHighDim_S$. The number of samples required to learn $\distEmpHighDim_S$ is:\\
\begin{align*}
    \OO\left(\frac{\dimTotal_S+\log((|S|+1)/\probFail)}{\tol^2}\right) = \OO\left(\frac{\dimTotal_S+\log(|S|+1)+\log(1/\probFail)}{\tol^2}\right) = \OO\left(\frac{\dimTotal_S}{\tol^2}\right),
\end{align*}
where the final equality follows since $\dimTotal_S = \prod_{i\in S} n_i \ge 2^{|S|}$ and hence $\log(|S|+1) = O(\log \dimTotal_S) = o(\dimTotal_S)$, and $\probFail$ is a constant.
\end{proof}

\subsection{Analysis of the Multivariate Augmented Independence Tester}
\label{subsec:d_Dim_correctness}

\begin{theorem}
    \label{thrm: HighDim}
    Let $\distHighDim$ be a distribution over $\prod_{i=1}^d[n_i]$, and $\predHighDim$ a prediction of $\distHighDim$ over the same domain. Assume w.l.o.g. that $n_1 \ge n_2 \ge \cdots \ge n_d \ge 2$. Let $\dimTotal \coloneqq \prod_{i=1}^d n_i$ denote the total domain size. For all $\tol\in(0,1)$ and $\predError\in[0,1]$, there exists an algorithm with sample access to $\distHighDim$ and explicit access to $\predHighDim$ that is an $(\predError,\tol,\probFail=0.1)$-augmented independence tester with sample complexity $O\left(\max\left\{ \frac{\sqrt{\dimTotal}}{\tol^2}, \frac{n_1^{1/3} \dimTotal^{1/3}\predError^{1/3}}{\tol^{4/3}} \right\}\right)$.
\end{theorem}
\begin{proof} 
    We first prove that Algorithm~\ref{alg:aug-ind-d} succeeds in augmented independence testing with high probability, and then upper bound its sample complexity.
    
    \noindent\textbf{Proof of correctness:}
    We verify that Algorithm~\ref{alg:aug-ind-d} satisfies the three conditions in Definition~\ref{def:augIndTester} to prove the correctness:\\
    (i) If $\distHighDim$ is a product distribution, the algorithm can incorrectly \textsc{reject} at two steps. First, the algorithm invokes either the 2D or 3D augmented independence tester. By the correctness of these testers, the probability of rejection in this step is at most $\probFailPrime/2 = 0.01$. Second, the algorithm may reject during the learning-based independence tests on the sets $B$ or $C$. Since $\distHighDim$ is a product distribution, its marginals $\distHighDim_B$ and $\distHighDim_C$ are also product distributions. By Theorem~\ref{thrm: LearningError}, each call to Algorithm~\ref{alg:test-ind-by-learning} rejects with probability at most $\probFailPrime = 0.02$. By a union bound over all possible rejection events, the total probability that Algorithm~\ref{alg:aug-ind-d} outputs \textsc{Reject} is at most $(0.01+0.02+0.02) \le \frac{\probFail}{2}$\\
    (ii) if $\distHighDim$ is $\tol$-far from any product distribution, we demonstrate that one of the sequential tests detects this farness and outputs \textsc{accept} only with low probability. Given that the algorithm's global acceptance follows a conjunctive rule, this individual low probability serves as an upper bound on the total probability that the algorithm outputs \textsc{accept}. We present the argument for the 3D case; the 2D case follows analogously. Consider two arbitrary product distributions over sets $B$ and $C$, denoted as $r_B = \prod_{i\in B} r_i$ and $r_C = \prod_{i\in C} r_i$. Form a product distribution over the same domain of $\distHighDim$ as $r=\distHighDim_1\times r_2 \times \ldots \times r_d$. Then, we have:\\
    \begin{align}
        \tol \le \|\distHighDim-r\| &\le \|\distHighDim-\distHighDim_1 \times \distHighDim_B \times \distHighDim_C\| + \|\distHighDim_1 \times \distHighDim_B \times \distHighDim_C - \distHighDim_1 \times r_B \times r_C\| \notag \\
        &\le \|\distHighDim-\distHighDim_1 \times \distHighDim_B \times \distHighDim_C\| + \| \distHighDim_B - r_
        B\| + \| \distHighDim_C - r_C\|
        \label{eq: soundness_proof}
    \end{align}
    Therefore, at least one of the three terms on the right-hand side of \eqref{eq: soundness_proof} must be greater than $\tol/3$. We analyze each case: 
    \begin{itemize}
        \item If $\|\distHighDim-\distHighDim_1 \times \distHighDim_B \times \distHighDim_C\| \ge \tol/3$, then by contrapositive of Lemma~\ref{lem:FarFromMargImpliesFarProd}, $\distHighDim$ is $\frac{\tol}{12}$-far from any product distribution over the domain $[n_1]\times[\dimTotal_B]\times[\dimTotal_C]$. Since the 3D augmented tester is called with parameters $\tolPrime = \frac{\tol}{12}$ and $\probFailPrime = \frac{\probFail}{5}$, it will \textsc{accept} with probability at most $\frac{\probFailPrime}{2}=0.01$
        \item If $\| \distHighDim_B - r_B\| \ge \tol/3 > \tolPrime=\tol/12$, then $\distHighDim_B$ is $\tolPrime$-far from any product distribution. By Theorem~\ref{thrm: LearningError}, Algorithm~\ref{alg:test-ind-by-learning} applied to set $B$ outputs \textsc{accept} with probability at most $\probFailPrime = 0.02$.
        \item If $\| \distHighDim_C - r_C\| \ge \tol/3$, then the same argument applies to set $C$, and Algorithm~\ref{alg:test-ind-by-learning} outputs \textsc{accept} with probability at most $\probFailPrime = 0.02$.
    \end{itemize}
    (iii) Algorithm~\ref{alg:aug-ind-d} can output \textsc{inaccurate information} only during the initial invocation of the 2D or 3D augmented independence tester. Since total variation distance contracts under marginalization, any global TV guarantee immediately implies the same guarantee for the marginals used by the 2D/3D call. By the correctness guarantees of these testers, this occurs with probability at most $\probFailPrime/2 = 0.01$. This completes the proof of correctness.\\
    
    \noindent \textbf{Analysis of sample complexity:} In the first case we know that $n_1 \geq \sqrt{\dimTotal} \geq \dimTotal_B$. Therefore, Theorem~\ref{conj:AugTest2D-UB} can be used to calculate the sample complexity for performing the two-dimensional augmented testing with proximity parameter $\tolPrime = \tol/12$ as follows:
    \begin{align}
        \label{eq: SC_2D_Tester}
        \OO\left(\max\left\{ \frac{\sqrt{n_1 \dimTotal_B}}{(\tolPrime)^2},\frac{n_1^{2/3} \dimTotal_B^{1/3}\predError^{1/3}}{(\tolPrime)^{4/3}}\right\}\right) 
        = \OO\left(\max\left\{ \frac{\sqrt{\dimTotal}}{\tol^2},\frac{\dimPOne_1^{1/3} \dimTotal^{1/3}\predError^{1/3}}{\tol^{4/3}}\right\}\right)
    \end{align}
    In the second case we know that $n_1,\dimTotal_B, \dimTotal_C \le \sqrt{\dimTotal}$ by construction. Let $n_2 = N_B, n_3=N_C$ denote the domain sizes in sample complexity of Algorithm~\ref{alg:augIndTest3D}. The sample complexity for performing the three-dimensional augmented testing with proximity parameter $\tolPrime = \tol/12$ is as follows:
   \begin{align}
        \label{eq: SC_3D_Tester}
        \OO\left(\max_l\left\{ \frac{\sqrt{\dimTotal}}{(\tolPrime)^2},\frac{n_l^{1/3} \dimTotal^{1/3}\predError^{1/3}}{(\tolPrime)^{4/3}}\right\}\right)
        \leq  \OO\left(\max\left\{ \frac{\sqrt{\dimTotal}}{\tol^2},\frac{\dimTotal^{1/6} \dimTotal^{1/3}\predError^{1/3}}{\tol^{4/3}}\right\}\right) = \OO \left( \frac{\sqrt{\dimTotal}}{\tol^2} \right)
    \end{align}
    For both cases, by Theorem~\ref{thrm: LearningError}, the number of samples required to learn the empirical distribution of $\distHighDim_B$ or $\distHighDim_C$ with accuracy $\tolDPrime = \tolPrime/7 = \tol/84$ is also bounded by
    \begin{align}
        \label{eq: SC_Learner}
        \OO\left(\frac{\dimTotal_{B}}{\tolDPrime^2}\right) + \OO\left(\frac{\dimTotal_{C}}{\tolDPrime^2}\right)\leq \OO\left(\frac{\sqrt{\dimTotal}}{\tol^2}\right),
    \end{align}
    where the last inequality follows by $\dimTotal_B, \dimTotal_C \le \sqrt{\dimTotal}$.
     Consequently, the sample complexity of Theorem~\ref{thrm: HighDim} is achieved by combining \eqref{eq: SC_2D_Tester}, \eqref{eq: SC_3D_Tester}, and \eqref{eq: SC_Learner}.
\end{proof}  

%% file: appendix/app10-2D-LB.tex
In this section, we establish a lower bound for augmented independence testing of two-dimensional distributions that matches the upper bound given in Section~\ref{sec:2D}.
\begin{theorem}
    \label{thm:2D-LB}
    Let $\distLB$ be a distribution over $[\dimPOne]\times[\dimPTwo]$ with $\dimPOne \geq \dimPTwo \geq 2$ and $\log \dimPOne \leq \dimPTwo$. Suppose $\predError \in [0, 1]$ and $\tol \in (0, 1)$ are given. Any $(\predError,\tol,\probFailure \coloneqq 0.1)$-augmented independence tester must use $\Omega\left(\max\left\{\frac{\sqrt{\dimPOne\dimPTwo}}{\tol^2},  \frac{\dimPOne^{2/3}\dimPTwo^{1/3}\predError^{1/3}}{\tol^{4/3}}\right\}\right)$ samples to succeed.
    
\end{theorem}

\begin{proof}

We distinguish between two cases, depending on the relationship between $\dimPOne, \dimPTwo, \predError$ and $ \tol$. Each cases establishes one term of the desired lower bound. To prove the first term, we give a reduction from standard independence testing. To prove the second term of the lower bound, we use an information-theoretic approach to show that distinguishing between two families of distributions---with one family consisting of product distributions and the other of distributions that are far from any product distribution---requires at least the stated sample complexity.

\subsection{Case 1: $\boldsymbol{\alpha \leq \frac{\sqrt{\dimPTwo}}{\sqrt{\dimPOne}\tol^2}}$.}
\label{sec:2D-LB-reduction}
    In this case, the dominant factor of the lower bound is $\Omega\left( \sqrt{\dimPOne \dimPTwo}/{\tol^2} \right)$. This bound follows from a reduction using the hard instance of standard independence testing. Specifically, let $\AA_{aug}$ be an augmented independence tester. We construct a standard independence tester, $\AA_{std}$, that draws samples from a distribution $\dist$ and passes these samples to $\AA_{aug}$. $\AA_{std}$ sets the prediction for $\AA_{aug}$ to be the uniform distribution, the accuracy estimate to $\predError \leq \sqrt{\dimPTwo}/ (\dimPTwo \tol^2)$, the proximity parameter to $\tol$, and the failure probability to $\probFailure$. If $\AA_{aug}$ outputs \textsc{reject} or \textsc{inaccurate information}, $\AA_{std}$ outputs \textsc{reject}. Otherwise, if $\AA_{aug}$ outputs \textsc{accept}, $\AA_{std}$ outputs \textsc {accept}.

    Existing lower bounds for standard independence testing (e.g. \cite{ADK15, DiakonikolasK16}) construct a hard instance consisting of the uniform distribution and a distribution which is $\tol$-far from any product distribution. We consider this same hard instance.
    
    In the \textsc{accept} case of the hard instance, when the distribution is uniform, the distance between the distribution and the prediction is $0$ and therefore satisfies $\TV{\dist}{\pred} \leq \predError$. Consequently, $\AA_{aug}$ outputs \textsc{inaccurate information} with probability at most $\probFailure / 2$. It also outputs \textsc{reject} with probability at most $\probFailure / 2$, implying that the probability that $\AA_{std}$ outputs \textsc{reject} is at most $\probFailure$. In the \textsc{reject} case of the hard instance, when the distribution is $\tol$-far from a product distribution, $\AA_{aug}$---and consequently, $\AA_{std}$---outputs \textsc{accept} with probability at most $\probFailure / 2$. Therefore, $\AA_{std}$ successfully tests the independence of this hard instance and the lower bound for standard independence testing of $\Omega\left( \sqrt{\dimPOne \dimPTwo}/{\tol^2} \right)$ extends to augmented independence testing.

\subsection{Case 2: $\boldsymbol{\alpha > \frac{\sqrt{\dimPTwo}}{\sqrt{\dimPOne}\tol^2}}$.}

\label{sec:2D-LB-mututalInfo}
    In this case, the dominant factor of the lower bound is 
    \begin{equation}
        \Omega\left( \frac{\dimPOne^{2/3} \dimPTwo^{1/3} \predError^{1/3}}{\tol^{4/3}}\right).
    \end{equation}
    We prove this bound by constructing two families of distributions that an augmented independence tester must be able to differentiate, then showing that these families are information-theoretically indistinguishable unless the number of samples is sufficiently large.

    Let $\numSamples$ be an integer such that there exists a sufficiently small constant $c > 0$ such that
    \begin{equation}
        \numSamples \leq c \cdot \left(\frac{\dimPOne^{2/3}\dimPTwo^{1/3}\predError^{1/3}}{\tol^{4/3}}\right).
        \label{eqn:chooseSmallK}
    \end{equation}
    Given $\numSamples$, we can construct the hard instance of augmented independence testing as follows. 
    Let $\predLB$ be the uniform distribution over $[\dimPOne] \times [\dimPTwo]$ and let $\tolMeas \in (0, 1)$ and $\predErrorMeas \in [0, 1]$, to be defined later. Then, construct $\distLB$ according to the following procedure:
    \begin{enumerate}
        \item Let $X$ be either 0 or 1 with equal probability.
        \item Draw $c_i$ as:
            \begin{equation}
                c_i = 
                \begin{cases}
                    \frac{1}{\numSamples} & \text{w.p.} ~ \frac{\predErrorMeas\numSamples}{\dimPOne} \\
                    \frac{1}{\dimPOne} & \text{w.p.} ~ 1 - \frac{\predErrorMeas \numSamples}{\dimPOne}
                \end{cases}.
                \label{eqn:drawC_i}
            \end{equation}
        \item If $X=0$, let $\prodLB(i, j) = \frac{1}{\dimPTwo}$ for all $i \in [\dimPOne], j \in [\dimPTwo]$.
        \item If $X=1$, for all $i \in [\dimPOne], j \in [\dimPTwo]$, draw $\farLB(i, j)$ as:
            \begin{equation}
                \farLB(i, j) = 
                \begin{cases}
                    \frac{1}{\dimPTwo} & \text{ if } c_i = \frac{1}{\numSamples} \\
                    \frac{1+ \tolMeas_{ij}}{\dimPTwo} & \text{ if } ~ c_i = \frac{1}{\dimPOne}, \text{ where } \tolMeas_{ij} = 
                    \begin{cases}
                        \tolMeas & \text{ w.p. } 1/2 \\
                        -\tolMeas & \text{ w.p. } 1/2  \\
                    \end{cases}
                \end{cases}.
                \label{eqn:drawMeasure}
            \end{equation}
        \item For all $i \in [\dimPOne], j \in [\dimPTwo]$, define $\prodMeasLB(i, j)$ as: 
            \begin{equation}
                \prodMeasLB(i, j) = c_i \cdot \prodLB(i, j)
                \label{eqn:defineQMeasure}
            \end{equation}
        \item Construct a distribution $\distLB$ as:
            \begin{equation}
                \distLB(i, j) = \frac{\prodMeasLB(i, j)}{s_i\cdot C},
                \label{eqn:normalizeMeasure}
            \end{equation}
        where
            \begin{equation}
                s_i = \sum_{j=1}^\dimPTwo \prodLB(i, j) \quad \text{ and } \quad C = \sum_{i=1}^\dimPOne c_i. 
                \label{eqn:measureNormalizationTerms}
            \end{equation}
    \end{enumerate}

    Assume that there exists an algorithm, $\AA$, for augmented independence testing which succeeds with probability at least $0.9$. Given a set of $\SS$ samples from a distribution $\dist$ over $[\dimPOne] \times [\dimPTwo]$, prediction $\pred$, and estimate of the prediction error $\predError$, $\AA$ returns \textsc{accept}, \textsc{reject}, or \textsc{inaccurate information}. 

    We will show that, upon receiving samples from $\distLB$, this tester must distinguish between $X=0$ and $X=1$ with high probability.
    However, instead of drawing a sample set $\SS$ from $\distLB$ directly, we draw $a_{i, j} \sim \Poi(\numSamples \cdot \prodMeasLB(i, j))$ and let $a_{i, j}$ be the number of copies of $(i, j)$ in $\SS$ for all $(i, j) \in [\dimPOne] \times [\dimPTwo]$. As long as $\abs{\SS} \leq \numSamples / 100$, this process is equivalent to drawing $\OO(\numSamples)$ samples from $\distLB$. 
    In Lemma~\ref{lem:SValid}, we show that, with probability at least 0.99: (i) $\abs{\SS} \geq \numSamples / 100$; (ii) if $X=0$, $\distLB$ is a product distribution and $\predError$-close to the prediction $\pred$; 
    (iii) if $X=1$, then $\distLB$ is far from a product distribution. 

    \begin{restatable}{lem}{SValid}
    \label{lem:SValid}
        Let $\dimPOne \geq \dimPTwo$ be sufficiently large integers such that $\dimPTwo \geq \log(\dimPOne)$. Let $\numSamples$ be an integer and let $\tol\in (0, 1/192), \predError \in (0, 1)$. Assume $\tolMeas = 192\tol$ and $\predErrorMeas = 2\predError/3$. Given $\tolMeas, \predErrorMeas$,  let $X$, $\prodMeasLB$ and $\distLB$ be drawn as in equations \eqref{eqn:drawC_i} to \eqref{eqn:measureNormalizationTerms}. Let $\predLB$ be the uniform distribution over $[\dimPOne] \times [\dimPTwo]$.
        Let $\SS$ be a multiset such that the number of copies of $(i, j) \in \SS$ is $\Poi(k \cdot \prodMeasLB(i, j))$ for all $(i, j) \in [\dimPOne] \times [\dimPTwo]$. Then, with probability at least $0.95$, the following events all hold:
        \begin{enumerate}
            \item $\abs{\SS} \geq \numSamples / 100$.
            \item If $X=0$, $\distLB$ is a product distribution and is at least $\predError$-close to $\predLB$.
            \item If $X=1$, $\distLB$ is at least $\tol$-far from all product distributions.
        \end{enumerate}
    \end{restatable}
    The proof of Lemma~\ref{lem:SValid} is deferred to Appendix~\ref{Apdx: SValid}. If $\distLB$ and $\SS$ meet all three criterion described in the lemma, we say that $(\distLB, \SS)$ is \textit{valid}. When $(\distLB, \SS)$ is \textit{valid}, we observe that $\AA$ can distinguish between $X=0$ and $X=1$ with probability at least $0.925$. Specifically, assume that there exists an algorithm $\AA'$ that receives $\SS$ from a valid construction $(\distLB, \SS)$, along with a prediction $\predLB$ and parameters $\tolMeas$ and $\predErrorMeas$ and outputs
    \begin{equation}
        \AA'(\SS, \pred, \predError, \tol) = \begin{cases}
            0 & \text{ if } ~ \AA(\SS, \pred, \predErrorMeas, \tolMeas) = \textsc{accept} \\
            1 & \text{ if } ~ \AA(\SS, \pred, \predErrorMeas, \tolMeas) = \textsc{reject} \text{ or } \textsc{inaccurate information}.
        \end{cases}
        \label{eqn:constructA'}
    \end{equation}
    If $X = 0$, then, because $(\distLB, \SS)$ is valid, $\distLB$ is a product distribution and $\predLB$ is $\predErrorMeas$-close to $\distLB$. As $\AA$ is an $(\tolMeas, 0.1)$-augmented independence tester, this implies that the probability that $\AA$ outputs  \textsc{reject} or \textsc{inaccurate information} is at most $0.05$. Applying a union bound, the probability that $\AA$ outputs neither \textsc{reject} nor \textsc{inaccurate information} is at least 0.9. Similarly, if $X=1$, $\distLB$ is $\tol'$-far from a product distribution and the probability that $\AA$ outputs \textsc{accept} is at most 0.05. By the law of total probability, we ultimately have:
    \begin{equation}
        \Pr{\AA'(\SS, \pred, \predErrorMeas, \tolMeas) = X \mid (\distLB, \SS) ~\text{valid}} \geq 0.925.
        \label{eqn:mutualInfoLB-successLowerBound}
    \end{equation}

    However, by an information-theoretic argument, we can prove that \textit{no} algorithm recovers $X$ on input $\SS$ with probability greater than $0.51$. Specifically, we will show that, given $\numSamples$ as chosen in \eqref{eqn:chooseSmallK}, the mutual information between the sample set and $X$ is $o(1)$. A standard result in information theory then implies that no algorithm can recover $X$ with better-than-random probability.

    Represent $\SS$ as a set of $\dimPOne$ random count vectors $A_i = (a_{i,1}, a_{i,2},\cdots, a_{i,\dimPTwo})$.
    Then, we will show that the mutual information between $X$ and this set is small.
    First, we apply the chain rule for mutual information and exploit the conditional independence of $A_i$'s given $X$:
    \begin{equation}
        \label{eqn:mutualInfoLB-chainRule}
        I(X;A_1,A_2,\cdots,A_\dimPOne) 
        = \sum_{i=1}^\dimPOne I(X;A_i\mid A_{1:i-1}) 
        \leq \sum_{i=1}^\dimPOne I(X;A_i)=n I(X;A), 
    \end{equation}
    As each $A_i$ follows the same distribution, let $A$ represent an arbitrary $A_i$ and rewrite \eqref{eqn:mutualInfoLB-chainRule} as:
    \begin{equation}
        \label{eqn:mutualInfoLB-oneA}
        I(X;A_1,\cdots,A_\dimPOne) \leq \dimPOne \cdot I(X;A).
    \end{equation}
    In Lemma~\ref{lem: MI}, we bound the mutual information between $X$ and $A$:
    \begin{restatable}{lem}{MI}
        \label{lem: MI}
        Let $\predError \in [0, 1], \tol \in (0, 1), \numSamples \in \mathbb{N}$.
        Let $X$ be either $0$ or $1$ with equal probability. Let $\prodMeasLB$ be a measure over $[\dimPOne] \times [\dimPTwo]$ constructed as in \eqref{eqn:defineQMeasure} using $X$. For an arbitrary $i \in [\dimPOne]$, let $A \coloneqq A_i$ be a vector of $\dimPTwo$ counts, such that $a_{i, j} \sim \Poi(\numSamples \prodMeasLB(i, j))$ for all $j \in [\dimPTwo]$.
        Then, the mutual information between $X$ and $A$ is bounded as
        \begin{equation}
            I(X;A) \leq \OO\left( \frac{\tol^4 \numSamples^3}{\dimPOne^3 \dimPTwo \predError} \right).
        \end{equation}
    \end{restatable}
    \noindent We defer the proof of Lemma~\ref{lem: MI} to Appendix~\ref{Apdx: MI_lemma}.
    Combining Lemma~\ref{lem: MI} with \eqref{eqn:mutualInfoLB-oneA}, we have the following bound on $I(X;A_1,\cdots,A_\dimPOne) $:
    \begin{equation}
        I(X;A_1,\cdots,A_n) \leq \dimPOne \cdot \OO\left(\frac{\tol^4\numSamples^3}{\dimPOne^3\dimPTwo\predError}\right)
        \leq \OO\left(\frac{\tol^4\numSamples^3}{\dimPOne^2\dimPTwo\predError}\right).
        \label{eqn:mutualInfoLB-bigOh}
    \end{equation}
    When $\numSamples$ is $o \left(\frac{\dimPOne^{2/3}\dimPTwo^{1/3}\predError^{1/3}}{\tol^{4/3}}\right)$, as established in \eqref{eqn:chooseSmallK},  this bound implies that $I(X;A_1,\cdots,A_n)  = o(1)$. 
    We can thus apply Fact~\ref{fact:mutualInfoMustBeLarge} and conclude that no algorithm can infer the value of $X$ with a probability exceeding 0.51. Specifically, the probability that $\AA'$ can recover $X$ on any sample set $\SS$, valid or not, is less than 0.51. Combining this bound with the fact that $(\distLB, \SS)$ is valid with probability at least 0.99 by Lemma~\ref{lem:SValid}, we have the following bound on the probability that $\AA$ recovers $X$ on a valid input:
    \begin{align}
        \Pr{\AA'(\SS, \pred, \predErrorMeas, \tolMeas) = X \mid (\distLB, \SS) ~\text{valid}} 
        &= \frac{\Pr{\AA'(\SS, \pred, \predErrorMeas, \tolMeas) = X, ~ (\distLB, \SS) ~\text{valid}}}{\Pr{(\distLB, \SS) ~\text{valid}}} \\
        &\leq \frac{\Pr{\AA'(\SS, \pred, \predErrorMeas, \tolMeas) = X}}{\Pr{(\distLB, \SS) ~\text{valid}}} \\
        &< \frac{0.51}{0.95}
    \end{align}
    This contradicts with \eqref{eqn:mutualInfoLB-successLowerBound}, which established, under the assumption that an augmented independence tester succeeds on $\SS$, that $\AA'$ can recover $X$ on a valid input with probability at least $0.925$. Therefore, there is no augmented independence tester that receives $\numSamples = o \left(\frac{\dimPOne^{2/3}\dimPTwo^{1/3}\predError^{1/3}}{\tol^{4/3}}\right)$ samples and succeeds on $\distLB$.
    \end{proof}

%% file: appendix/app11-highD-LB.tex
    In this section, we give lower bounds for augmented independence testing in higher dimensions by proving that $\numDims$-dimensional independence testing is at least as hard as two-dimensional independence testing. These lower bounds are tight, matching the upper bounds given in Theorem~\ref{thrm: HighDim}.
    \begin{theorem}
    \label{thm:LB-d-dim-testing}
        Let $\distLB$ be a distribution over the product domain  $\prod_{i=1}^d[\dimPOne_i]$, and $\predLB$ be a prediction of $\distLB$ over the same domain. Assume w.l.o.g. that $n_1 \geq n_2 \geq \ldots \geq n_\numDims \geq 2$. Let $\dimTotal \coloneqq \prod_{i=1}^d n_i$ denote the total domain size. Fix parameters $\predError \in [0,1]$, and $\tol \in (0,1/3)$. Any $(\predError,\tol,\probFailure)$-augmented independence tester must use $\Omega\left(\max_j\left\{ \frac{\sqrt{\dimTotal}}{\tol^2},\frac{\dimPOne_j^{1/3} \dimTotal^{1/3} \predError^{1/3}}{\tol^{4/3}}\right\}\right)$ samples to succeed. 
    \end{theorem}
    
    \begin{proof}
        We reduce the hard instance of two-dimensional independence testing (as described in Section~\ref{sec:2D-LB}) to $\numDims$-dimensional independence testing. Assume that there is an algorithm, $\AA_{\numDims}$, for $\numDims$-dimensional augmented independence testing of distributions over $[\dimPOne_1] \times \ldots [\dimPOne_\numDims]$. 
        Consider the following algorithm for two-dimensional augmented independence testing which has sample access to a distribution $\tilde{\distLB}$ over $[\dimPOne_1] \times \left[\prod_{i=2}^\numDims \dimPOne_i\right]$ and takes as input a prediction $\hat{\tilde{\distLB}}$, a proximity parameter $\tol$ and an estimate of the prediction error $\predError$:
        \begin{enumerate}
            \item Draw $\numSamples$ samples from $\tilde{\distLB}$.
            \item Reshape each sample to a sample over $\left[\dimPOne_1\right] \times \left[\dimPOne_2\right] \times \ldots \times \left[\dimPOne_\numDims \right]$ by leaving the first coordinate unchanged and mapping the second coordinate to a tuple in $\left[\dimPOne_2\right] \times \ldots \times \left[\dimPOne_\numDims \right]$. This mapping unflattens the second dimension into dimensions $2$ to $\numDims$, varying the last dimension the fastest.
            \item Reshape $\hat{\tilde{\distLB}}$ to $\hat{{\distLB}}$ via the same mapping.
            \item Return the result of $\AA_\numDims$ on the reshaped samples, the reshaped prediction, $\tol$ and $\predError$.
        \end{enumerate}
        
        We complete our reduction by proving that this algorithm succeeds on the hard instance of two-dimensional augmented independence testing. Let $\tilde{\distLB}$ be the hard instance of two-dimensional independence testing over $[\dimPOne_1] \times \left[\prod_{i=2}^\numDims \dimPOne_i\right]$  with proximity parameter $\tol$. If $\alpha \leq \frac{\sqrt{\prod_{i=2}^\numDims \dimPOne_i}}{\sqrt{\dimPOne_1} \tol^2}$, consider $\tilde{\distLB}$ as the hard instance of the first case of the two-dimensional lower bound, described in Section~\ref{sec:2D-LB-reduction}. Else, consider $\tilde{\distLB}$ as the hard instance of the second case of the two-dimensional lower bound, described in Section~\ref{sec:2D-LB-mututalInfo}. Let $\distLB$ be the distribution of the reshaped samples. We prove that if the $\numDims$-dimensional tester succeeds on $\distLB$, then the two-dimensional tester succeeds on $\tilde{\distLB}$. 
        
        First, observe that the reshaping described is a mapping between domain elements of the distributions and does not affect the probabilities of these elements.
        Because total variation distance is preserved under this reshaping, the error of the reshaped prediction, or its distance from the reshaped distribution, matches the error of the original prediction. 

        Then, consider the \textsc{accept} instance of each case of the two-dimensional lower bound. 
        For the first case of this bound (Section~\ref{sec:2D-LB-reduction}), recall that the \textsc{accept} instance is the uniform distribution. Under the reshaping, $\distLB$ remains uniform and, consequently, independent.
        In the second case (Section~\ref{sec:2D-LB-mututalInfo}), the \textsc{accept} instance is the product of two independent distributions, where the first marginal is random and the second is uniform. Reshaping this uniform second marginal produces a uniform---and therefore, a product---distribution over $[\dimPOne_2] \times \ldots \times [\dimPOne_\numDims]$. Because the first marginal was independent of the original second marginal, it must be independent of the reshaped marginals. Therefore, the reshaped distribution is a product distribution.

        For the \textsc{reject} instance, in both cases of the lower bound, we know that $\tilde{\distLB}$ is $\tol$-far from any product distribution and want to show that $\distLB$ is also $\tol$-far from any product distribution. We will prove the contrapositive of this implication: if there exists a product distribution $q$ over $[\dimPOne_1]~\times~\ldots~\times~[\dimPOne_\numDims]$ such that $\TV{\distLB}{q} \leq \tol$, there must exist a product distribution $\tilde{q}$ over 
        $[\dimPOne_1] \times \left[\prod_{i=2}^\numDims \dimPOne_i\right]$ such that $\TV{\tilde{\distLB}}{\tilde{q}} \leq \tol$. Assume such a distribution $q$ exists. Following the mapping described above, reshape $q$ to a distribution $\tilde{q}$ over $[\dimPOne_1] \times \left[\prod_{i=2}^\numDims \dimPOne_i\right]$. Because $q$ is a product distribution, its coordinates are all mutually independent. In particular, its first coordinate is independent of the tuple of the remaining coordinates and $\tilde{q}$, whose marginals correspond to the distributions of the  first coordinate and the tuple of remaining coordinates, is a product distribution.
        Further, because the reshaping preserves total variation distance, $\TV{\tilde{p}}{\tilde{q}} = \TV{p}{q} \leq \epsilon$. Therefore, $\tilde{p}$ is $\tol$-close to a product distribution and the contrapositive is proved.
    
        Ultimately, we have: (i) the distance between the distribution and the prediction is unchanged after the reshaping; (ii) when $\tilde{\distLB}$ is a product distribution, $\distLB$ is also a product distribution; and (iii) when $\tilde{\distLB}$ is $\tol$-far from all product distributions over its support, $\distLB$ is far from all product distributions over its 
        support. As $\AA_\numDims$ is an augmented $\numDims$-dimensional independence tester, $\AA_2$ can distinguish between $\distLB$ being a product distribution or being $\tol$-far from a product distribution. 
        In other words, a correct decision by $\AA_\numDims$ on the reshaped distribution $\distLB$ leads to a correct decision by $\AA_2$ on the hard instance $\tilde{\distLB}$. Therefore, $\AA_2$ is an augmented two-dimensional independence tester that succeeds on $\tilde{\distLB}$. We can thus lower bound the sample complexity of augmented $\numDims$-dimensional independence testing of distributions  by the sample complexity of augmented independence testing of $\tilde{\distLB}$. Substituting the bounds established in Section~\ref{sec:2D-LB}, we have that any algorithm for augmented independence testing over $[\dimPOne_1] \times \ldots [\dimPOne_\numDims]$ with proximity $\tol$ and prediction error $\predError$ requires at least $\Omega\left(\max_j\left\{ \frac{\sqrt{\dimTotal}}{\tol^2},\frac{\dimPOne_j^{1/3} \dimTotal^{1/3}\predError^{1/3}}{\tol^{4/3}}\right\}\right)$ samples.

        
    \end{proof}

%% file: 99-acknowledgements.tex
R.S. acknowledges partial support from the Ken Kennedy Institute Research Cluster Fund, the Ken Kennedy Institute Computational Science and Engineering Recruiting Fellowship, funded by the Energy HPC Conference and the Rice University Department of Computer Science, and the Ken Kennedy Institute 2025/26 Andrew Ladd
Memorial Excellence in Computer Science Graduate Fellowship.

%% file: Appendix.tex
\appendix
\section*{Appendix}
\addcontentsline{toc}{section}{Appendix}



\ifcolt

\section{Three-Dimensional Augmented Independence Testing}
\label{ApxSec: 3D_Tester}
\input{appendix/app13-ThreeDimTester}

\section{Lower Bounds for Bivariate Augmented Independence Testing}
\label{sec:2D-LB}

\input{appendix/app10-2D-LB}
\section{Lower Bounds for Multivariate Augmented Independence Testing}
\label{sec: d_Dim_LB}
\input{appendix/app11-highD-LB}
\fi

\section{Background Results}
\label{Apdx: usefulFacts}
\input{appendix/app06-usefulFacts}

\section{Three-Dimensional Augmented Independence Testing}
\label{ApxSec: 3D_Tester}

\input{appendix/app13-ThreeDimTester}
\section{Omitted Proofs from Section~\ref{sec:2D-LB}}

\subsection{Proof of Lemma~\ref{lem:SValid}}
\label{Apdx: SValid}
\input{appendix/app03-validProof}

\subsection{Proof of Lemma~\ref{lem: MI}}
\label{Apdx: MI_lemma}
\input{appendix/app04-mutualInfoProof}

\subsection{Proof of Lemma~\ref{lem: MIFormula}}
\label{Apdx: MI_Formula}
\input{appendix/app05-formulaMutualInfoProof}

\subsection{Auxiliary Lemmas}

\input{appendix/app12-informationTheoryLemmas}

%% file: appendix/app13-ThreeDimTester.tex
In this section, we consider the problem of testing independence for a three-dimensional discrete distribution $\dist$ supported on $[\dimPOne_1]\times[\dimPOne_2]\times[\dimPOne_3]$, given a prediction distribution $\pred$ over the same domain. Let $\dist_1, \dist_2,$ and $\dist_3$ denote the respective marginals of $\dist$. Without loss of generality, assume that $\dimOne \geq \dimTwo \geq \dimThree$.

\begin{algorithm}[!ht]
    \caption{Augmented Independence Tester for Three-Dimensional Distributions}
    \label{alg:augIndTest3D}

    \LinesNumbered
    \SetAlgoLined

    \SetKwProg{AugTestThreeD}{Augmented-Independence-Tester-3D}{:}{end}
    \SetKwFunction{Poi}{Poi}
    \SetKwFunction{CloseTest}{Closeness-Tester}

    \KwIn{sample access to a distribution $\dist$ over $[\dimOne]\times[\dimTwo]\times[\dimThree]$ with marginals $\dist_1,\dist_2,\dist_3$; prediction $\pred$ over the same domain; estimate of prediction error $\predError$; proximity parameter $\tol$; confidence $\probFailure$}

    \AugTestThreeD{$(\dimOne,\dimTwo,\dimThree,\dist,\pred,\predError,\tol)$}{

        $s_1 \leftarrow \min\!\left(\dimOne\predError, \dimOne^{2/3}\dimTwo^{1/3}\dimThree^{1/3}\predError^{1/3}\tol^{-4/3}\right)$ ,  
        $s_2 \leftarrow \dimTwo\predError$ , 
        $s_3 \leftarrow \dimThree\predError$  

        $\hat{s}_1 \leftarrow \Poi(s_1)$, 
        $\hat{s}_2 \leftarrow \Poi(s_2)$, 
        $\hat{s}_3 \leftarrow \Poi(s_3)$  

        \If{ $\hat{s}_1 > 240s_1 $ or $\hat{s}_2 > 240s_2 $ or $\hat{s}_3 > 240 s_3$}{ 
            \KwRet \textsc{reject} 
        }
        \For{$\ell \in \{1,2,3\}$}{
        $S_\ell \leftarrow \hat{s}_\ell$ samples from $\dist_{\ell}$  

        \For{$i_\ell \in [n_\ell]$}{     

            $N_{i_\ell} \leftarrow$ frequency of $i_\ell$ in $S_\ell$  

            $b_{i_\ell}^{(\ell)} \leftarrow \lfloor n_\ell \cdot \pred_\ell(i_\ell)\rfloor + N_{i_\ell} + 1$   \tcp{See Equation~\ref{eqn:background-augFlatBuckets}}
        }
        $F_\ell \leftarrow$ flattening of $[n_\ell]$ given buckets $b_{i_\ell}^{(\ell)}$ for all $i_\ell$  \tcp{See Section~\ref{sec:augFlat2D}}

        $L_\ell \leftarrow$ \EstNorm $\left(\dist_{\ell}^{(F_\ell)} ,\sum_{i_\ell=1}^{n_\ell} b_{i_\ell}^{(\ell)}, 1/180 \right)$   \tcp{Algorithm of Fact~\ref{fact:estimateNorm}}
        
        $\tau_\ell \leftarrow \frac{2 \predError}{s_\ell} + \frac{4}{n_\ell}$ \label{line:setTau3D} 
        }

        \If{ $\exists\, \ell \text{ s.t. } L_\ell > {180\tau_\ell}$ }{
            \KwRet \textsc{inaccurate information}  
        }
        
        $F \leftarrow$ flattening of $[\dimOne] \times [\dimTwo] \times [\dimThree]$ given buckets $b_{i_\ell}^{(\ell)}$ for all $i_\ell$ and for $\ell \in\{1,2,3\}$   

        $L \leftarrow$ \EstNorm $\left(\distFlat, \prod_{\ell=1}^{3} \left( \sum_{i_\ell=1}^{n_\ell} b_{i_\ell}^{(\ell)}\right), 1/180 \right)$  

        \If{$L > {(6\cdot180^3)\tau_1 \tau_2 \tau_3}$ }{
            \KwRet \textsc{reject} }

        \KwRet \CloseTest $\left( \dist^{(F)}, \dist_1^{(F_1)} \times \dist_2^{(F_2)} \times \dist_3^{(F_3)}, (6\cdot180^3) \cdot {\tau_1 \tau_2 \tau_3} , \tol, 1/120 \right)$ 
    }
\end{algorithm}
The correctness proof for Algorithm~\ref{alg:augIndTest3D} follows from an analogous argument to the two-dimensional case. Next, we analyze the sample complexity of this algorithm. It must draw samples from $\dist$ to flatten the distributions, estimate the squared norms of the flattened distributions, and run the standard closeness tester on the flattened distributions.

To construct the flattening, we draw at most $240(s_1  + s_2  + s_3) = \OO(s_1 + s_2 + s_3)$ samples.

Given the construction of the flattening $F_\ell$ described in Fact~\ref{fact:augFlatCloseness}, the domain size of $\dist_\ell^{(F_\ell)}$ is:
\begin{equation}
    \sum_{i_\ell=1}^{n_\ell} \left( \floor{n_\ell \cdot \pred_{\ell}(i_\ell)} + N_{i_\ell} + 1 \right)
    \leq \sum_{i_\ell=1}^{n_\ell} \left( n_\ell \cdot \pred_\ell(i_\ell)+ 1 \right) + \sum_{i_\ell=1}^{n_\ell} N_i + \sum_{i_\ell=1}^{n_\ell} 1 = 3n_\ell + s_\ell = \OO(n_\ell).
\end{equation}
Consequently, the domain size of $\distFlat$ is $\OO(n_1n_2n_3)$.
By Fact~\ref{fact:estimateNorm}, estimating the $\ell_2$-norm of a distribution over a domain of size $n$ requires $\OO\left(\sqrt{n}\right)$ samples. Therefore, to estimate the norms of our flattened distributions $\distOneFlat, \distTwoFlat, \dist_3^{(F_3)}$ and $\distFlat$ we must draw  $\OO\left(\sqrt{\dimOne}\right), \OO\left(\sqrt{\dimTwo}\right), \OO\left(\sqrt{\dimThree}\right)$ and $\OO\left(\sqrt{\dimOne \dimTwo\dimThree}\right)$ samples, respectively.

Applying Fact~\ref{fact:closenessTest}, the standard closeness tester draws $\OO( \dimOne \dimTwo \dimThree \sqrt{\tau_1 \tau_2 \tau_3} / \tol^2)$ samples from both $\distFlat$ and $\distOneFlat \times \distTwoFlat \times \dist_3^{(F_3)}$. Note that we may draw one sample from $\distOneFlat \times \distTwoFlat \times \dist_3^{(F_3)}$ by drawing three independent samples from $\distFlat$ and returning the first coordinate of one point, the second coordinate of another, and the third coordinate of the last point. Therefore, the overall sample complexity of the closeness tester is $\OO({\dimOne \dimTwo \dimThree}\sqrt{\tau_1 \tau_2 \tau_3} / \tol^2)$.

Ultimately, substituting $\tau_1$, $\tau_2$, and $\tau_3$ with their realizations (in Line~\ref{line:setTau3D}), considering $s_2=\dimTwo\predError$ and $s_3=\dimThree\predError$, and including each component, the sample complexity of Algorithm~\ref{alg:augIndTest3D} is:
\begin{align}
    &\OO \left( s_1 + s_2 + s_3 + \sqrt{\dimOne} + \sqrt{\dimTwo} + \sqrt{\dimThree} + \sqrt{\dimOne \dimTwo \dimThree} + \frac{\dimOne \dimTwo \dimThree}{\tol^2} \sqrt{\left(\frac{2\predError}{s_1}+\frac{4}{\dimOne}\right)\left(\frac{6}{\dimTwo} \right) \left(\frac{6}{\dimThree} \right)} \right) \notag \\
    = &\OO \left( s_1 + (\dimTwo+\dimThree) \predError +  \sqrt{\dimOne \dimTwo \dimThree} + \frac{\dimOne \sqrt{\dimTwo \dimThree}}{\tol^2} \sqrt{\left(\frac{\predError}{s_1}+\frac{1}{\dimOne}\right)} \right).
\end{align}
Recall that $s_1 = \min\left( \dimOne \predError, \frac{\dimOne^{2/3} \dimTwo^{1/3} \dimThree^{1/3} \predError^{1/3}}{\tol^{4/3}} \right)$. Substituting each of these possibilities, the overall sample complexity of Algorithm~\ref{alg:augIndTest3D} is:
\begin{equation}
    \label{eq:3Dsample_complexity}
    \OO \left( \max \left\{ \frac{\sqrt{\dimOne\dimTwo\dimThree}}{\tol^2}, \frac{\dimOne^{2/3}\dimTwo^{1/3}\dimThree^{1/3}\predError^{1/3}}{\tol^{4/3}}\right\}\right)
\end{equation}


%% file: appendix/app06-usefulFacts.tex
In the following, we list useful facts, definitions, lemmas and theorems used throughout our proofs.

\subsection{Results in Distribution Testing}

\begin{lemma}[\cite{BatuFFKRW01, LRR13}]
\label{lem:FarFromMargImpliesFarProd}
    Let $p$ and $q$ be a distribution over $[\dimPOne] \times [\dimPTwo]$ and let $\tol \in [0, 1]$. Let $p_1$ and $p_2$ denote the marginals of $p$. If $q$ is a product distribution and $\TV{p}{q} \leq \tol/3$, then $\TV{p}{p_1 \times p_2} \leq \tol$.
\end{lemma}

\begin{fact}[\cite{GR11, AIRS24}]   
\label{fact:estimateNorm}
    There exists an algorithm which, given sample access to an unknown distribution $p$ over $[\dimPOne]$, draws $\OO\left( \sqrt{\dimPOne} \log(1/\probFail) \right)$ samples from $p$
    and outputs an estimate $L$ of the $\ell_2^2$-norm of $p$ such that with probability at least $1-\probFail$,
    \begin{equation}
        \frac{\norm{p}_2^2}{2} \leq L \leq \frac{3 \norm{p}_2^2}{2}.
    \end{equation}
\end{fact}

\begin{fact}[\cite{ChanDVV14, ADKR19}]   
\label{fact:closenessTest}
    There exists an algorithm, \begin{equation}
    \label{def:closenessTest}
        \textsc{ClosenessTest}(p, q, b, \tol, \probFail),
    \end{equation}
    which is given sample access to distribution $p$ and $q$ over $[n]$, along with parameters $b$, $\tol$ and $\probFail$. If $b \geq \min(\norm{p}_2^2, \norm{q}_2^2)$, this algorithm can distinguish between $p = q$ or $\TV{p}{q} > \tol$ with probability at least $1-\probFail$ using $\OO(n \sqrt{b} / \tol^2 \log(1/\probFail))$ samples. 
\end{fact}

\begin{fact}[\cite{AIRS24}]
\label{fact:augFlatCloseness}
    Let $\predError \in [0, 1]$ and $\nu \leq 1$ be two given parameters. Let $p$ and $\pred$ be distributions over $[\dimPOne]$ such that $\TV{p}{\pred} \leq \predError$. Assume that we draw a set of $\Poi(s)$ samples from $p$ and construct a flattened distribution, $p^{(F)}$, with buckets chosen as in:
    \begin{equation}
    b_i \coloneqq \floor{\frac{\pred(i)}{\nu}} + N_i + 1.
    \label{eqn:background-augFlatBuckets-nu}
\end{equation}
where $N_i$ is the frequency of the element $i$ in the sample set. 
    In expectation, this flattening increases the domain size of the distribution by $\frac{1}{\nu} + s$.
    Further, the expected $\ell_2^2$-norm of $p^{(F)}$ is bounded as:
    \begin{equation}
        \E{\norm{p^{(F)}}^2_2} \leq \frac{2\predError}{s} + 4\nu.
    \end{equation}
\end{fact}

\subsection{Results in Information Theory}
\begin{fact}[Contrapositive of Lemma~3.2 of \cite{DiakonikolasK16}]
\label{fact:mutualInfoMustBeLarge}
    Let $X$ be a uniform random bit and let $A$ be a random variable. If $I(X; A) < 2\cdot 10^{-4}$, then for all functions $f$, $\Pr{f(A) = X} < 0.51$.
\end{fact}

\subsection{Results in \textcolor{black}{Sampling}}
\begin{fact}
\label{fact: EmpiricalDist}
Let $q$ be any distribution over $[m]$. For any $\tol, \probFail \in (0,1)$, given $t \ge \frac{m+\log(1/\probFail)}{\tol^2}$ independent samples $X_1,\ldots,X_t \sim q$, the empirical distribution $\tilde{q}_i \coloneqq \frac{1}{t}\sum_{j=1}^t \mathbb{I}\left\{X_j=i\right\}$ satisfies $\|q-\tilde{q}\|_{\mathsf{TV}} \leq \tol$ with probability at least $1{-}\probFail$.
\end{fact}
\begin{proof}
    Each element $t\cdot \tilde{q}_i$ follows a binomial distribution with parameters $t$ and $q_i$. Therefore, the vector $\left(t\tilde{q}_1, \cdots, t\tilde{q}_m\right)$ is multinomially distributed with parameters $t$ and $(q_1,\cdots, q_m)$. The Bretagnolle-Huber-Carol inequality can be applied to bound the total variation distance as follows:
    \begin{align}
        &\Pr{\sum_{i=1}^m |t\cdot\tilde{q}_i-t\cdot q_i| \geq 2t\tol} \leq 2^m e^{-2t\tol^2} \notag\\
        &\Pr{\sum_{i=1}^m |\tilde{q}_i- q_i| \geq 2\tol} \leq \probFail \qquad \text{if  } t\geq \frac{m+\log{1/\probFail}}{\tol^2}\notag\\
        \Rightarrow \quad &\Pr{\|\tilde{q}-q\|_\mathsf{TV} \leq \tol} \geq 1{-}\probFail \qquad \text{if  } t\geq \frac{m+\log{1/\probFail}}{\tol^2} \label{eq: learning_num_samples} 
    \end{align}
\end{proof}
\begin{corollary}
\label{cor:empirical_bound}
Let $\distTwoD$ be any distribution over $[m]$. For any $\probFail \in (0,1)$ and any number of samples $t$, the empirical distribution $\predEmp$ satisfies
\begin{equation}
    \Pr{\|\predEmp-\distTwoD\|_\mathsf{TV} \leq \sqrt{\frac{m+\log{1/\probFail}}{t}}} \geq 1{-}\probFail 
    \label{eq: learning_distance}
\end{equation}
\end{corollary}

%% file: appendix/app03-validProof.tex
\SValid*

\begin{proof}

    Let $\HH$ be the set of ``heavy'' rows---or the subset of $[\dimPOne]$ such that $c_i = \frac{1}{\numSamples}$.     
    To show that $\SS$ is valid with high probability, we will prove that the following events occur with high probability, and show that, conditioned on these events, $\SS$ is valid. These events are:
    \begin{enumerate}
        \item For all $i \in \HH^\mathsf{c}$, $\abs{\sum_{j =1}^\dimPTwo \frac{\tolMeas_{ij}}{\dimPTwo}} \leq \frac{\tol}{2}$.
        \label{eqn:IndTestLBNormCond}
        \item For all $j \in [\dimPTwo]$, $\abs{\sum_{i \in \HH^\mathsf{c}} \frac{\tolMeas_{ij}}{\dimPOne}} \leq \frac{\tol}{2}$.  \label{eqn:IndTestLBSumTolCond}
        \item $\abs{\HH} \leq \frac{3}{2} \predErrorMeas \numSamples$. 
        \label{eqn:IndTestLBHeavyCond}
    \end{enumerate}
    We can show using standard concentration inequalities that, with high probability, each of these events occurs.
    First, applying a Hoeffding bound and a union bound over all $i$, we find that Event~\ref{eqn:IndTestLBNormCond} occurs with probability at least 0.99:
    \begin{equation}
        \Pr{\exists i \in  \HH^\mathsf{c} \; : \abs{\sum_{j=1}^\dimPTwo \frac{\tolMeas_{ij}}{\dimPTwo}} \geq \tolMeas \; \sqrt{\frac{2}{\dimPTwo} \log \left(50 \dimPOne \right)}} \leq 0.01 \;. 
    \end{equation}
    Similarly, by a Hoeffding bound and a union bound over all $j$, we find that Event~\ref{eqn:IndTestLBSumTolCond} occurs with probability at least 0.99:
    \begin{equation}
        \Pr{\exists j \in  [\dimPTwo] \; : \abs{\sum_{i \in \HH^\mathsf{c}} \frac{\tolMeas_{ij}}{\dimPOne}} \geq \tolMeas \; \sqrt{\frac{2}{\dimPOne} \log \left( 50 \dimPTwo \right)}} \leq 0.01 \;.
    \end{equation}
    Finally, using a Chernoff bound and the fact that the expectation of $\abs{\HH}$ is $\predErrorMeas \numSamples$, we can bound the probability that $\abs{\HH}$ is large as:
    \begin{equation}
        \Pr{\abs{\HH} \geq \left(1+ \sqrt{\frac{3}{\predErrorMeas \numSamples} \log\left( 100 \right)} \right) \predErrorMeas \numSamples} \leq 0.01 \;.
        \label{eqn:ValidProof-sizeH-notConstant}
    \end{equation}
    Remark that \eqref{eqn:ValidProof-sizeH-notConstant} directly implies that Event~\ref{eqn:IndTestLBHeavyCond} occurs with high probability. In particular, since we are establishing the second term of the lower bound, we have $\predError \ge \frac{\sqrt{\dimPTwo}}{\sqrt{\dimPOne}\tol^2}$. Because the first term of the lower bound necessarily applies to the number of samples in this case, we also have $\numSamples \ge \frac{\sqrt{\dimPOne\dimPTwo}}{\tol^2}$. Combining these two inequalities yields $\predError\numSamples > \frac{\dimPTwo}{\tol^4}$. For sufficiently large $\dimPTwo$ and small $\tol$, this bound implies $\frac{\predErrorMeas \numSamples}{2} > 3\log (100)$. Therefore, \eqref{eqn:ValidProof-sizeH-notConstant} implies that $\abs{\HH} \geq \frac{3}{2} \predErrorMeas \numSamples$ with probability at least 0.01.
    
    
    \paragraph{$\mathbf{X=0}$ implies $\mathbf{\distLB}$ product and $\boldsymbol{\predError}$-close to $\mathbf{\predLB}$.}
    First, note that if $X=0$, then $\distLB$ is a product distribution by construction. Further, conditioned on Event \ref{eqn:IndTestLBHeavyCond}, when $X = 0$ and $\predErrorMeas \numSamples \ll 1$, $\distLB = \predLB$. In the case that $\predErrorMeas \numSamples > 1$, we will  show that, $\TV{\distLB}{\predLB} \leq \predError$ with high probability. Consider $\TV{\distLB}{\predLB}$:
    \begin{align}
        \TV{\distLB}{\predLB} &= \frac{1}{2} \sum_{i=1}^\dimPOne \sum_{j=1}^\dimPTwo \abs{\frac{ \prodLB(i, j) c_i}{s_i \cdot C} - \frac{1}{\dimPOne \dimPTwo}} \\ 
        &= \frac{1}{2} \left( \sum_{i \in \HH} \sum_{j=1}^\dimPTwo \abs{\frac{1}{\numSamples \dimPTwo C} - \frac{1}{\dimPOne \dimPTwo}} + \sum_{i \in \HH^\mathsf{c}} \sum_{j=1}^\dimPTwo \abs{\frac{1}{\dimPOne \dimPTwo C} - \frac{1}{\dimPOne \dimPTwo}} \right). \\
        &= \frac{1}{2} \left( \abs{\HH} \abs{ \frac{1}{\numSamples C} - \frac{1}{\dimPOne} } + \abs{\HH^\mathsf{c}} \abs{ \frac{1}{\dimPOne C} - \frac{1}{\dimPOne} } \right) \\
        &= \frac{1}{2} \left( \abs{\HH} \abs{ \frac{1}{\numSamples C} - \frac{1}{\dimPOne C} + \frac{1}{\dimPOne C} - \frac{1}{\dimPOne} } + \left(\dimPOne - \abs{\HH}\right) \abs{ \frac{1}{\dimPOne C} - \frac{1}{\dimPOne} } \right) \label{eqn:IndTestLBPreTriangle}.
    \end{align}
    Applying the triangle inequality, we can upper bound \eqref{eqn:IndTestLBPreTriangle} by
    \begin{align}
        \TV{\distLB}{\predLB} 
        &\leq \frac{1}{2} \left( \abs{\HH} \abs{ \frac{1}{\numSamples C} - \frac{1}{\dimPOne C}} + \abs{\HH} \abs{\frac{1}{\dimPOne C} - \frac{1}{\dimPOne} } + \left(\dimPOne - \abs{\HH}\right) \abs{ \frac{1}{\dimPOne C} - \frac{1}{\dimPOne} } \right) \\
        &= \frac{1}{2} \left( \abs{\HH} \abs{ \frac{1}{\numSamples C} - \frac{1}{\dimPOne C}} + \dimPOne \cdot \abs{ \frac{1}{\dimPOne C} - \frac{1}{\dimPOne} } \right) \\
        &= \frac{1}{2} \left( \frac{\abs{\HH}}{C} \abs{\frac{1}{\numSamples} - \frac{1}{\dimPOne}} + \abs{\frac{1}{C} - 1} \right).
    \end{align}
    Because $\numSamples \leq \dimPOne / 2$, we know that $\abs{1/\numSamples - 1/\dimPOne} = 1/\numSamples - 1/\dimPOne \leq 1/
    \numSamples$. Therefore, we have
    \begin{align}
        \TV{\distLB}{\predLB} \leq \frac{1}{2} \left( \frac{\abs{\HH}}{C \numSamples}  + \abs{\frac{1}{C} - 1} \right) \label{eqn:IndTestLBAlphaTerms}.
    \end{align}
    To prove that \eqref{eqn:IndTestLBAlphaTerms} is upper bounded by $\predError$, we now show that $\frac{\abs{\HH}}{C \numSamples} \leq \OO(\predErrorMeas)$ and $\abs{\frac{1}{C} - 1} \leq \OO(\predErrorMeas)$, then choose $\predErrorMeas$ correspondingly. First, recall that, by Event \ref{eqn:IndTestLBHeavyCond}, we can upper bound $\abs{\HH}$ by $3 \predErrorMeas \numSamples / 2$. 
    A direct calculation shows that $C \geq 1$:
    \begin{equation}
        C = \frac{\abs{\HH}}{\numSamples} + \frac{\dimPOne - \abs{\HH}}{\dimPOne} = 1 + \frac{\abs{\HH}}{\numSamples} \left(1-\frac{\numSamples}{\dimPOne}\right) \geq  1 + \frac{\abs{\HH}}{2\numSamples} \geq 1 \label{eqn:IndTestLB_Clowerbound},
    \end{equation}
    Applying \eqref{eqn:IndTestLB_Clowerbound} and our upper bound on $\abs{\HH}$, we have
    \begin{equation}
        \frac{\abs{\HH}}{C \numSamples} \leq \frac{3 \predErrorMeas \numSamples}{2\numSamples}  = \frac{3}{2}\predErrorMeas. \label{eqn:IndTestLBAlphaUBTerm1}
    \end{equation}
    Further, we can apply Event \ref{eqn:IndTestLBHeavyCond} to upper bound $C$ as 
    \begin{equation}
        C = \frac{\abs{\HH}}{\numSamples} + \frac{\dimPOne - \abs{\HH}}{\dimPOne} = 1 + \frac{\abs{\HH}}{\numSamples}\left(1-\frac{\numSamples}{\dimPOne}\right) \leq  1 + \frac{\abs{\HH}}{\numSamples} \leq 1+\frac{3}{2} \predErrorMeas .
        \label{eqn:IndTestLB_Cupperbound}
    \end{equation}
    We can upper bound $\abs{1-1/C}$ using the bounds $1 \leq C \leq 1+\frac{3}{2}\predErrorMeas$:
    \begin{align}
        \label{eqn:IndTestLBAlphaUBTerm2}
        0 \leq 1-\frac{1}{C} \leq 1-\frac{1}{1+\frac{3}{2}\predErrorMeas} = \frac{\frac{3}{2}\predErrorMeas}{1+\frac{3}{2}\predErrorMeas} \leq \frac{3}{2}\predErrorMeas
    \end{align}
    By \eqref{eqn:IndTestLBAlphaUBTerm1} and \eqref{eqn:IndTestLBAlphaUBTerm2}, we can upper bound \eqref{eqn:IndTestLBAlphaTerms} as
    \begin{equation}
        \TV{\distLB}{\predLB} \leq \frac12 \left( \frac{3}{2}\predErrorMeas + \frac{3}{2}\predErrorMeas \right) = \frac32\predErrorMeas.
    \end{equation}
    Choosing $\predErrorMeas = \frac{2}{3} \predError$ ensures that $\TV{\distLB}{\predLB} \leq \predError$.
    
    \paragraph{$\mathbf{X=1}$ implies $\mathbf{\distLB}$ far from product.}
    The marginals of $\fardistLB(i, j) = \frac{\farLB(i, j)\cdot c_i}{\measSum_i \cdot C}$ are:
    \begin{equation}
        \fardistLB_1(i) = \frac{c_i}{C} \quaaaad \fardistLB_2(j) = \sum_{i=1}^n \frac{\farLB(i, j) \cdot c_i}{\measSum_i C}
    \end{equation}
    We want to show that $\TV{\fardistLB}{\fardistLB_1 \times \fardistLB_2} > \tol$. Again, let $\HH$ be the set of ``heavy'' rows and condition on Events \ref{eqn:IndTestLBNormCond},  \ref{eqn:IndTestLBSumTolCond}, and \ref{eqn:IndTestLBHeavyCond}.
    Consider the distance between $\fardistLB$ and $\fardistLB_1 \times \fardistLB_2$. We have:
    \begin{align}
        \TV{\fardistLB}{\fardistLB_1 \times \fardistLB_2} 
        &= \frac{1}{2} \sum_{i=1}^\dimPOne \sum_{j=1}^\dimPTwo \abs{\frac{\farLB(i, j)}{\measSum_i} \cdot \frac{c_i}{C} - \frac{c_i}{C} \cdot \sum_{i'=1}^\dimPOne \frac{\farLB(i', j)}{\measSum_{i'}} \cdot \frac{c_{i'}}{C}} \\
        &= \frac{1}{2} \sum_{i=1}^\dimPOne \frac{c_i}{C}  \sum_{j=1}^\dimPTwo \abs{\frac{\farLB(i, j)}{\measSum_i} - \sum_{i'=1}^\dimPOne \frac{\farLB(i', j)}{\measSum_{i'}} \cdot \frac{c_{i'}}{C}} \\        
        &= \frac{1}{2} \sum_{i=1}^\dimPOne \frac{c_i}{C} \sum_{j=1}^\dimPTwo \abs{\frac{\farLB(i, j)}{\measSum_i} - \farLB(i, j) + \farLB(i, j) - \frac{1}{\dimPTwo} + \frac{1}{\dimPTwo} - \sum_{i'=1}^\dimPOne \frac{\farLB(i', j)}{\measSum_{i'}} \cdot \frac{c_{i'}}{C}}.
    \end{align}
    Applying the reverse triangle inequality, we can lower bound this quantity as:
    \begin{align}
        \TV{\fardistLB}{\fardistLB_1 \times \fardistLB_2} \geq \frac{1}{2} \sum_{i=1}^\dimPOne \frac{c_i}{C} \sum_{j=1}^\dimPTwo  \abs{ \farLB(i, j) - \frac{1}{\dimPTwo}} - \abs{\frac{\farLB(i, j)}{\measSum_i} - \farLB(i, j)} - \abs{\frac{1}{\dimPTwo} - \sum_{i'=1}^\dimPOne \frac{\farLB(i', j)}{\measSum_{i'}} \cdot \frac{c_{i'}}{C}}.
        \label{eqn:IndTestLBThreeTerms}
    \end{align}
    To establish that $\TV{\fardistLB}{\fardistLB_1 \times \fardistLB_2} > \tol$, we bound each term in the inner summation of \eqref{eqn:IndTestLBThreeTerms} separately for fixed $i$ and $j$, applying the conditioning on Events \ref{eqn:IndTestLBNormCond},  \ref{eqn:IndTestLBSumTolCond} and \ref{eqn:IndTestLBHeavyCond}. Note that applying upper bounds to the second and third terms with negative coefficients yields a lower bound on the TV distance. \\
    The first term is
    \begin{equation}
        \abs{ \farLB(i, j) - \frac{1}{\dimPTwo}} 
        \quad = \begin{cases}
            \abs{ \frac{1}{\dimPTwo} - \frac{1}{\dimPTwo}} & i \in \HH \\
            \abs{ \frac{1 + \tolMeas_{ij}}{\dimPTwo} - \frac{1}{\dimPTwo}} & i \in \HH^\mathsf{c} \\
        \end{cases}
        \quad 
        = \begin{cases}
            0 & i \in \HH \\
            \frac{\tolMeas}{\dimPTwo} & i \in \HH^\mathsf{c} \\
        \end{cases}.
        \label{eqn:IndTestLBTerm1}
    \end{equation}
    The second term is
    \begin{equation}
        \abs{\frac{\farLB(i, j)}{\measSum_i} - \farLB(i, j)} = \farLB(i, j) \abs{\frac{1}{\measSum_i} - 1}.
    \end{equation}
    When $i \in \HH$, $\measSum_i = \sum_{j=1}^\dimPTwo \frac{1}{\dimPTwo} = 1$. We can therefore separate this term based on the cases that $i \in \HH$ and $i \notin \HH$ as
    \begin{equation}
        \abs{\frac{\farLB(i, j)}{\measSum_i} - \farLB(i, j)}
        \quad = \begin{cases}
            0 & i \in \HH \\
            \frac{1 + \tolMeas_{ij}}{\dimPTwo} \abs{\frac{1}{\measSum_i} - 1} & i \in \HH^\mathsf{c}
        \end{cases}.
        \label{eqn:IndTestLBTerm2Unsimplified}
    \end{equation}
    Using Event \ref{eqn:IndTestLBNormCond} that for all $i \in \HH^\mathsf{c}$, $\abs{\sum_{j =1}^\dimPTwo \frac{\tolMeas_{ij}}{\dimPTwo}} \leq \tol / 2$ and the definition $\measSum_i = \sum_{j=1}^\dimPTwo \farLB(i, j) = 1 + \sum_{j =1}^\dimPTwo \frac{\tolMeas_{ij}}{\dimPTwo}$, we have $\measSum_i \in (1 - \tol / 2, 1 + \tol / 2)$ implying that $\abs{1/\measSum_i - 1} \leq \frac{\tol/2}{1- \tol/2} \leq \tol$ and resulting in the simplified bound:
    \begin{equation}
        \abs{\frac{\farLB(i, j)}{\measSum_i} - \farLB(i, j)}
        \quad \leq \begin{cases}
            0 & i \in \HH \\
            \frac{1 + \tolMeas_{ij}}{\dimPTwo} \cdot \tol
            & i \in \HH^\mathsf{c}
        \end{cases}.
        \label{eqn:IndTestLBTerm2}
    \end{equation}
    The third term is
    \begin{align}
        \abs{\frac{1}{\dimPTwo} - \sum_{i'=1}^\dimPOne \frac{\farLB(i', j)}{\measSum_{i'}} \cdot \frac{c_{i'}}{C}} 
        &= \frac{1}{\dimPTwo} \abs{1 - \frac{1}{C} \left( \sum_{i' \in \HH} \frac{1}{\numSamples} + \sum_{i' \in \HH^\mathsf{c}} \frac{1}{\measSum_{i'}} \cdot \frac{1 + \tolMeas_{i'j}}{\dimPOne} \right)} \\
        &= \frac{1}{\dimPTwo}\abs{1 - \frac{1}{C} \left( \sum_{i' \in \HH} \frac{1}{\numSamples} + \sum_{i' \in \HH^\mathsf{c}} \frac{1}{\dimPOne} - \sum_{i' \in \HH^\mathsf{c}} \frac{1}{\dimPOne} + \sum_{i' \in \HH^\mathsf{c}} \frac{1}{\measSum_{i'}} \cdot \frac{1 + \tolMeas_{i'j}}{\dimPOne} \right)} \\
        &= \frac{1}{\dimPTwo}\abs{1 - \frac{1}{C} \left( C - \sum_{i' \in \HH^\mathsf{c}} \frac{1}{\dimPOne} + \sum_{i' \in \HH^\mathsf{c}} \frac{1}{\measSum_{i'}} \cdot \frac{1 + \tolMeas_{i'j}}{\dimPOne} \right)} \\
        &= \frac{1}{\dimPTwo C}\abs{ \sum_{i' \in \HH^\mathsf{c}} \frac{1}{\dimPOne} - \sum_{i' \in \HH^\mathsf{c}} \frac{1}{\measSum_{i'}} \cdot \frac{1 + \tolMeas_{i'j}}{\dimPOne} } \\
        &= \frac{1}{\dimPTwo C}\abs{ \sum_{i' \in \HH^\mathsf{c}} \frac{1}{\dimPOne}\left(1-\frac{1}{\measSum_{i'}}\right)- \sum_{i' \in \HH^\mathsf{c}} \frac{\tolMeas_{i'j}}{\dimPOne \measSum_{i'}}} \\
        &= \frac{1}{\dimPTwo C}\abs{ \sum_{i' \in \HH^\mathsf{c}} \frac{1}{\dimPOne}\left(1-\frac{1}{\measSum_{i'}}\right)+ \sum_{i' \in \HH^\mathsf{c}} \left(1-\frac{1}{\measSum_{i'}}\right)\frac{\tolMeas_{i'j}}{\dimPOne}-\sum_{i' \in \HH^\mathsf{c}} \frac{\tolMeas_{i'j}}{\dimPOne}} \\
        &\leq \frac{1}{\dimPTwo C} \left( \sum_{i' \in \HH^\mathsf{c}} \frac{1}{\dimPOne}\abs{1-\frac{1}{\measSum_{i'}}} + \abs{\sum_{i' \in \HH^\mathsf{c}} \left(1-\frac{1}{\measSum_{i'}}\right)\frac{\tolMeas_{i'j}}{\dimPOne}} + \abs{\sum_{i' \in \HH^\mathsf{c}} \frac{\tolMeas_{i'j}}{\dimPOne}}\right), \label{eqn: IndTestLBTerm3}  
    \end{align}
    where the triangle inequality is applied at the last step. Again, we have for all $i' \in \HH^\mathsf{c}$, $\abs{1-1/\measSum_{i'}} \leq \frac{\tol/2}{1- \tol/2} \leq \tol$, and we can apply Cauchy-Schwarz to bound the middle term in \eqref{eqn: IndTestLBTerm3} as:
    \begin{align}
        \abs{\sum_{i' \in \HH^\mathsf{c}} \left(1-\frac{1}{\measSum_{i'}}\right)\frac{\tolMeas_{i'j}}{\dimPOne}} 
        &\leq \sqrt{\sum_{i' \in \HH^\mathsf{c}} \left(1-\frac{1}{\measSum_{i'}}\right)^2} \cdot \sqrt{\sum_{i' \in \HH^\mathsf{c}}\left(\frac{\tolMeas_{i'j}}{\dimPOne}\right)^2}\\
        &\leq \sqrt{ \abs{\HH^\mathsf{c}}\tol^2} \cdot \sqrt{\abs{\HH^\mathsf{c}}\left(\frac{\tolMeas}{\dimPOne}\right)^2}\\
        &= \frac{\abs{\HH^\mathsf{c}}}{\dimPOne}\tol\tolMeas. \label{eqn: Term3SecondPart}
    \end{align}
   Applying the condition that for all $j \in [\dimPTwo], \abs{\sum_{{i'} \in \HH^\mathsf{c}} \frac{\tolMeas_{i'j}}{\dimPOne}} \leq \tol / 2$ (i.e. Event \ref{eqn:IndTestLBSumTolCond})  together with \eqref{eqn: Term3SecondPart} and $\abs{1-1/\measSum_{i'}} \leq \tol$, we can bound the third term as
    \begin{align}
        \abs{\frac{1}{\dimPTwo} - \sum_{i'=1}^\dimPOne \frac{\farLB(i', j)}{\measSum_{i'}} \cdot \frac{c_{i'}}{C}} 
        &\leq \frac{1}{\dimPTwo C} \left( \sum_{i' \in \HH^\mathsf{c}} \frac{1}{\dimPOne}\abs{1-\frac{1}{\measSum_{i'}}} + \abs{\sum_{i' \in \HH^\mathsf{c}} \left(1-\frac{1}{\measSum_{i'}}\right)\frac{\tolMeas_{i'j}}{\dimPOne}} + \abs{\sum_{i' \in \HH^\mathsf{c}} \frac{\tolMeas_{i'j}}{\dimPOne}}\right) \\
        &\leq \frac{1}{\dimPTwo C} \cdot \left( \frac{\abs{\HH^\mathsf{c}}}{\dimPOne}\tol+\frac{\abs{\HH^\mathsf{c}}}{\dimPOne}\tol\tolMeas + \frac{\tol}{2} \right)\\
        &= \frac{\abs{\HH^\mathsf{c}}}{\dimPOne\dimPTwo C} \left(\tol + \tol\tolMeas\right) + \frac{\tol}{2\dimPTwo C}. \label{eqn:IndTestLBTerm3}
    \end{align}
    Substituting \eqref{eqn:IndTestLBTerm1}, \eqref{eqn:IndTestLBTerm2} and \eqref{eqn:IndTestLBTerm3} back into \eqref{eqn:IndTestLBThreeTerms}, we get the following bound:
    \begin{align}
        \TV{\fardistLB}{\fardistLB_1 \times \fardistLB_2} &\geq \frac{1}{2} \sum_{i=1}^\dimPOne \frac{c_i}{C} \sum_{j=1}^\dimPTwo  \abs{ \farLB(i, j) - \frac{1}{\dimPTwo}} - \abs{\frac{\farLB(i, j)}{\measSum_i} - \farLB(i, j)} - \abs{\frac{1}{\dimPTwo} - \sum_{i'=1}^\dimPOne \frac{\farLB(i', j)}{\measSum_{i'}} \cdot \frac{c_{i'}}{C}}\\
        &\geq \frac{1}{2} \left( \sum_{i \in \HH^\mathsf{c}} \frac{c_i}{C} \sum_{j=1}^\dimPTwo \frac{\tolMeas}{\dimPTwo} 
        - \sum_{i \in \HH^\mathsf{c}} \frac{c_i}{C} \sum_{j=1}^\dimPTwo \frac{1 + \tolMeas_{ij}}{\dimPTwo} \cdot \tol - \sum_{i=1}^\dimPOne \frac{c_i}{C} \sum_{j=1}^\dimPTwo \frac{\abs{\HH^\mathsf{c}}}{\dimPOne\dimPTwo C} \left(\tol + \tol\tolMeas\right)+\frac{\tol}{2\dimPTwo C}\right) \\
        &= \frac{1}{2} \left( \frac{\abs{\HH^\mathsf{c}} \tolMeas}{nC}  - \frac{\abs{\HH^\mathsf{c}}}{nC} \tol - \left( \sum_{i \in \HH^\mathsf{c}} \frac{1}{\dimPOne C}\sum_{j=1}^\dimPTwo \frac{\tolMeas_{ij}}{\dimPTwo} \right)\tol -  \frac{\abs{\HH^\mathsf{c}}}{\dimPOne C}\left(\tol + \tol\tolMeas\right)-\frac{\tol}{2C}\right), 
    \end{align}
    Again, by the conditioning on Event \ref{eqn:IndTestLBNormCond}, we have $\abs{\sum_{j =1}^\dimPTwo \frac{\tolMeas_{ij}}{\dimPTwo}} \leq \tol / 2$ for all $i \in \HH^\mathsf{c}$, implying
    \begin{equation}
        \TV{\fardistLB}{\fardistLB_1 \times \fardistLB_2} \geq \frac{1}{2C} \left( \tolMeas \left( \frac{\abs{\HH^\mathsf{c}}}{n} \left(1-\tol\right)\right) - \frac{\abs{\HH^\mathsf{c}}}{n} \left( \tol + \frac{\tol}{2} + \tol \right) -  \frac{\tol}{2} \right),
        \label{eqn:IndTestLBSubbed}
    \end{equation}
    We can also upper and lower bound $\frac{\abs{\HH^\mathsf{c}}}{\dimPOne}$ as
    \begin{equation}
        \frac{1}{4} \leq 1- \frac{3 \predErrorMeas}{4} \leq 
        \frac{\dimPOne - 3 \predErrorMeas \numSamples/2}{ \dimPOne} \leq \frac{\abs{\HH^\mathsf{c}}}{\dimPOne} \leq \frac{\dimPOne}{\dimPOne} = 1,
        \label{eqn:ValidLB-boundsonHc}
    \end{equation}
    implying \eqref{eqn:IndTestLBSubbed} is bounded as
    \begin{equation}
        \TV{\fardistLB}{\fardistLB_1 \times \fardistLB_2}
        \geq \frac{1}{2C} \left( \frac{\tolMeas}{4}(1-\tol) - 3\tol \right). 
    \end{equation}
    Given the upper bound $C \leq 1+ 3\predErrorMeas / 2  \leq 5/2$ partially established in \eqref{eqn:IndTestLB_Cupperbound}, we ultimately have
    \begin{equation}
        \TV{\fardistLB}{\fardistLB_1 \times \fardistLB_2} \geq \frac{1}{5} \left( \frac{\tolMeas}{4}(1-\tol) - 3\tol \right).
    \end{equation}
    Any $\tolMeas$ satisfying $\tolMeas \geq \frac{72\tol}{1-\tol}$  guarantees that $ \TV{\fardistLB}{\fardistLB_1 \times \fardistLB_2} \geq 3\tol$. For example, by setting $\tolMeas = 192 \tol$, we ensure that $ \TV{\fardistLB}{\fardistLB_1 \times \fardistLB_2} \geq 3\tol$ for all $\tol \leq 1/6$. Applying Lemma~\ref{lem:FarFromMargImpliesFarProd}, this implies that for all product distributions $q$, $\TV{\fardistLB}{q} \geq \tol$.

    \paragraph{$\boldsymbol{\abs{\SS}}$ bounded}
    We now prove that the first event---the size of $\SS$ is lower bounded by $\numSamples / 100$---holds with high probability.
    Let $a_{i, j}$ be the number of copies of $(i, j)$ in $\SS$. By construction, $a_{i, j}$ is $\Poi(\numSamples \cdot \prodMeasLB(i, j))$ and the size of $\SS$ is sum of all $a_{i, j}$. Because the sum of Poisson random variables is itself a Poisson random variable, 
    \begin{equation}
        \abs{\SS} 
        = \sum_{(i, j)} a_{i, j} 
        \sim \Poi\left( \sum_{(i, j)} \numSamples \cdot \prodMeasLB(i, j) \right) 
        = \Poi\left(  \numSamples \cdot \norm{\prodMeasLB}_1 \right)
    \end{equation}
    To prove that $\abs{\SS}$ is at least $\numSamples / 100$ with high probability, we prove via a Hoeffding bound that $\norm{\prodMeasLB}_1$ is $\Theta(1)$ and apply a {Chernoff} bound for Poisson random variables. Consider the norm of $\prodMeasLB$:
    \begin{equation}
        \norm{\prodMeasLB}_1 
        = \sum_{i \in [\dimPOne]} \sum_{j \in \dimPTwo} \abs{c_i \cdot \prodLB(i, j) }
        = \sum_{i \in \HH} c_i \sum_{j \in \dimPTwo} \abs{\prodLB(i, j)} + \sum_{i \in \HH^\mathsf{c}} c_i \sum_{j \in \dimPTwo} \abs{\prodLB(i, j)}.
        \label{eqn:valid-normQ}
    \end{equation}
    When $X=0$, $\prodLB(i, j)$ is always $1/\dimPTwo$. Therefore, \eqref{eqn:valid-normQ} is:
    \begin{equation}
        \norm{\prodMeasLB}_1  
        = \sum_{i \in \HH} c_i \sum_{j \in [\dimPTwo]} \frac{1}{\dimPTwo} + \sum_{i \in \HH^\mathsf{c}} c_i \sum_{j \in [\dimPTwo]} \frac{1}{\dimPTwo}
        = \sum_{i \in \HH} c_i  + \sum_{i \in \HH^\mathsf{c}} c_i
        = C.
        \label{eqn:valid-normQGivenX=0}
    \end{equation}
    On the other hand, when $X=1$, $\prodLB(i, j)$ depends on $i$. Though it remains true that $\prodLB(i, j) = 1/\dimPTwo$ when $i \in \HH$, when $i \in \HH^\mathsf{c}$, $\prodLB(i, j)$ is $(1+\tolMeas_{ij}) / \dimPTwo$. In this case, \eqref{eqn:valid-normQ} is:
    \begin{align}
        \norm{\prodMeasLB}_1  
        &= \sum_{i \in \HH} c_i \sum_{j \in [\dimPTwo]} \frac{1}{\dimPTwo} + \sum_{i \in \HH^\mathsf{c}} c_i \sum_{j \in [\dimPTwo]} \abs{\frac{1+ \tolMeas_{ij}}{\dimPTwo}} \\
        &= \sum_{i \in \HH} c_i \sum_{j \in [\dimPTwo]} \frac{1}{\dimPTwo} + \sum_{i \in \HH^\mathsf{c}} c_i \sum_{j \in [\dimPTwo]} {\frac{1+ \tolMeas_{ij}}{\dimPTwo}} \\
        &= \sum_{i \in \HH} c_i  + \sum_{i \in \HH^\mathsf{c}} c_i + \sum_{i \in \HH^\mathsf{c}} c_i \sum_{j \in [\dimPTwo]} \frac{\tolMeas_{ij}}{\dimPTwo} \\
        &= C + \sum_{i \in \HH^\mathsf{c}} \sum_{j \in [\dimPTwo]} \frac{\tolMeas_{ij}}{\dimPOne\dimPTwo}.
        \label{eqn:valid-normQGivenX=1}
    \end{align}
    As each $i, j$ term in \eqref{eqn:valid-normQGivenX=1} is $\pm \frac{\tol'}{\dimPOne \dimPTwo}$, we may apply a Hoeffding bound to their sum:
    \begin{equation}
        \Pr{\abs{\sum_{i \in \HH^\mathsf{c}} \sum_{j \in [\dimPTwo]} \frac{\tolMeas_{ij}}{\dimPOne\dimPTwo}} \geq 0.1} \leq 2\exp\left( -\frac{0.005 \cdot \dimPOne^2 \dimPTwo}{\abs{\HH^\mathsf{c}} \tol^{'2}} \right).
    \end{equation}
    From \eqref{eqn:ValidLB-boundsonHc}, we know $\abs{\HH^\mathsf{c}} \leq \dimPOne$, or, equivalently, $ -\frac{1}{\abs{\HH^\mathsf{c}}} \leq -\frac{1}{\dimPOne}$. Substituting this inequality simplifies our Hoeffding bound to:
    \begin{equation}
        \Pr{\abs{\sum_{i \in \HH^\mathsf{c}} \sum_{j \in [\dimPTwo]} \frac{\tolMeas_{ij}}{\dimPOne\dimPTwo}} \geq 0.1} 
        \leq 2\exp\left( -\frac{0.005 \cdot \dimPOne \dimPTwo}{ \tol^{'2}} \right) \leq 0.01,
        \label{eqn:ValidLB-hoeffdingNumerical}
    \end{equation}
    where the final inequality holds provided $\dimPOne \dimPTwo / \tol^{'2}$ is sufficiently large. 
    Combining this bound with the bounds on $C$ established in \eqref{eqn:IndTestLB_Clowerbound} and \eqref{eqn:IndTestLB_Cupperbound}, we have, when $X=0$,
    \begin{equation}
        1 \leq \norm{\prodMeasLB}_1 = C \leq 2.5,
    \end{equation}
    and, when $X=1$, with probability at least 0.99,
    \begin{equation}
        0.9 \leq C - 0.1 \leq \norm{\prodMeasLB}_1 \leq C+0.1 \leq 2.6 \; .
    \end{equation}
    Because $\norm{\prodMeasLB} = \Theta(1)$ with high probability, regardless of $X$, $\abs{\SS} \sim \Poi\left( \numSamples \norm{\prodMeasLB}_1 \right)$ should be $\Theta(\numSamples)$ with high probability. We will now prove a high probability lower bound of $\numSamples / 100$ on $\abs{\SS}$ (the upper bound is not necessary for this lemma).
    First, we note that conditioned on $0.9 \leq \norm{\prodMeasLB}_1$, $\Poi(\numSamples \cdot 0.9)$ is stochastically less than  $\Poi( \numSamples \norm{\prodMeasLB}_1)$, implying 
    \begin{equation}
        \Pr{\abs{\SS} \leq \frac{\numSamples}{100} \mid 0.9 \leq \norm{\prodMeasLB}_1} 
        = \Pr[S \sim \Poi\left( \numSamples \norm{\prodMeasLB}_1 \right)]{S \leq \frac{\numSamples}{100} 
        \mid 0.9 \leq \norm{\prodMeasLB}_1}
        \leq \Pr[S \sim \Poi\left( 0.9 \numSamples \right)]{S \leq \frac{\numSamples}{100}}.
        \label{eqn:ValidLB-poissonDominance}
    \end{equation}
    Applying a Chernoff bound specific to Poisson distributions (see, for example, Theorem 5.4 of \cite{Mitzenmacher_book}), we can bound \eqref{eqn:ValidLB-poissonDominance} as
    \begin{equation}
        \Pr{\abs{\SS} \leq \frac{\numSamples}{100} \mid 0.9 \leq \norm{\prodMeasLB}_1} 
        \leq \frac{e^{-0.9\numSamples} \cdot \left(0.9 \numSamples \cdot e \right)^{\numSamples / 100}}{\left( \numSamples/100\right)^{\numSamples / 100}} \leq 0.01 \;, 
        \label{eqn:ValidLB-boundSConditional}
    \end{equation}
    where the final inequality holds for any $\numSamples \geq 6$. Combining \eqref{eqn:ValidLB-hoeffdingNumerical} and \eqref{eqn:ValidLB-boundSConditional} with the law of total probability, we ultimately have, conditioned on Events \ref{eqn:IndTestLBNormCond}, \ref{eqn:IndTestLBSumTolCond} and \ref{eqn:IndTestLBHeavyCond}, \begin{equation}
        \Pr{\abs{\SS} \geq \frac{\numSamples}{100}} \leq 0.02 \; .
    \end{equation}

    \paragraph{Bringing it all together.}
    We have thus far shown that Events~\ref{eqn:IndTestLBNormCond},  \ref{eqn:IndTestLBSumTolCond} and \ref{eqn:IndTestLBHeavyCond} each occur with probability at least $0.99$. By a union bound, all three occur with probability at least $0.97$. Conditioned on these events, we have shown that: (i) if $X=0$, then $\distLB$ is a product distribution and $\TV{\distLB}{\predLB} \leq \predError$; (ii) if $X=1$, then $\distLB$ is $\tol$-far from all product distributions; and (iii) $\abs{\SS} \leq \numSamples / 100$ with probability at least $ 0.98$. Ultimately, by a chain rule, the probability that all three desiderata hold is at least $0.97 \cdot 0.98 = 0.9506 > 0.95$.
\end{proof}

%% file: appendix/app04-mutualInfoProof.tex
\MI*

\begin{proof}
    
    For all $v \in \mathbb{N}^\dimPTwo$, define $x_{\min}(v)$ and $x_{\max}(v)$ as:
    \begin{gather}
        x_{\min}(v) = \argmin_{x \in \{0, 1\}} \left( \pVx \right) \\
        x_{\max}(v) = \argmax_{x \in \{0, 1\}} \left( \pVx \right)
        \label{eqn:xMin_xMax}
    \end{gather} 
    Note that throughout the remainder of the proof, we will omit $v$, writing $x_{\min}$ and $x_{\max}$. $v$ will be clear from context.
    Lemma~\ref{lem: MIFormula} allows us to re-express the formula for mutual information between $X$ and $A$ in a form amenable to analysis.
    \begin{restatable}{lem}{MIFormula}
        \label{lem: MIFormula}
        Let $X$ be either $0$ or $1$ with equal probability. Let $A$ be a random vector taking values in $\mathbb{N}^\dimPTwo$. For all $v \in \mathbb{N}^\dimPTwo$, let $x_{\min}(v)$ and $x_{\max}(v)$ be as defined in \eqref{eqn:xMin_xMax}. 
        Then, the mutual information between $X$ and $A$ satisfies:
        \begin{equation}
            I(X;A) = \sum_{v \in \mathbb{N}^\dimPTwo} \OO\left( \pXmin \left( 1- \frac{\pXmax}{\pXmin} \right)^2 \right).
        \end{equation}
    \end{restatable}
    \noindent See Appendix~\ref{Apdx: MI_Formula} for proof of Lemma~\ref{lem: MIFormula}.
    
    To continue to prove our bound, we will split the sum given in Lemma~\ref{lem: MIFormula} into three separate sums: the sums over all $v$ such that $\norm{v}_1 = 0$, $\norm{v}_1 = 1$, and $\norm{v}_1 \geq 2$. Intuitively, the dominant contribution to the mutual information arises from the third case where at least two samples fall into the corresponding row. In the light rows, the algorithm can obtain some information about whether it is uniform or $\tol$-far from uniformity by comparing the number of hits across different bins.

    Recall that $A$ corresponds to some row, $i$, of the domain, which is characterized by a random variable $c_i$. This row is either ``light'' when $c_i = 1/\dimPOne$, or ``heavy'', when $c_i = 1 / \numSamples$.
    By construction, when $c_i = 1 / \numSamples$, each $a_{ij} \sim \Poi \left( 1/\dimPTwo \right)$, regardless of $X$. Consequently, for all $v$,
    \begin{equation}
    \label{eqn:XdoesntAffectK}
        \pXminAndK = \pXmaxAndK.
    \end{equation}
    As a consequence of Equation~\eqref{eqn:XdoesntAffectK}, the law of total probability and our definitions of $x_{\min}$ and $x_{\max}$, we also have that, for all $v$,
    \begin{equation}
    \label{eqn:XaffectsN}
        \pXminAndN \leq \pXmaxAndN.
    \end{equation}
    Using monotonicity of ratios under additive shifts, we can show that Equations~\eqref{eqn:XdoesntAffectK}~and~\eqref{eqn:XaffectsN} together with the law of total probability imply for all $v$,
    \begin{equation}
    \label{eqn:MIv0}
        \OO\left( \!\left( 1{-} \frac{\pXmax}{\pXmin} \right)^2 \right) \leq \OO\left(\!\left( 1{-} \frac{\pXmaxGivenN}{\pXminGivenN} \right)^2 \right).
    \end{equation}
    We apply this inequality for the cases $\norm{v}_1 = 0$, and $\norm{v}_1 = 1$, yielding the following bound on the mutual information:
    \begin{align}
        I(X;A) 
        = &\sum_{v: \norm{v}_1 = 0} \OO\left( \pXmin \left( 1{-} \frac{\pXmaxGivenN}{\pXminGivenN} \right)^2 \right) \nonumber \\
        &+ \sum_{v: \norm{v}_1 = 1} \!\OO\!\left( \! \pXmin \!\left( \!1{-} \frac{\pXmaxGivenN}{\pXminGivenN} \!\right)^2 \right) \nonumber \\
        &+ \sum_{v: \norm{v}_1 \geq 2} \OO\left( \pXmin \left( 1{-} \frac{\pXmax}{\pXmin} \right)^2 \right). \label{eqn:MISplitUpConditional}
    \end{align}
    In Lemmas~\ref{lem:MIboundnorm0}, \ref{lem:MIboundnorm1}, and \ref{lem:MIboundnorm2}, we bound each term separately. Combining these bounds, we have
    \begin{align}
        I(X;A) 
        &\leq \OO \left( \frac{\tol^4 \numSamples^4}{\dimPOne^4 \dimPTwo^2} \right) + \OO \left( \frac{\tol^4 \numSamples^3}{\dimPOne^3 \dimPTwo^2} \right) + \OO \left( \frac{\tol^4 \numSamples^3}{\dimPOne^3 \dimPTwo \predError} \right) = \OO \left( \frac{\tol^4 \numSamples^3}{\dimPOne^3 \dimPTwo \predError} \right).
    \end{align}
\end{proof}


%% file: appendix/app05-formulaMutualInfoProof.tex
\MIFormula*

\begin{proof}

\noindent The standard definition of mutual information is given by:
\begin{align*}
    I(X;A) &= \sum_v \sum_{x\in\{0,1\}} \ensuremath{\Pr{A=v, X=x}} \log\left(\frac{\ensuremath{\Pr{A=v, X=x}}}{\ensuremath{\Pr{X=x}}\ensuremath{\Pr{A=v}}}\right) \\
    &= \sum_v \sum_{x\in\{0,1\}} \ensuremath{\Pr{X=x}} \ensuremath{\Pr{A=v \mid X=x}} \log\left(\frac{\ensuremath{\Pr{A=v \mid X=x}}}{\ensuremath{\Pr{A=v}}}\right) 
\end{align*}
For simplicity, we introduce the following notation:
\begin{itemize}
    \item $\beta = \ensuremath{\Pr{A=v \mid X=x_{\max}(v)}}$
    \item $\gamma = \ensuremath{\Pr{A=v \mid X=x_{\min}(v)}}$
    \item $p = \ensuremath{\Pr{A=v}} = \frac{1}{2} \ensuremath{\Pr{A=v \mid X=x_{\max}(v)}} + \frac{1}{2} \ensuremath{\Pr{A=v \mid X=x_{\min}(v)}}= \frac{\beta+\gamma}{2}$
\end{itemize}
Substituting these into the definition of mutual information gives:
\begin{align*}
    I(X;A) &= \sum_v \sum_{x\in\{0,1\}} \ensuremath{\Pr{X=x}} \ensuremath{\Pr{A=v \mid X=x}} \log\left(\frac{\ensuremath{\Pr{A=v \mid X=x}}}{\ensuremath{\Pr{A=v}}}\right)\\
    &= \sum_v \frac{1}{2}\beta \log\left(\frac{\beta}{p}\right) + \frac{1}{2}\gamma \log\left(\frac{\gamma}{p}\right) \\
    &= \sum_v \frac{1}{2}\beta \log\beta + \frac{1}{2}\gamma \log\gamma -\frac{\beta+\gamma}{2}\log p\\
    &= \sum_v \frac{\beta-p}{2} \log\beta + \frac{\gamma-p}{2} \log\gamma +\frac{p}{2}\left(\log \beta + \log \gamma\right)-p\log p \\
    &= \sum_v \frac{\beta-\gamma}{4} \log\beta + \frac{\gamma-\beta}{4} \log\gamma +p\log \sqrt{\beta \gamma} -p \log p\\
    &= \sum_v \frac{\beta-\gamma}{4} \log\left(\frac{\beta}{\gamma}\right) +p\log \left(\frac{\sqrt{\beta \gamma}}{p}\right)\\
    &\leq \sum_v \frac{\beta-\gamma}{4 \ln(2)} \left(\frac{\beta}{\gamma}-1\right) \\
    &= \sum_v \frac{1}{4 \ln(2)}\frac{(\beta-\gamma)^2}{\gamma} \\
    &= \sum_v \frac{\gamma}{4 \ln(2)}\left(\frac{\beta}{\gamma}-1\right)^2,
\end{align*}
where the last inequality holds since $\ln(x) \leq x-1$ and the geometric mean is always less than or equal to the arithmetic mean, i.e., $\sqrt{\beta\gamma} \leq p=\frac{\beta+\gamma}{2}$. Substituting the expressions for $\beta$ and $\gamma$ gives the desired upper-bound on the mutual information:
\begin{align*}
    I(X;A) &= \sum_{v} \OO\left( \pXmin \left( 1- \frac{\pXmax}{\pXmin} \right)^2 \right).
\end{align*}

\end{proof}

%% file: appendix/app12-informationTheoryLemmas.tex

\begin{restatable}{lem}{MIboundZero}
    \label{lem:MIboundnorm0}
        For an arbitrary $i \in [\dimPOne]$, let $A \coloneqq A_i$ be a vector of $\dimPTwo$ counts, such that $a_{i, j} \sim \Poi(\numSamples \prodMeasLB(i, j))$ for all $j \in [\dimPTwo]$, where $\prodMeasLB$ is constructed as in \eqref{eqn:drawMeasure}. For all $v \in \mathbb{N}^\dimPTwo$, define $x_{\min}(v)$ and $x_{\max}(v)$ as in \eqref{eqn:xMin_xMax}. Then, the following holds.
        \begin{align*}
            \sum_{v: \norm{v}_1 = 0} \OO &\left( \pXmin \left( 1- \frac{\pXmaxGivenN}{\pXminGivenN} \right)^2 \right) \\
            &\leq \OO \left( \frac{\tol^4 \numSamples^4}{\dimPOne^4 \dimPTwo^2} \right).
        \end{align*}
\end{restatable}

\begin{proof}
    
     First, note that there is only one $v$ such that $\norm{v}_1 = 0$, allowing us to remove the summation and replace $v$ with $\mathbf{0}$. We can further simplify this expression by upper-bounding $\pXmin$ by 1. This leaves us with:
    \begin{align}
        \sum_{v: \norm{v}_1 = 0} &\OO\left( \pXmin \left( 1- \frac{\pXmaxGivenN}{\pXminGivenN} \right)^2 \right) \nonumber \\
        &\quaaad \leq \OO\left( \left( 1- \frac{\pZeroGivenXmaxN}{\pZeroGivenXminN} \right)^2 \right). \label{eqn:MIv0Simplified}
    \end{align}
    To bound this expression, we calculate $\pZeroGivenZeroN$ and $\pZeroGivenOneN$. As, conditioned on $X=x$, each $a_{ij}$ is i.i.d., the probability that $\abs{A} = 0$ is the product of the probabilities that $a_{ij} = 0$. As a result, we have:
    \begin{align*}
        \pZeroGivenZeroN = \prod_{j=1}^\dimPTwo \ensuremath{\Pr{a_{i,j}=0}} = \prod_{j=1}^\dimPTwo e^{-\frac{\numSamples}{\dimPOne\dimPTwo}}= e^{-\frac{\numSamples}{\dimPOne}},
    \end{align*}
    and
    \begin{align*}
        \pZeroGivenOneN 
        &= \prod_{j=1}^\dimPTwo \ensuremath{\Pr{a_{i,j}=0}}\\
        &= \prod_{j=1}^\dimPTwo \frac{1}{2} \left(e^{-\frac{\numSamples(1+\tol)}{\dimPOne\dimPTwo}}+e^{-\frac{\numSamples(1-\tol)}{\dimPOne\dimPTwo}}\right) \\
        &= \prod_{j=1}^\dimPTwo e^{-\frac{\numSamples}{\dimPOne \dimPTwo}} \left(\frac{e^{-\frac{\numSamples\tol}{\dimPOne\dimPTwo}}+e^{\frac{\numSamples\tol}{\dimPOne\dimPTwo}}}{2}\right)\\
        &= e^{-\frac{\numSamples}{\dimPOne}} \cdot \cosh^m \left( \frac{\numSamples \tol}{\dimPOne \dimPTwo} \right).
    \end{align*}
    As $\cosh(x) \geq 1$, we know $x_{\max} = 1$ and $x_{\min} = 0$. The ratio between the probabilities in question is therefore upper-bounded as:
    \begin{equation*}
        \cosh^m \left( \frac{\numSamples \tol}{\dimPOne \dimPTwo} \right) \leq e^{\frac{\dimPTwo}{2}\left(\frac{\numSamples\tol}{\dimPOne\dimPTwo}\right)^2} \leq 1 + \frac{\numSamples^2 \tol^2}{\dimPOne^2 \dimPTwo} ,
    \end{equation*}
    where we used $\cosh(x) \leq e^{x^2/2}$, and $e^x \leq 1+2x$ for all $x\in[0,1]$.
    This implies that \eqref{eqn:MIv0Simplified} satisfies:
    \begin{equation*}
        \OO\left( \left( 1- \frac{\pZeroGivenXmaxN}{\pZeroGivenXminN} \right)^2 \right)
        \leq \OO\left( \frac{\numSamples^4 \tol^4}{\dimPOne^4 \dimPTwo^2}\right).
    \end{equation*}
\end{proof}


\begin{restatable}{lem}{MIboundOne}
\label{lem:MIboundnorm1}
    For an arbitrary $i \in [\dimPOne]$, let $A \coloneqq A_i$ be a vector of $\dimPTwo$ counts, such that $a_{i, j} \sim \Poi(\numSamples \prodMeasLB(i, j))$ for all $j \in [\dimPTwo]$, where $\prodMeasLB$ is constructed as in \eqref{eqn:drawMeasure}. For all $v \in \mathbb{N}^\dimPTwo$, define $x_{\min}(v)$ and $x_{\max}(v)$ as in \eqref{eqn:xMin_xMax}. Then, the following holds.
    \begin{align*}
        \sum_{v: \norm{v}_1 = 1} \OO &\left(\pXmin \left( 1- \frac{\pXmaxGivenN}{\pXminGivenN} \right)^2 \right) \\
        &\leq \OO \left( \frac{\tol^4 \numSamples^3}{\dimPOne^3 \dimPTwo^2} \right).
    \end{align*}
\end{restatable}

\begin{proof}

    There are $\binom{m}{1}$ standard basis vectors such that $\lvert v \rvert_1=1$, each having a single 1 in coordinate $j'$ and zeros elsewhere. The probability of $A=v$ conditioned on $X=x$ and $ c_i = 1/\dimPOne$ can be calculated considering the independence of each $a_{i,j}$ conditioned on $X=x$ as:
    \begin{align}
        \label{eq: Prob_v1_X0}
        \pOneGivenZeroN &= \ensuremath{\Pr{a_{i,j'}=1}}\prod_{j\neq j'} \ensuremath{\Pr{a_{i,j}=0}} \notag \\
        &=\frac{\numSamples}{\dimPOne \dimPTwo} e^{-\frac{\numSamples}{\dimPOne\dimPTwo}} \left(e^{-\frac{\numSamples}{\dimPOne\dimPTwo}}\right)^{\dimPTwo-1} \notag \\
        &=\frac{\numSamples}{\dimPOne \dimPTwo} e^{-\frac{\numSamples}{\dimPOne}},
    \end{align}
    \begin{align}
        \label{eq: Prob_v1_X1}
        &\pOneGivenOneN = \ensuremath{\Pr{a_{i,j'}=1}}\prod_{j\neq j'} \ensuremath{\Pr{a_{i,j}=0}} \notag\\
        &= \frac{1}{2} \left(\frac{\numSamples(1+\tol)}{\dimPOne\dimPTwo} e^{-\frac{\numSamples(1+\tol)}{\dimPOne\dimPTwo}} + \frac{\numSamples(1-\tol)}{\dimPOne\dimPTwo} e^{-\frac{\numSamples(1-\tol)}{\dimPOne\dimPTwo}}\right) \prod_{j\neq j'} \frac{1}{2} \left(e^{-\frac{\numSamples(1+\tol)}{\dimPOne\dimPTwo}}+e^{-\frac{\numSamples(1-\tol)}{\dimPOne\dimPTwo}}\right) \notag\\
        &= \frac{\numSamples}{\dimPOne\dimPTwo} e^{-\frac{\numSamples}{\dimPOne\dimPTwo}} \left(\cosh\left(\frac{\numSamples\tol}{\dimPOne\dimPTwo}\right)-\tol \sinh\left(\frac{\numSamples\tol}{\dimPOne\dimPTwo}\right)\right) \prod_{j\neq j'} e^{-\frac{\numSamples}{\dimPOne\dimPTwo}} \cosh\left(\frac{\numSamples\tol}{\dimPOne\dimPTwo}\right) \notag\\
        &= \frac{\numSamples}{\dimPOne\dimPTwo} e^{-\frac{\numSamples}{\dimPOne}} \cosh^\dimPTwo \left(\frac{\numSamples\tol}{\dimPOne\dimPTwo}\right) \left(1-\tol \tanh\left(\frac{\numSamples\tol}{\dimPOne\dimPTwo}\right)\right),
    \end{align}
    We now bound the ratio between probabilities in \eqref{eq: Prob_v1_X0} and \eqref{eq: Prob_v1_X1} by setting $u \coloneqq \frac{\numSamples\tol}{\dimPOne\dimPTwo}$ and applying the inequalities $1 \leq \cosh^\dimPTwo(u) \leq e^{\dimPTwo u^2/2}$, and $u-u^3/3 \leq \tanh(u) \leq u$ for all $u\ge 0$. Note that $\dimPTwo u \leq \frac{\tol}{2}$ since $\frac{\numSamples}{\dimPOne}\leq 0.5$.
    \begin{align}
        \label{eq: tanhLB}
        1-\tol u &\leq \cosh^\dimPTwo(u)\left(1-\tol \tanh(u)\right) \leq e^{\frac{\dimPTwo u^2}{2}}\left(1-\tol u + \tol \frac{u^3}{3}\right) \notag\\
        1-\tol u &\leq \cosh^\dimPTwo(u)\left(1-\tol \tanh(u)\right) \leq e^{\frac{\tol u}{4}}\left(1-\frac{11}{12}\tol u\right) \leq 1\notag\\
        \Rightarrow \quad \frac{1}{1-\tol u} &\geq \frac{1}{\cosh^\dimPTwo(u)\left(1-\tol \tanh(u)\right)} \geq 1
    \end{align}
    Since $\frac{\numSamples}{\dimPOne}\leq 0.5$, we have $0 \leq \tol u \leq 0.5$. Therefore, we can use the inequality $\frac{1}{1-\tol u} \leq 1+2\tol u$ and thus:
    \begin{equation}
        \label{eq:ProbRatio_v1}
        1 \leq \frac{1}{\cosh^m \left( \frac{\numSamples \tol}{\dimPOne \dimPTwo} \right) \left( 1 - \tol \tanh(\frac{\numSamples\tol}{\dimPOne \dimPTwo}) \right)} \leq 1 + \frac{2\numSamples \tol^2}{\dimPOne \dimPTwo} 
    \end{equation}
    The next step is to demonstrate that $\sum_{v: \norm{v}_1 = 1} \pXmin = \OO\left( \numSamples/\dimPOne \right)$. First, by the definition of $x_{\min}(v)$, we have
    \begin{align}
        \sum_{v: \norm{v}_1 = 1} \pXmin \leq  \sum_{v: \norm{v}_1 = 1} \pVZero.
    \end{align}
    We can then bound this sum by applying the law of total probability to consider the different cases for $c_i$ for each term:
    \begin{align}
    \label{eq: ProbAis1}
        \sum_{v: \norm{v}_1 = 1} \pVZero
        &=  \sum_{v: \norm{v}_1 = 1} \left[ \frac{\predError\numSamples}{\dimPOne}\pXzeroGivenK \right. \notag \\
        &\quaaad \left. + \left(1-\frac{\predError\numSamples}{\dimPOne}\right)\pXzeroGivenN \right] \notag\\
        &= \sum_{v: \norm{v}_1 = 1} \frac{\predError\numSamples}{\dimPOne} \left(\frac{1}{\dimPTwo}e^{-\frac{1}{\dimPTwo}} \left(e^{-\frac{1}{\dimPTwo}}\right)^{\dimPTwo-1}\right) + \left(1-\frac{\predError\numSamples}{\dimPOne}\right) \frac{\numSamples}{\dimPOne \dimPTwo} e^{-\frac{\numSamples}{\dimPOne}} \notag \\
        & = \frac{\numSamples}{\dimPOne} \left( \predError e^{-1} + \left(1-\frac{\predError\numSamples}{\dimPOne}\right) e^{- \frac{\numSamples}{\dimPOne}}\right) \notag  \\
        &\leq \OO\left( \frac{\numSamples}{\dimPOne}\right)
    \end{align}
    The last inequality holds since $0\leq \frac{\numSamples}{\dimPOne}\leq 0.5$; hence, the term inside the parentheses remains a constant within $[e^{-\frac{1}{2}},1+e^{-1}]$. We can combine the bounds from \eqref{eq:ProbRatio_v1} and \eqref{eq: ProbAis1} to derive the desired upper bound of $\OO\left(\frac{\numSamples}{\dimPOne}\left(\frac{\numSamples \tol^2}{\dimPOne \dimPTwo}\right)^2\right) = \OO\left(\frac{\numSamples^3 \tol^4}{\dimPOne^3 \dimPTwo^2}\right)$ , completing the proof.
    
\end{proof}

\begin{restatable}{lem}{MIboundTwo}
\label{lem:MIboundnorm2}
    For an arbitrary $i \in [\dimPOne]$, let $A \coloneqq A_i$ be a vector of $\dimPTwo$ counts, such that $a_{i, j} \sim \Poi(\numSamples \prodMeasLB(i, j))$ for all $j \in [\dimPTwo]$, where $\prodMeasLB$ is constructed as in \eqref{eqn:drawMeasure}. For all $v \in \mathbb{N}^\dimPTwo$, define $x_{\min}(v)$ and $x_{\max}(v)$ as in \eqref{eqn:xMin_xMax}. Then, the following holds.
    \begin{equation*}
        \sum_{v: \norm{v}_1 \geq 2} \OO\left( \pXmin \left( 1- \frac{\pXmax}{\pXmin} \right)^2 \right) \leq \OO \left( \frac{\tol^4 \numSamples^3}{\dimPOne^3 \dimPTwo \predError} \right)
    \end{equation*}
\end{restatable}

To prove Lemma~\ref{lem:MIboundnorm2}, we use Lemma~\ref{lem: ProbRatio_LightHeavyRows} to show that we can bound the given expression by $\OO\left(\frac{\numSamples}{\dimPOne \predError}\right) \cdot I\left(X; A \mid c_i = 1 / \dimPOne \right)$. 
Then, the chain rule allows us to bound $I\left(X; A \mid c_i = 1 / \dimPOne \right)$ by the sum over all $j$ of  $I\left(X; a_{i, j} \mid c_i = 1 / \dimPOne \right)$. We then use the result of Lemma~\ref{lem:MIofXwith_aij} to bound each term of this sum.

\begin{proof}
    
    Let $v$ be a possible realization of $A$ such that $\norm{v}_1 \geq 2$. Consider the term of our mutual information bound involving $v$:
    \begin{equation}
    \label{eqn:MISingleTerm}
        \pXmin \left( 1 - \frac{\pXmax}{\pXmin} \right)^2.
    \end{equation}
    Applying the law of total probability, Equation~\eqref{eqn:MISingleTerm} is equal to:
    \begin{align}
        &\left(\pXminAndK + \pXminAndN \right) \nonumber \\
        &\quaaaad \cdot \left( 1 - \frac{\pXmaxAndK + \pXmaxAndN}{\pXminAndK + \pXminAndN} \right)^2 \label{eqn:MISingleTermTotalProb}.
    \end{align}
    Note that, if $c_i = 1/\numSamples$, $A$ does not depend on $X$; whether $X$ is 0 or 1, $a_{ij} \sim \Poi(1 / \dimPTwo)$. Throughout the next few steps of the proof, we adopt the following notation for simplicity:
    \begin{itemize}
        \item $\beta = \pXmaxAndN$, 
        \item $\gamma = \pXminAndN$,  
        \item $\kappa = \pXmaxAndK = \pXminAndK$. 
    \end{itemize}
    With this notation, Equation~\eqref{eqn:MISingleTermTotalProb} can be rewritten as:
    \begin{equation}
        \left(\kappa + \gamma \right) \left( 1- \frac{\kappa + \beta}{\kappa + \gamma} \right)^2 
        = \left(\kappa + \gamma \right) \left( \frac{\gamma - \beta}{\kappa + \gamma} \right)^2 \nonumber 
        = \frac{1}{\kappa + \gamma} \cdot \left( \gamma - \beta \right)^2 \label{eqn:MISingleTermVars}.
    \end{equation}
    By Lemma~\ref{lem: ProbRatio_LightHeavyRows}, for all $v$ such that $\norm{v}_1 \geq 2$, and for all $x$, we have: 

        \begin{equation*}
            \pXxAndn = \OO\left(\frac{\numSamples}{\dimPOne \predError}\right) \pXxAndk.
        \end{equation*}
    Applying this lemma with $x = x_{\min}(v)$ and rewriting the bound under our notation, we have:
    \begin{align*}
        \gamma 
        &\leq \OO\left( \frac{\numSamples}{\dimPOne\predError} \right) \cdot \kappa \\
        &\implies \gamma + \OO\left( \frac{\numSamples}{\dimPOne\predError} \right) \cdot \gamma 
        \leq \OO\left( \frac{\numSamples}{\dimPOne\predError} \right) \cdot \kappa + \OO\left( \frac{\numSamples}{\dimPOne\predError} \right) \cdot \gamma \\
        &\implies \gamma + \OO\left( \frac{\numSamples}{\dimPOne\predError} \right) \cdot \gamma 
        \leq \OO\left( \frac{\numSamples}{\dimPOne\predError} \right) \cdot (\kappa + \gamma) \\
        &\implies \frac{1}{\kappa + \gamma}
        \leq \frac{\OO\left( \frac{\numSamples}{\dimPOne\predError} \right)}{\gamma \left(1 + \OO\left( \frac{\numSamples}{\dimPOne\predError} \right) \right)} \\
        &\implies \frac{1}{\kappa + \gamma}
        \leq \frac{\OO\left( \frac{\numSamples}{\dimPOne\predError} \right)}{\gamma}. \\
    \end{align*}
    Therefore, we can lower bound Equation~\eqref{eqn:MISingleTermVars} as:
    \begin{equation*}
        \frac{1}{\kappa + \gamma} \cdot \left( \gamma - \beta \right)^2 \leq \OO\left( \frac{\numSamples}{\dimPOne\predError} \right) \cdot \frac{\left( \gamma - \beta \right)^2}{\gamma} = \OO\left( \frac{\numSamples}{\dimPOne\predError} \right) \cdot \gamma \cdot \left( \frac{\gamma - \beta}{\gamma} \right)^2 = \OO\left( \frac{\numSamples}{\dimPOne\predError} \right) \cdot \gamma \cdot \left( 1 - \frac{\beta}{\gamma} \right)^2.
    \end{equation*}
    Substituting for the values of $\beta$ and $\gamma$, this bound is:
    \begin{align}
        &\OO\left( \frac{\numSamples}{\dimPOne\predError} \right) \cdot \pXminAndN \cdot \left( 1 - \frac{\pXmaxAndN}{\pXminAndN} \right)^2 \nonumber \\
        &\quaad= \OO\left( \frac{\numSamples}{\dimPOne\predError} \right) \cdot \pXminGivenN \; \pOneOverN \nonumber \\
        &\quaaaaaad \cdot \left( 1 - \frac{\pXmaxGivenN \; \pOneOverN}{\pXminGivenN \; \pOneOverN} \right)^2 \nonumber \\
        &\quaad= \OO\left( \frac{\numSamples}{\dimPOne\predError} \right) \cdot \left( 1 - \frac{\predError\numSamples}{\dimPOne} \right) \cdot \pXminGivenN \\
        &\quaaaaaad \cdot \left( 1 - \frac{\pXmaxGivenN}{\pXminGivenN} \right)^2 \nonumber \\
        &\quaad\leq \OO\left( \frac{\numSamples}{\dimPOne\predError} \right) \cdot \pXminGivenN \cdot \left( 1 - \frac{\pXmaxGivenN}{\pXminGivenN} \right)^2 \nonumber.
    \end{align}

    \noindent Thus far, we have shown that, for a fixed $v$ such that $\norm{v}_1 \geq 2$:
    \begin{align}
    \label{eqn:MISingleTermConditionalBound}
        &\pXmin \left( 1- \frac{\pXmax}{\pXmin} \right)^2 \nonumber \\ 
        &\quad \leq \OO\left( \frac{\numSamples}{\dimPOne\predError} \right) \cdot \pXminGivenN \cdot \left( 1 - \frac{\pXmaxGivenN}{\pXminGivenN} \right)^2
    \end{align}
    We can apply this bound to bound the sum over all $v$ with $\norm{v}_1 \geq 2$ as:
    \begin{align}
        \sum_{v: \norm{v}_1 \geq 2} &\OO\left( \pXmin \left( 1- \frac{\pXmax}{\pXmin} \right)^2 \right) \nonumber \\
        &\leq \sum_{v: \norm{v}_1 \geq 2} \OO\left( \frac{\numSamples}{\dimPOne\predError} \cdot \pXminGivenN \cdot \left( 1 - \frac{\pXmaxGivenN}{\pXminGivenN} \right)^2 \right) \nonumber \\
        &\leq  \OO\left( \frac{\numSamples}{\dimPOne\predError} \right) \sum_{v} \OO \left( \pXminGivenN \cdot \left( 1 - \frac{\pXmaxGivenN}{\pXminGivenN} \right)^2 \right) \nonumber \\
        &= \OO\left( \frac{\numSamples}{\dimPOne \predError} \right) \Theta \left( I\left( X ; A \mid c_i = \frac{1}{\dimPOne} \right) \right) \label{eqn:MINorm2ConditionalMI}.
    \end{align}
    The chain rule for mutual information allows us to bound $I(X; A \mid c_i = 1/\dimPOne)$ as:
    \begin{equation*}
        I\left( X ; A \mid c_i = \frac{1}{\dimPOne} \right) \leq \sum_{j=1}^\dimPTwo I\left( X ; a_{i,j} \mid c_i = \frac{1}{\dimPOne} \right)
        \label{eqn:sumIndivMIs}
    \end{equation*}
    When $c_i = 1 / \dimPOne$, $a_{i, j}$ follows the distribution of the random variable considered in Lemma~\ref{lem:MIofXwith_aij}. This lemma allows us to bound the mutual information between $X$ and each element of row $i$ in the sample set.
    
    \noindent Applying the bound given in the lemma, we have:
    \begin{equation}
    \label{eqn:MINorm2ChainRule}
        \sum_{j=1}^\dimPTwo I\left( X ; a_{i,j} \mid c_i = \frac{1}{\dimPOne} \right) \leq 
        \sum_{j=1}^\dimPTwo \OO\left( \frac{\tol^4 \numSamples^2}{\dimPOne^2 \dimPTwo^2} \right) = \OO\left( \frac{\tol^4 \numSamples^2}{\dimPOne^2 \dimPTwo} \right)
    \end{equation}
    Ultimately, combining Equation~\eqref{eqn:MINorm2ConditionalMI} and Equation~\eqref{eqn:MINorm2ChainRule}, we have our desired bound:
    \begin{equation*}
        \sum_{v: \norm{v}_1 \geq 2} \OO\left( \pXmin \left( 1- \frac{\pXmax}{\pXmin} \right)^2 \right) \leq \OO\left( \frac{\tol^4 \numSamples^3}{\dimPOne^3 \dimPTwo \predError} \right)
    \end{equation*}
\end{proof}


\begin{restatable}{lem}{LightHeavyRatio}
\label{lem: ProbRatio_LightHeavyRows}
    Let $x\in\{0,1\}, v \in \mathbb{N}^\dimPTwo$ such that  $\norm{v}_1 \geq 2\}$. Let $c_i$ be a random variable and $\prodMeasLB$ a distribution, defined according to \eqref{eqn:drawMeasure} where $X = x$. For an arbitrary $i \in [\dimPOne]$, let $A \coloneqq A_i$ be a vector of $\dimPTwo$ counts, such that $a_{i, j} \sim \Poi(\numSamples \prodMeasLB(i, j))$ for all $j \in [\dimPTwo]$. Then, the following holds.
    \begin{equation*}
        \pXxAndn = \OO\left(\frac{\numSamples}{\dimPOne \predError}\right) \pXxAndk.
    \end{equation*}
\end{restatable}

\begin{proof}

    The right-hand side of the equation is independent of the particular value of $X$, since it corresponds to the rows with $c_i = 1/\numSamples$. For such rows, each element follows a Poisson distribution $a_{i,j}\sim \Poi\left(1/m\right)$. 
    \begin{align}
        \label{eq: HeavyRows}
        \pXxAndk &= \ensuremath{\Pr{c_i=1/\numSamples}} \ensuremath{\Pr{A = v \mid X=x, c_i=1/\numSamples}} \notag\\
        &= \frac{\predError \numSamples}{\dimPOne} \prod_{j=1}^\dimPTwo \left(\frac{1}{\dimPTwo}\right)^{v_j} e^{-\frac{1}{\dimPTwo}}\frac{1}{v_j!} = \frac{\predError \numSamples e^{-1}}{\dimPOne \dimPTwo^{\lVert v \lVert_1 }} \prod_{j=1}^\dimPTwo \frac{1}{v_j!}
    \end{align}
    To establish  the lemma, we upper-bound the left-hand side for both cases $X=0$ and $X=1$, and then show that the larger of these two bounds meets the claimed asymptotic condition.
    \begin{align}
        \label{eq: Xis0}
        \ensuremath{\Pr{A=v, c_i=1/\dimPOne \mid X=0}} &= \ensuremath{\Pr{c_i=1/\dimPOne}} \ensuremath{\Pr{A = v \mid X=0, c_i=1/\dimPOne}} \notag\\
        &= \left(1-\frac{\predError\numSamples}{\dimPOne}\right) \prod_{j=1}^\dimPTwo \left(\frac{\numSamples}{\dimPOne\dimPTwo}\right)^{v_j} e^{-\frac{\numSamples}{\dimPOne\dimPTwo}}\frac{1}{v_j!} \notag\\
        &\leq \left(\frac{\numSamples}{\dimPOne\dimPTwo}\right)^{\lVert v\lVert_1}e^{-\frac{\numSamples}{\dimPOne}} \prod_{j=1}^\dimPTwo \frac{1}{v_j!}
    \end{align}
    \begin{align}
        \label{eq: Xis1}
        &\ensuremath{\Pr{A=v, c_i=1/\dimPOne \mid X=1}} \notag \\
        &\quad = \ensuremath{\Pr{c_i=1/\dimPOne}} \ensuremath{\Pr{A = v \mid X=1,  c_i=1/\dimPOne}} \notag\\
        &\quad=\left(1-\frac{\predError\numSamples}{\dimPOne}\right) \prod_{j=1}^\dimPTwo \frac{1}{2}\left( \left(\frac{\numSamples\left(1+\tol\right)}{\dimPOne\dimPTwo}\right)^{v_j} e^{-\frac{\numSamples(1+\tol)}{\dimPOne\dimPTwo}}+ \left(\frac{\numSamples\left(1-\tol\right)}{\dimPOne\dimPTwo}\right)^{v_j} e^{-\frac{\numSamples(1-\tol)}{\dimPOne\dimPTwo}}\right)\frac{1}{v_j!} \notag\\
        &\quad\leq \left(\frac{\numSamples}{\dimPOne\dimPTwo}\right)^{\lVert v\lVert_1}e^{-\frac{\numSamples}{\dimPOne}} \prod_{j=1}^\dimPTwo \frac{(1+\tol)^{v_j}}{2} \left(e^{-\frac{\numSamples\tol}{\dimPOne\dimPTwo}}+\left(\frac{1-\tol}{1+\tol}\right)^{v_j} e^{\frac{\numSamples\tol}{\dimPOne\dimPTwo}}\right)\frac{1}{v_j!} \notag\\
        &\quad\leq \left(\frac{\numSamples(1+\tol)}{\dimPOne\dimPTwo}\right)^{\lVert v\lVert_1}e^{-\frac{\numSamples\left(1-\tol\right)}{\dimPOne}} \prod_{j=1}^\dimPTwo \frac{1}{v_j!}
    \end{align}
    It is trivial that the upper-bound in \eqref{eq: Xis1} is greater than the upper-bound in \eqref{eq: Xis0}. Therefore, it suffices to show that the ratio of the probabilities in \eqref{eq: Xis1} and \eqref{eq: HeavyRows} is of order $O(\frac{\numSamples}{\dimPOne\predError})$ for any $v$ such that $\lVert v \lVert_1\geq 2$:
    \begin{align}
        \frac{ \left(\frac{\numSamples(1+\tol)}{\dimPOne\dimPTwo}\right)^{\lVert v\lVert_1}e^{-\frac{\numSamples\left(1-\tol\right)}{\dimPOne}} \prod_{j=1}^\dimPTwo \frac{1}{v_j!}}{\frac{\predError \numSamples e^{-1}}{\dimPOne \dimPTwo^{\lVert v \lVert_1 }} \prod_{j=1}^\dimPTwo \frac{1}{v_j!}} = \frac{\numSamples}{\dimPOne\predError} \cdot \left(1+\tol\right)^2\left(\frac{\numSamples(1+\tol)}{\dimPOne}\right)^{\lVert v \lVert_1-2} e^{1-\frac{\numSamples(1-\tol)}{\dimPOne}} \leq 4e \frac{\numSamples}{\dimPOne\predError} = \OO\left(\frac{\numSamples}{\dimPOne\predError}\right)
    \end{align}
\end{proof}


\begin{restatable}{lem}{MISmalla}
\label{lem:MIofXwith_aij}
    Let $\predError \in [0, 1], \tol \in (0, 1), \numSamples \in \mathbb{N}$.
    Let $X$ be either $0$ or $1$ with equal probability. Let $\prodMeasLB$ be a distribution over $[\dimPOne] \times [\dimPTwo]$ constructed as in \eqref{eqn:drawMeasure} using $X$. For an arbitrary $i \in [\dimPOne]$ corresponding to rows with $c
    _i=\frac{1}{\dimPOne}$, let $A \coloneqq A_i$ be a vector of $\dimPTwo$ counts, such that $a_{i, j} \sim \Poi(\numSamples \prodMeasLB(i, j))$ for all $j \in [\dimPTwo]$.
    Then, the mutual information between $X$ and $a_{ij}$ is bounded as

    \begin{equation*}
        I\left(X; a_{ij}\mid c_i = \frac{1}{\dimPOne} \right) =  \OO \left( \frac{\tol^4 \numSamples^2}{\dimPOne^2 \dimPTwo^2} \right)
    \end{equation*}
\end{restatable}  

\begin{proof} 

    \noindent For all $l$, define $x_{\min}(l)$ and $x_{\max}(l)$ as:
    \begin{gather}
        x_{\min}(l) = \argmin_{x \in \{0, 1\}} \left( \paX \right), \quad 
        x_{\max}(l)  = \argmax_{x \in \{0, 1\}} \left( \paX \right).
    \end{gather}
    Using an argument similar to that in Appendix~\ref{Apdx: MI_Formula}, we can show that the mutual information is given by:
    \begin{equation}
       I(X;a) = \sum_{l=0}^\infty \OO\left( \paXmin \left( 1- \frac{\paXmax}{\paXmin} \right)^2 \right)
       \label{eqn:InfoTheoryProofs-singleA-sumOverL}
    \end{equation}
    For any $l$, we have that:
    \begin{align}
        \paXzero &= \frac{e^{-\frac{\numSamples}{\dimPOne\dimPTwo}}\left(\frac{\numSamples}{\dimPOne\dimPTwo}\right)^l}{l!} \\
        \paXone &= \frac{1}{2}\left(\frac{e^{-\frac{\numSamples(1+\tol)}{\dimPOne\dimPTwo}}\left(\frac{\numSamples(1+\tol)}{\dimPOne\dimPTwo}\right)^l}{l!}\right)+\frac{1}{2}\left(\frac{e^{-\frac{\numSamples(1-\tol)}{\dimPOne\dimPTwo}}\left(\frac{\numSamples(1-\tol)}{\dimPOne\dimPTwo}\right)^l}{l!}\right) \notag\\
        &= \frac{e^{-\frac{\numSamples}{\dimPOne\dimPTwo}}\left(\frac{\numSamples}{\dimPOne\dimPTwo}\right)^l}{l!} \left(\frac{e^{-\frac{\numSamples\tol}{\dimPOne\dimPTwo}}(1+\tol)^l + e^{\frac{\numSamples\tol}{\dimPOne\dimPTwo}}(1-\tol)^l}{2}\right)
    \end{align}
    \begin{align}
        \label{eqn: Rl}
        \frac{\paXone}{\paXzero} 
        &= \frac{e^{-\frac{\numSamples\tol}{\dimPOne\dimPTwo}}(1+\tol)^l + e^{\frac{\numSamples\tol}{\dimPOne\dimPTwo}}(1-\tol)^l}{2} \notag \\
        &= \cosh\left(\frac{\numSamples\tol}{\dimPOne\dimPTwo}\right) \left(\sum_{i=0}^{\floor{l/2}} \binom{l}{2i}\tol^{2i}\right) - \sinh\left(\frac{\numSamples\tol}{\dimPOne\dimPTwo}\right) \left(\sum_{i=0}^{\floor{\frac{l-1}{2}}} \binom{l}{2i+1}\tol^{2i+1}\right)
    \end{align}
    We will split the sum in \eqref{eqn:InfoTheoryProofs-singleA-sumOverL} over $l=0$, $l=1$, $2\leq l < 2/\tol$, and $l\geq 2/\tol$. For $l=0$, the ratio in \eqref{eqn: Rl} is equal to $\cosh\left(\frac{\numSamples\tol}{\dimPOne\dimPTwo}\right) $ which is always greater than 1 and less than $1+\left(\frac{\numSamples\tol}{\dimPOne\dimPTwo}\right)^2$. Therefore, the contribution to the mutual information can be bounded as
    \begin{equation}
        \label{eq: MI_L0}
        \paXminLzero \left( 1- \frac{\paXmaxLzero}{\paXminLzero} \right)^2 \leq \left(\frac{\numSamples\tol}{\dimPOne\dimPTwo}\right)^4 .
    \end{equation}
    For $l=1$, using reasoning analogous to the derivation of \eqref{eq: tanhLB}, we can verify that $x_{\max}(l) = 0$ and $x_{\min}(l)=1$, and we can bound the reciprocal of the ratio in \eqref{eqn: Rl} as follows:
    \begin{align*}
         \frac{\ensuremath{\Pr{a=1 \mid X=1}}}{\ensuremath{\Pr{a=1 \mid X=0}}} &= \cosh\left(\frac{\numSamples\tol}{\dimPOne\dimPTwo}\right)-\tol \sinh\left(\frac{\numSamples\tol}{\dimPOne\dimPTwo}\right) \\
         &= \cosh\left(\frac{\numSamples\tol}{\dimPOne\dimPTwo}\right) \left(1-\tol \tanh\left(\frac{\numSamples\tol}{\dimPOne\dimPTwo}\right) \right) \\
         &\geq 1\cdot\left(1-\frac{\numSamples\tol^2}{\dimPOne\dimPTwo}\right) \\
         \implies \frac{\paXmaxLone}{\paXminLone} 
         &= \frac{1}{\cosh\left(\frac{\numSamples\tol}{\dimPOne\dimPTwo}\right)-\tol \sinh\left(\frac{\numSamples\tol}{\dimPOne\dimPTwo}\right)}  
         \leq \frac{1}{1-\frac{\numSamples\tol^2}{\dimPOne\dimPTwo}}
         \leq 1+\frac{2\numSamples\tol^2}{\dimPOne\dimPTwo}.
    \end{align*}
    Therefore, the contribution to the mutual information can be bounded as
    \begin{equation}
        \label{eq: MI_L1}
        \paXminLone \left( 1- \frac{\paXmaxLone}{\paXminLone} \right)^2 \leq \OO\left(\frac{\numSamples^2\tol^4}{\dimPOne^2\dimPTwo^2}\right) .
    \end{equation}
     For $2\leq l < 2/\tol$, using the inequality $\cosh(u)\leq 1+u^2$ for all $u\in[0,1]$, we can bound the ratio in \eqref{eqn: Rl} as follows:
     \begin{align}
        \frac{\paXone}{\paXzero}  &\leq \left(1+\left(\frac{\numSamples\tol}{\dimPOne\dimPTwo}\right)^2\right) \left(\sum_{i=0}^{\floor{l/2}} \binom{l}{2i}\tol^{2i}\right) \\
        &= 1 + \tol^2\sum_{i=1}^{\floor{l/2}} \binom{l}{2i} \tol^{2i-2} + \left(\frac{\numSamples\tol}{\dimPOne\dimPTwo}\right)^2 \sum_{i=0}^{\floor{l/2}} \binom{l}{2i}\tol^{2i}
        \label{eq: thirdTerm_2<l<2/eps}
    \end{align}
    Now, we bound each term in the right-hand side of \eqref{eq: thirdTerm_2<l<2/eps} as  follows:
    \begin{equation}
        \sum_{i=1}^{\floor{l/2}} \binom{l}{2i} \tol^{2i-2} \leq \sum_{i=1}^{\floor{l/2}} \left(\frac{e l}{2i}\right)^{2i} \tol^{2i-2} = l^2 \sum_{i=1}^{\floor{l/2}} \left(\frac{e l \tol}{2i}\right)^{2i-2} (\frac{e}{2i})^2 \leq l^2 \sum_{i=1}^{\infty} \left(\frac{e}{i}\right)^{2i} \leq 12l^2.
    \end{equation}
    The first inequality follows from the standard Stirling-type bound $\binom{n}{m}\le\left(\frac{en}{m}\right)^m$. The second inequality uses the assumption that we are in the regime $l < 2/\tol$, which implies $\frac{e l \tol}{2i} \le \frac{e}{i}$. Finally, the last inequality holds because the series converges to a finite constant (in fact, less than $12$), which follows directly from the root test.
    \begin{equation}
            \sum_{i=0}^{\floor{l/2}} \binom{l}{2i} \tol^{2i}=\frac{(1+\tol)^l+(1-\tol)^l}{2} \leq (1+\tol)^l \leq e^{l\tol}\leq e^2
    \end{equation}   
    The equality follows from the binomial expansions of $(1+\tol)^l$ and $(1-\tol)^l$, where the odd-degree terms cancel due to opposite signs and only the even-degree terms remain. 
    The first inequality holds since $(1-\tol)^l \le (1+\tol)^l$.  
    The second inequality uses the standard bound $1+u \le e^u$.  
    The last inequality follows from the fact that $l < 2/\tol$.\\
    Substituting the above bounds into~\eqref{eq: thirdTerm_2<l<2/eps}, we obtain:
    \begin{align}
        \frac{\paXone}{\paXzero}  &\leq 1 + \tol^2\sum_{i=1}^{\floor{l/2}} \binom{l}{2i} \tol^{2i-2} + \left(\frac{\numSamples\tol}{\dimPOne\dimPTwo}\right)^2 \sum_{i=0}^{\floor{l/2}} \binom{l}{2i}\tol^{2i}\notag \\
        &\leq 1+\tol^2\left(12l^2+e^2\left(\frac{\numSamples}{\dimPOne\dimPTwo}\right)^2\right) \notag \\
        &= 1+\OO(\tol^2l^2)
    \end{align}
    We then upper-bound the contribution to the mutual information from this part by calculating the fourth moment of the Poisson random variable with mean $\lambda = \frac{\numSamples}{\dimPOne\dimPTwo}$. This calculation utilizes the fact that for $l\geq2$, the ratio of $l^4$ and $l(l-1)(l-2)(l-3)+l(l-1)$ is bounded by a constant, specifically $1 \leq \frac{l^4}{l(l-1)(l-2)(l-3)+l(l-1)}\leq 15$:
    \begin{align}
        \label{eq: MI_L2to2overE}
        &\sum_{l=2}^{2/\tol}\paXmin \left(1-\frac{\paXmax}{\paXmin}\right)^2 \notag \\
        &\quaaad \leq  \sum_{l=2}^{2/\tol}\frac{e^{-\lambda}\lambda^l}{l!}\tol^4l^4 \notag \\
        &\quaaad\leq 15\tol^4\sum_{l=2}^{2/\tol}\frac{e^{-\lambda}\lambda^l}{l!}\left(l(l-1)(l-2)(l-3)+l(l-1)\right) \notag \\
        &\quaaad\leq 15\tol^4\left(\lambda^4+\lambda^2\right)\notag  \\
        &\quaaad= \OO\left(\left(\frac{\numSamples}{\dimPOne\dimPTwo}\right)^2 \tol^4\right).
    \end{align}
    For $l \geq 2/\tol$, it is straightforward to show that the ratio in \eqref{eqn: Rl} is upper-bounded by $(1+\tol)^le^{-\lambda \tol}$ with $\lambda= \frac{\numSamples}{\dimPOne \dimPTwo}$. Then, the corresponding contribution to the mutual information can be bounded by applying Chernoff's bound to the tails of the Poisson distribution and utilizing the fact that $x^2e^{1-x}\leq x$ for all $x\in[0,1]$:
    \begin{align}
        \label{eq:MI_Poi_tail}
         &\sum_{l\geq2/\tol}\paXmin \left(1-\frac{\paXmax}{\paXmin}\right)^2 \\
         &\quaaad \leq \sum_{l\geq2/\tol}\frac{e^{-\lambda}\lambda^l}{l!}\left(1-(1+\tol)^l e^{-\lambda \tol}\right)^2 \notag \\
         &\quaaad \leq \sum_{l\geq2/\tol}\frac{e^{-\lambda}\lambda^l}{l!}\left(1+(1+\tol)^{2l} e^{-2\lambda \tol}\right) \\
         &\quaaad= \textbf{Pr}[\mathrm{Poi}(\lambda)\geq2/\tol]+e^{\lambda \tol^2} \textbf{Pr}[\mathrm{Poi}\left(\lambda(1+\tol)^2\right)\geq2/\tol] \notag \\
         &\quaaad\leq \left(\frac{\lambda}{2/\tol}\right)^{2/\tol} e^{2/\tol-\lambda} + e^{\lambda \tol^2} \left(\frac{\lambda(1+\tol)^2}{2/\tol}\right)^{2/\tol} e^{2/\tol-\lambda(1+\tol)^2} \notag \\
         &\quaaad= \left(\frac{e}{4}(\lambda\tol)^2e^{1-\lambda\tol}\right)^{1/\tol} + \left(\frac{e^{1-2\lambda\tol^2}(1+\tol)^4}{4}(\lambda\tol)^2e^{1-\lambda\tol}\right)^{1/\tol} \notag \\
         &\quaaad\leq \left(\lambda \tol\right)^{1/\tol}+\left((1+\tol)^4\lambda \tol\right)^{1/\tol}
    \end{align}
    When $\tol$ is sufficiently small, say $\tol \leq \frac{1}{4}$, the upper bound in \eqref{eq:MI_Poi_tail} is of order $O(\lambda^4 \tol^4)$. Combining this with the bounds in \eqref{eq: MI_L0}, \eqref{eq: MI_L1}, and \eqref{eq: MI_L2to2overE} establishes the claimed upper bound on the mutual information and completes the proof.
\end{proof}